\documentclass{article}
\usepackage{graphicx} % Required for inserting images

\title{Para-Gaussian Approximations of Random Deep Neural Networks}
\author{Mihai Nica and Janosch Ortmann}
\date{October 2026}

\usepackage[utf8]{inputenc} % allow utf-8 input

\usepackage[style=alphabetic, maxbibnames=99]{biblatex}
\usepackage{float}

\usepackage[margin=1in]{geometry}
\usepackage{amsmath}
\usepackage{amsthm}
\usepackage{amssymb}
\usepackage{enumerate}

\usepackage[T1]{fontenc}    % use 8-bit T1 fonts
\usepackage{hyperref}       % hyperlinks
\hypersetup{
    colorlinks=true,
    linkcolor=blue,
    filecolor=magenta,      
    urlcolor=cyan}
\usepackage{url}            % simple URL typesetting
\usepackage{booktabs}       % professional-quality tables

\usepackage{amsfonts}       % blackboard math symbols
\usepackage{nicefrac}       % compact symbols for 1/2, etc.
\usepackage{microtype}      % microtypography
\usepackage{xcolor}         % colors
\usepackage{graphicx}
\usepackage{bm}

\renewbibmacro{in:}{%
  \ifentrytype{article}
    {}
    {\printtext{\bibstring{in}\intitlepunct}}}

\DeclareFieldFormat[article, inproceedings]{title}{#1}

\global\long\def\cov{\mathbf{Cov}}%
\global\long\def\var{\mathbf{Var}}%
\global\long\def\e{\mathbf{E}}%
\global\long\def\norm#1{\left\Vert #1\right\Vert }%
\global\long\def\abs#1{\left|#1\right|}%
\global\long\def\given#1{\left|#1\right.}%
\global\long\def\bE{\mathbb{E}}%
\global\long\def\bN{\mathbb{N}}%
\global\long\def\bR{\mathbb{R}}%
\global\long\def\bZ{\mathbb{Z}}%

\global\long\def\cE{\mathcal{E}}%
\global\long\def\cF{\mathcal{F}}%
\global\long\def\d{\text{d}}%

\global\long\def\al{\alpha}%
\global\long\def\be{\beta}%
\global\long\def\ga{\gamma}%
\global\long\def\de{\delta}%
\global\long\def\ka{\kappa}%
\global\long\def\la{\lambda}%
\global\long\def\si{\sigma}%
\global\long\def\vp{\varphi}%
\global\long\def\Si{\Sigma}%

\global\long\def\dequal{\stackrel{d}{=}}%

\newcommand{\rb}[1]{\left(#1\right)}
\newcommand{\ab}[1]{\left[#1\right]}
\newcommand{\set}[1]{\left\{#1\right\}}
\newcommand{\wt}[1]{\widetilde{#1}}

\theoremstyle{plain}
\newtheorem{thm}{Theorem}[section]
\newtheorem{defn}[thm]{Definition}
\newtheorem{assumption}[thm]{Assumption}
\theoremstyle{remark}
\newtheorem{rem}[thm]{Remark}

\theoremstyle{plain}
\newtheorem{prop}[thm]{Proposition}
\theoremstyle{plain}
\newtheorem{lem}[thm]{Lemma}
\theoremstyle{plain}
\newtheorem{cor}[thm]{Corollary}

\newcommand{\nneuron}{n}

\newcommand{\dataset}{A}

\newcommand{\pg}{para-Gaussian}
\newcommand{\Pg}{Para-Gaussian}
\newcommand{\pgdk}[4]{\mathcal{P}^{(#1,#2,#4)}_{#3}}
\newcommand{\pgd}[3]{\mathcal{P}^{(#1,#2)}_{#3}}
\newcommand{\pglk}[4]{\mathcal{P}^{(#1,#2,#4)}_{#3}}
\newcommand{\pgl}[3]{\mathcal{P}^{(#1,#2)}_{#3}}

\newcommand{\GaussExp}[1]{\bm{\cE}_{#1}}

\newcommand{\pgfourierk}[4]{\phi^{(#1,#2,#4)}_{#3}}

\begin{document}
	
	\maketitle
	
	\begin{abstract}
    It is known that randomly initialized neural networks have Gaussian behaviour in the $n\to \infty$ limit. Unfortunately, this approximation can be poor for finite $n$ and gets worse as one goes deeper into the network. In this work, we develop a more accurate approximation that is \emph{slightly} non-Gaussian, which we call the \emph{\pg}\ distribution and is constructed from the Hermite polynomials. We  show how this para-Gaussian approximation can lead to a practical recursive method to approximate the deeper layers of the network. Finally, we apply this theory to linear neural networks (where things can be solved exactly) and to ReLU neural networks (where things can be approximated).
			\end{abstract}

	\section{Introduction}
\begin{figure}[H]
    \centering
    \includegraphics[width=0.95\linewidth]{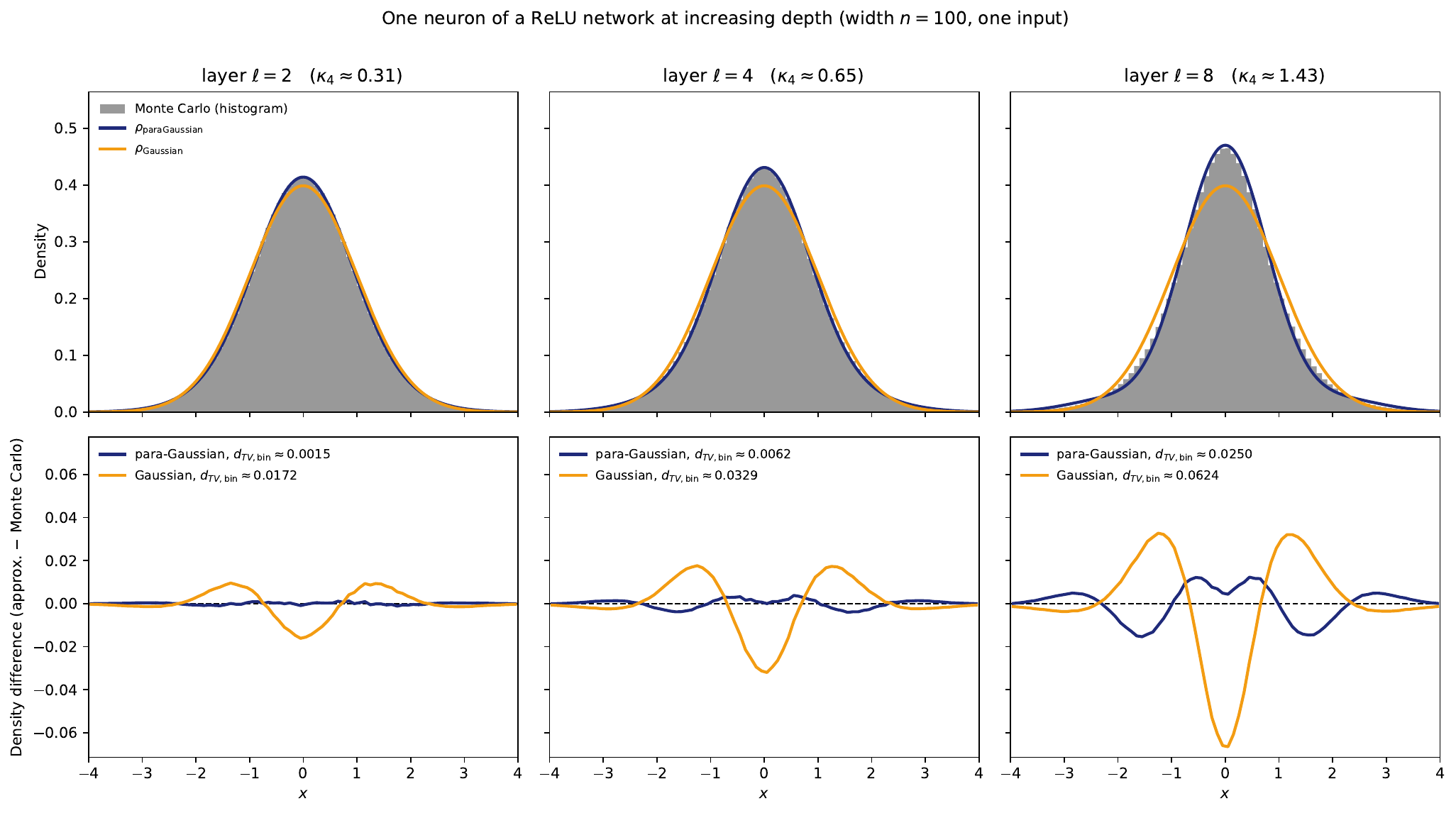}
    \caption{\footnotesize The distribution of a single neuron in a randomly initialized ReLU network ($\vp(x)=\sqrt{2}\max\{x,0\}$, width $n=100$, one input) at
    layers $\ell=2,4,8$, estimated from $2^{24}$ Monte Carlo samples per layer. \textbf{Top row:} Comparing the Monte Carlo estimates (grey), to the the Gaussian approximation $\mathcal{N}(0,1)$ (yellow) and the para-Gaussian law (blue) with density
    $\frac{1}{\sqrt{2\pi}}e^{-x^2/2}\big[1+\frac{\ka_4}{4!}\mathrm{He}_4(x)\big]$, where $\mathrm{He}_4(x)=x^4-6x^2+3$ and $\ka_4=3\big[(1+5/n)^\ell-1\big]$.
    \textbf{Bottom row:} the \emph{difference} between each approximation and the Monte Carlo density, estimated by a histogram with bin width $0.1$,
    with both approximations averaged over the same bins. The legend shows the binned total variation distance
    $d_{TV,\mathrm{bin}}=\frac12\sum_{k}\abs{p_k-q_k}$, where $p_k$ is the fraction of Monte Carlo samples in the $k$-th bin, $q_k$ is the mass the approximation assigns to that bin. Notice that the deviation from the Gaussian grows with depth. Theorem \ref{thm:TV-gauss-pg} and Theorem \ref{thm:main} generalize this picture to multiple inputs and outputs to obtain the \emph{joint} para-Gaussian distribution that approximates the distribution of multiple neurons in a neural network.}
    \label{fig:intro-not-quite-gaussian}
\end{figure}

    \subsection{Gaussians and not-quite-Gaussians}
    It is a staple of statistics to approximate complicated distributions by the humble Gaussian law. Randomly initialized large neural networks are no exception and many theoretical works  have studied how as the number of neurons, $n$, in the hidden layers grows large, $n \to \infty$, there is convergence to a Gaussian process \cite{Lee2018NNGP,Matthews2018,Yang2019,Hanin2023GP}. However, to invoke another statistical trope, one may inquire: how large of an $n$ is large enough? This question is of particular interest in neural networks because the Gaussian approximation gets \emph{worse} the deeper one looks in the network (see Figure \ref{fig:intro-not-quite-gaussian} for a simple Monte Carlo experiment illustrating this). In other words, the errors in the Gaussian approximation stack up as more layers are added to the network. To understand truly \emph{deep} neural networks, we must go beyond the Gaussian law.

    %This is because each output is a sum of independent random variables, so the Central Limit Theorem can be applied. In other words, the vector $\vec X$ of outputs of a layer has \textbf{approximate (?)} characteristic function
    %\begin{align}
    %    \phi(\vec u): = \bE[e^{i \vec u \cdot \vec X}] = \exp\set{-\frac12 \vec u^T \Sigma u}
    %\end{align}
    %where $\Sigma$ is the covariance matrix between the neurons.

    %This is approximation is the \emph{NN Gaussian process} (NNGP) \textbf{TODO: references?}. 
   %But: $n\to \infty$ limit is not quite right. (e.g. depth degeneracy and whatnot).

    %While this picture holds cleanly for a fixed depth of the network as the number of neurons in a layer goes to infinity, the infinite-depth-and-with view has received much attention recently. (\textbf{TODO: references}). [\textbf{TODO:} give some reasons here already why the approximation is insufficient? References]

    %Since the NNGP theory follows from the central limit theorem, can we take lower orders of that result into account  to obtain a better approximation? In \cite{NicaOrtmann2024M3L} we hint at the fact that we can, see Figure \ref{fig:intro-not-quite-gaussian} which shows the distribution of a single neuron in a ReLU network.

 One classical method beyond the Gaussian is the \emph{Edgeworth expansion} around the Central Limit Theorem (CLT) for sums of i.i.d. random variables. Concretely,  let $Y^1,Y^2,\ldots$ be i.i.d. centred $\bR^d$-valued random vectors of some law $Y$ with covariance matrix $\Si$ and let $X^n:=\frac1{\sqrt n}\sum_{j=1}^nY^j$ be their normalized sum\footnote{Note that we use superscripts for sequences of vectors like $X^n$ in order to save subscripts for the components. Don't be mixed up with exponents!}. Assume also that the law of $Y$ has bounded density w.r.t Lebesgue\footnote{This avoids discrete distributions that complicate things; a more general condition that is also sufficient is that the characteristic function satisfies $\intop |\phi_{Y}(u)|^p\d u <\infty $ for some $p\geq 1$.} and enough finite moments. The classical  proof of the CLT shows the only surviving part of the characteristic function of $X^n$ is the Gaussian part, whereas the Edgeworth expansion keeps the next order vanishing term, precisely one has (see \cite{NicaOrtmann2026} for more details),
    \begin{align}
	 \label{eq:Fourier-PG-k}
	\mathbb{E}\left[ e^{\imath \langle X^n, \vec u \rangle} \right] = \phi_{X^n} (\vec{u})
	&\approx \exp\left( -\frac{1}{2}\vec{u}^T \Sigma \vec{u} + \frac{1}{\sqrt{n}^{k-2}} \sum_{\substack{k_1,\ldots,k_d\\ k_1+\ldots+k_d = k}} \frac{\kappa_{k_1,\ldots,k_d}}{k_1!\cdots k_d!} (\imath \vec{u})^{\vec{k}} \right), 
\end{align}
where $\vec{u}^{\vec{k}} = {u_1}^{k_1}\cdots u_d^{k_d}$ and $\kappa_{k_1,\ldots,k_n} = \kappa\big( \underbrace{Y_1,\ldots,Y_1}_{k_1\text{ times}},\ldots, \underbrace{Y_d,\ldots,Y_d}_{k_d\text{ times}}  \big)$ are the \emph{joint cumulants} of the components of the random vector $Y$, and $k \geq 3$ is the order of the \emph{first non-trivial higher cumulant} of $Y$ (i.e. the smallest $k$ for which this is non-zero). For symmetric random variables (i.e. $Y \dequal -Y$), this first non-trivial higher order is typically $k=4$.

The characteristic function \eqref{eq:Fourier-PG-k} is a not-quite-Gaussian characteristic function. Modulo the approximation $\exp \frac{1}{n^p} x \approx 1 + \frac{1}{n^p}x$, when we invert the characteristic function \eqref{eq:Fourier-PG-k} to obtain a density for the sum $X^n$, we find the following not-quite-Gaussian density function involving Hermite polynomials. We call this a para-Gaussian\footnote{``para'' here being the Greek prefix meaning nearby, alongside, beyond, or beside. The common word ``paranormal'', meaning slightly outside the normal laws of physics, has a similar etymological decomposition. } distribution.

    \begin{defn}
    \label{def:PG}
    Fix $d\in \mathbb{N}$ and a covariance matrix $\Sigma \in \mathbb{R}^{d \times d}$. For $\vec{k}=(k_1,\ldots,k_d) \in \mathbb{Z}_{\geq 0}^d$, the {\bf \emph{multivariable Hermite polynomial}} $H_{k_1,\ldots,k_d}(\cdot ;\Sigma) : \mathbb{R}^d \to \mathbb{R}$ is the polynomial that satisfies
    \begin{align}\label{eq:H_is_iu_to_the_n}
			H_{k_1,\ldots,k_d} \rb{x_1,\ldots,x_d; \Si}\vp_\Si(\vec{x}) & = \frac1{(2\pi)^d}\int_{\bR^d} (\imath \vec{u})^{\vec{k}}\, \exp\set{-\frac12 \vec{u}^T \Si \vec{u}-\imath \,\vec{x}\cdot \vec{u}} \, d\,\vec{u} = (-1)^{|\vec{k}|} \partial_{\vec{k}} \varphi_\Si(\vec{x}),
		\end{align}
where $\varphi_\Sigma(\vec{x}) ={\sqrt{(2\pi)^d \det{\Sigma}}}^{-1}  e^{-\frac12 \vec{x}^T \Sigma^{-1} \vec{x}}$ is the multi-variable Gaussian density of covariance $\Sigma$. This can alternatively be written as a Gaussian expectation
    \begin{align*}
        H_{k_1,\ldots,k_d} \rb{x_1,\ldots,x_d; \Si} & = \bE_{\vec Z\sim \mathcal{N}(0,\Si^{-1})}\ab{\prod_{a=1}^d \Big(\rb{\Si^{-1}\vec x}_a +\imath Z_a\Big)^{k_a}}
    \end{align*}
		The {\bf\emph{\pg\ measure}} $\pgdk \Sigma\ka n k$ on $\bR^d$ of order $k$ with covariance kernel $\Sigma$, $k$-cumulants $\kappa=\{\kappa_{k_1,\ldots,k_d}\}_{k_1 + \ldots +k_d = k}$ and scale parameter $n$ has density with respect to the Lebesgue measure on $\bR^d$ given by 
		\begin{align}
        \label{eq:paraGaussian-density}
			\pgdk \Sigma\ka n k (\vec{x}) &= 
			\varphi_{\Si}(\vec{x})\Bigg(1 + \frac{1}{\sqrt{n}^{k-2}} \sum_{\substack{k_1,\ldots,k_d\\ k_1+\ldots+k_d = k}} \frac{\kappa_{k_1, k_2,\ldots,k_d}}{k_1!k_2!\cdots k_d!} H_{k_1,\ldots,k_p}(\vec{x},\Sigma)\Bigg)
		\end{align}
        This density has the following Fourier transform (aka characteristic function):
        \begin{align}
        \label{eq:paraGaussian-transform}
			\intop e^{\imath \langle{\vec{x},  \vec u}\rangle }{\pgdk \Sigma\ka n k} (\vec{x}) \d \vec{x}  &=\exp\left( -\frac{1}{2}\vec{u}^T \Sigma \vec{u}\right)\left(1 + \frac{1}{\sqrt{n}^{k-2}} \sum_{\substack{k_1,\ldots,k_d\\ k_1+\ldots+k_d = k}} \frac{\kappa_{k_1,\ldots,k_d}}{k_1!\cdots k_d!} (\imath \vec{u})^{\vec{k}} \right),
		\end{align}
        We abuse notation and interchangeably use the notation $\pgdk \Sigma\ka n k$ to refer to the density or the measure whenever it is convenient. When $k$ is understood from context (for example, the main case studied in this article has $k=4$ everywhere) we omit the notation $k$ and simply write $\pgd\Si\ka n$. 
	\end{defn}

See Appendix \ref{sec:pg-theory} for some more useful facts about para-Gaussians. For instance, like the classical Gaussian, the para-Gaussian family is closed under linear combinations.

With this definition in hand, the classical Edgeworth expansion result can be summarized as \emph{``Sums of i.i.d. random variables are approximately para-Gaussian''}. The main purpose of this article is to discuss how \pg\ distributions \emph{also} arise in deep fully connected neural networks when they are randomly initialized. The \pg\ law shown in Figure \ref{fig:intro-not-quite-gaussian} illustrates the simplest single-input-single-output version of this idea, where the para-Gaussian law matches the Monte Carlo simulations more accurately than the plain Gaussian. Our main results generalize Figure \ref{fig:intro-not-quite-gaussian} to multiple inputs and multiple outputs (more Monte Carlo simulations showing 2 inputs or 2 outputs can also be found in Sections \ref{sec:linear-application} and \ref{sec:relu-application}).

\subsection{Outline and Contributions}

The rest of the paper is structured as follows. 

\begin{itemize}
\item[\S\ref{sec:tv-distance}] We use the total variation distance to precisely introduce our notation of ``asymptotic rate of convergence''
\item[\S\ref{sec:neural-nets-defn}] We precisely define the randomly initialized neural networks we study here. We state one of our main results, Theorem \ref{thm:TV-gauss-pg}, that the para-Gaussian approximation with $k=4$ has a better rate of convergence to neural networks than the ordinary Gaussian, $O(n^{-2})$ vs. $O(n^{-1})$
\item[\S\ref{sec:recursions}] We show how the para-Gaussian approximation can be used to obtain a recursion relation that estimates the 4-th cumulant in neural networks in terms of the multi-variable Hermite polynomials and state this result as Theorem \ref{thm:main}.
\item[\S\ref{sec:literature}] We discuss how our results fit into the related literature. 
\item[\S\ref{sec:higher_cumulants}] We state and prove Theorem \ref{thm:exact-general-cumulants} that gives the connection between cumulants of the neural network and the random covariance matrix $\Si$
\item[\S\ref{sec:linear-application}] We use our theory to exactly solve the covariance and 4th cumulant in linear networks, and thereby compute the para-Gaussian approximation for these. This is tested against Monte Carlo experiments.
\item[\S \ref{sec:relu-application}] We apply our theory to get approximations for ReLU Neural Networks. This is tested against Monte Carlo experiments. 
\item[\S \ref{sec:para-proof}] Proof details for Theorem \ref{thm:TV-gauss-pg} on the convergence rates of para-Gaussian approximations.
\item[\S \ref{sec:proof-main}] Proof details for Theorem \ref{thm:main} on the recursions for approximating neural networks
\end{itemize}

Our main contributions are Theorem \ref{thm:TV-gauss-pg}, Theorem \ref{thm:main}, and Theorem \ref{thm:exact-general-cumulants}, as described above, as well as the exact solution for the para-Gaussian approximation to linear neural networks of any depth described in \S \ref{sec:linear-application}.

%The characteristic function can be used to calculate expected values and higher moments, see [Boris Hanin and others]. physics types of things, differentiate relu, tempered distributions.

    %However, for other considerations the probability distribution is required. Note that in the one-dimensional setting, Hermite polynomials turn up [\textbf{TODO:} give a reference for edgeworth expansions?] and that a theory of multi-dimensional Hermite polynomials exist. See for example \PSpaper. In particular, the Hermite polynomial with respect to a covariance matrix $\Si$ is given by

    %In \PSpaper it was shown (see Section 7.4 for the one-dimensional case and Section 7.5 for the multidimensional case), that the Fourier transform of \eqref{eq:Fourier-PG} can be inverted to get density for Edgeworth expansions. The key step is to use the approximation $\exp(x)\approx 1+x$ for small $x$. Note that this approximation also appears in the one-dimensional Edgeworth expansions.
    %Applying this approximation, we obtain what we call the \pg\ measure.

\subsection{Total Variation Distance and Asymptotic Rates of Convergence}
\label{sec:tv-distance}

One of our main results is that the para-Gaussian law is a better approximation that the purely Gaussian law in terms of total varaition distance, which we make precise with the following notion of convergence rate.

\begin{defn}
\label{def:TV}
Let $\mu,\nu$ be finite signed measures on $\bR^d$. The {\bf\emph{total variation distance}} between them is the worst-case
difference of integrals over measurable test functions bounded by one,
\[
d_{TV}(\mu,\nu) := \frac{1}{2}\sup_{f:\,\norm{f}_\infty\leq1} \abs{\intop f\, \d\mu - \intop f\, \d\nu},
\]
where the supremum is over measurable $f:\bR^d\to\bR$. For random vectors $X,Y$ we write $d_{TV}(X,Y)$ for the distance between
their laws, so that $d_{TV}(X,Y)=\frac12 \sup_{\norm{f}_\infty\leq1}\abs{\mathbb{E}[f(X)]-\mathbb{E}[f(Y)]}$, and similarly $d_{TV}(X,\nu)$ for
the distance between the law of $X$ and a signed measure $\nu$. We will freely mix and match between random variables and their
laws whenever it is convenient.
\end{defn}

\begin{defn}
\label{def:asymptotically-distributed}
Let $(\mu_n)_{n\geq1}$ be a sequence of finite signed measures on $\bR^d$, and let $\alpha>0$. A sequence $(X_n)_{n\geq1}$ of
$\bR^d$-valued random vectors is {\bf\emph{asymptotically distributed like $\mu_n$ with rate $n^{-\alpha}$}} if
\[
d_{TV}(X_n,\mu_n)=O\left(n^{-\al}\right),\qquad\text{that is,}\qquad \limsup_{n\to\infty}\,n^{\alpha}\,d_{TV}(X_n,\mu_n)<\infty.
\]
When $\mu_n=\mathcal{N}(0,\Si_n)$ we say that $X_n$ is {\bf\emph{asymptotically Gaussian}} with rate $n^{-\alpha}$, and when
$\mu_n=\pglk {\Si_n}{\ka_n}{n}k$ is a \pg\ law %(see Proposition \ref{prop:integrability})
we say that $X_n$ is {\bf\emph{asymptotically \pg}}
with rate $n^{-\alpha}$.
\end{defn}

\begin{rem}
\label{rem:iid-sums}
The classical example is a normalized sum of i.i.d.\ random vectors as in \eqref{eq:Fourier-PG-k}. This sequence is \emph{both} asymptotically Gaussian and
asymptotically \pg, but with different rates. 
\begin{itemize}
\item $X^n$ is asymptotically Gaussian, with $\Si_n=\Si$, with rate $n^{-1/2}$.
\item When the third cumulants of the law $Y$ vanish, then the leading correction to the Gaussian approximation is of order $\frac1n$ and is
given by the fourth cumulants: $X^n$ is asymptotically para-Gaussian with
\[
\Si_n=\Si,\qquad \ka=\ka^{(4)}(Y),
\]
where $\ka^{(4)}(Y)_{abcd}=\ka\big(Y_{a},Y_{b},Y_{c},Y_{d}\big)$ is the fourth cumulant tensor of the law $Y$, with rate $n^{-3/2}$.
\end{itemize}
If all the odd cumulants of $Y$ vanish, for instance if $Y$ is symmetric, both rates improve: $X^n$ is asymptotically Gaussian
with rate $n^{-1}$ and asymptotically para-Gaussian with rate $n^{-2}$. These are exactly the rates in Theorem
\ref{thm:TV-gauss-pg}. See \cite[Theorems 7.9 and 7.10]{NicaOrtmann2026} for a derivation of these expansions via Hermite polynomials.
% TODO: check the exact statement numbers in Bhattacharya--Rao (Theorem 19.2 gives the weighted sup-norm density bound,
% which integrates to an L^1 = total variation bound; Corollary 19.6 was the reference in the earlier draft).
\end{rem}

\subsection{Random fully connected neural networks on initialization and their rates of convergence}
\label{sec:neural-nets-defn}

We now precisely define the fully connected neural networks studied in this paper. 

\begin{defn}
\label{def:neural_net_defn}
    Fix a finite set of inputs $\dataset$, and with each input $\alpha\in\dataset$ a 
vector $x_\alpha\in\bR^{n_{\mathrm{in}}}$, a hidden layer width $n$, and a non-linearity $\vp:\bR\longrightarrow\bR$. The network is defined layer by layer with each layers output indexed by $\{ z_{\alpha,k}^\ell\}_{\alpha \in A, k \in [n]}$: the
neurons of layer $0$ are
\[
z^0_{\alpha,k}=\frac1{\sqrt{n_{\mathrm{in}}}}\sum_{i=1}^{n_{\mathrm{in}}}W^0_{k,i}\,x_{\alpha,i},\qquad k\in[n],
\]
and for $\ell\geq1$ the value of neuron $k\in[n]$ on input $\alpha\in\dataset$ in layer $\ell$ is
\begin{align}
\label{eq:NN-recursion}
z_{\al,k}^\ell = \frac1{\sqrt {n}} \sum_{j=1}^{\nneuron} W^{\ell}_{k,j}\, \varphi\rb{z^{\ell-1}_{\alpha,j}},
\end{align}
where all the weights $W^\ell_{k,j}$ are independent standard Gaussians, independent across layers. 
\end{defn}
\begin{rem}
    The results here can easily be generalized to different hidden layer size $n_1,n_2,\ldots$ but we use $n$ for all of them for notational convenience in this article.
\end{rem}
We are interested in the joint
distribution of all the neurons and all the inputs, layer by layer. A convenient normalization is to choose $\vp$ with
$\mathbb{E}_{G\sim\mathcal{N}(0,1)}\big[\vp(G)^2\big]=1$, so that the variance of a neuron is preserved from layer to layer when the
input has unit norm. Examples are the identity $\vp(x)=x$ and the normalized ReLU $\vp(x)=\sqrt2\max\set{x,0}$. In order to prove any results, we need some regularity on the non-linearity $\vp$. We take the following standard assumption on the non-linearity from \cite[Assumption 1]{Hanin2024}
\begin{assumption}
\label{ass:nonlinearity}
We assume that there exists an $r\in \mathbb{N}$ so that the non-linearity $\varphi:\mathbb{R}\to \mathbb{R}$ is $r$ times differentiable almost everywhere, and grows at most polynomially. 
\end{assumption}

\begin{defn}
\label{def:Si}\textbf{(Random covariance matrix)}
For each layer $\ell$, define the random matrix $\Si^\ell \in \mathbb{R}^{A \times A}$ by:
\begin{equation}
\label{eq:Sigma-def-intro}
\Si^\ell(\alpha,\beta):=\frac1n\sum_{j=1}^n\vp\rb{z^{\ell-1}_{\alpha,j}}\vp\rb{z^{\ell-1}_{\beta,j}},\qquad \alpha,\beta\in\dataset.
\end{equation}
\end{defn}
\begin{rem}
Note that, perhaps confusingly, $\Si^\ell$ is defined in terms of the vectors in layer $\ell-1$. This choice of definition is because the matrix $\Si^\ell$ importantly controls the distribution of layer $\ell$ as in \eqref{eq:conditional_gauss}, so we have chosen to name it after the layer $\ell$ it is \emph{used for} (rather than the layer $\ell-1$ which it is \emph{defined from}.\footnote{And indeed other authors (for example \cite{Hanin2024}) make the reasonable opposite choice to index this by $\ell-1$. Watch out!})
\end{rem}

    \begin{rem}
\label{rem:cond-gauss}
     From the definition of the network \eqref{eq:NN-recursion}, since the weights in each layer are Gaussian, we see that conditionally on the values previous layer $z^{\ell-1}$, the next layer is \emph{conditionally} Gaussian with this covariance matrix 
     \begin{equation}
     \label{eq:conditional_gauss}
     \big(z^\ell_{\alpha,j}\big)_{\alpha\in A,\,j\in[n]} \sim \mathcal{N}\left(\vec{0},I_n \otimes \Si^\ell\right) \text{ conditionally on } \big(z^{\ell-1}_{\alpha,j}\big)_{\alpha\in A,\,j\in[n]}, 
     \end{equation}
     where the covariance $I_n \otimes \Si^\ell$ means that $\e\left[z^\ell_{\alpha,j}\,z^\ell_{\be,k}\middle| \big(z^{\ell-1}_{\alpha,j}\big)_{\alpha\in A,\,j\in[n]}\right]=\de_{j,k}\Si^{\ell}\left(\alpha,\beta\right)$ and $\delta_{a,b}$ is the Kronecker delta. 
     \end{rem}

     As $n\to\infty$, $\Si^\ell$ concentrates around the deterministic infinite width kernel $K^\ell$, defined by the input kernel
$K^0(\alpha,\beta):=\frac1{n_{\mathrm{in}}}\langle x_\alpha,x_\beta\rangle$ and recursively afterwards by
\begin{equation}
\label{eq:K-def}
K^\ell(\alpha,\beta):=\mathbb{E}_{z\sim\mathcal{N}(0,K^{\ell-1})}\big[\vp(z_\alpha)\vp(z_\beta)\big],
\end{equation}
and the infinite width network converges to a Gaussian
process \cite{Lee2018NNGP,Hanin2024} with this kernel. Many of the results are contingent on this infinite-width kernel being non-degenerate, which we state as an assumption.
\begin{assumption}
\label{ass:non-degen}
For every layer $\ell'\leq\ell$, the infinite width kernel $K^{\ell'}$ defined in \eqref{eq:K-def} is non-degenerate (strictly positive definite)
on the inputs $A$.
\end{assumption}

The fluctuations of $\Si^\ell$ around its mean $$\bar\Si^\ell(\alpha,\beta):=\mathbb{E}[\Si^\ell(\alpha,\beta)],$$ are of order $n^{-1/2}$, and
are the source of the non-Gaussian corrections at finite width.

  In fact, more precisely, the higher cumulants of collections of neurons $z^\ell$ can be calculated directly in terms of the higher cumulants of the random matrix $\Si^\ell$, and this is carried out in Theorem \ref{thm:exact-general-cumulants}. To state our main result we need only the second and fourth cumulant, which we record in the following corrollary.

\begin{cor}[to Theorem \ref{thm:exact-general-cumulants}] \label{cor:cumulant_two_and_four} The second and fourth cumulant of the neurons in layer $\ell$ can be written in terms of the random matrix $\Si^\ell$ as follows.

\underline{Covariance:} The covariance of two neurons is given in terms of the expectation of $\Si^\ell$, namely:

    \begin{equation}
    \label{eq:cov-z}
\cov\left[z^\ell_{\alpha_1,j_1},z^\ell_{\alpha_2,j_2}\right] = \delta_{j_1,j_2} \mathbb{E}[\Si^\ell(\alpha_1,\alpha_2)].
    \end{equation}

    For notational compactness, \eqref{eq:cov-z} is saying the covariance of $\{z^\ell_{\alpha,j}\}_{\alpha\in A,1\leq j \leq m}$ is $I_m \otimes \bar\Si^\ell$ for $m \in \mathbb{N}$

    \underline{4th cumulant:} The fourth cumulant for four neurons (not necessarily distinct) is given in terms of covariances of $\Si^\ell$, namely

    \begin{align}
    \label{eq:kappa_from_C}
    \kappa\left( z_{\alpha_1,j_1}^\ell ,z_{\alpha_2,j_2}^\ell,z_{\alpha_3,j_3}^\ell ,z_{\alpha_4,j_4}^\ell \right) & = \delta_{j_1,j_2}\delta_{j_3,j_4} \cov \left[ \Si^\ell(\alpha_1,\alpha_2),\Si^\ell(\alpha_3,\alpha_4)\right]\\ \notag &+ \delta_{j_1,j_3}\delta_{j_2,j_4} \cov \left[ \Si^\ell(\alpha_1,\alpha_3),\Si^\ell(\alpha_2,\alpha_4)\right]\\ \notag &+ \delta_{j_1,j_4}\delta_{j_2,j_3} \cov \left[ \Si^\ell(\alpha_1,\alpha_4),\Si^\ell(\alpha_2,\alpha_3)\right]. \notag
    \end{align}
\end{cor}

One of the main purposes of this paper is to establish the following result, which shows that the para-Gaussian law is a better approximation than the pure Gaussian law. 

\begin{thm}
\label{thm:TV-gauss-pg}We fix a layer $\ell$ and a number $m$ of neurons, and
study the random vector $Z^\ell:=\set{z^\ell_{\alpha,j}}_{\alpha\in\dataset,\,j\in[m]}\in\bR^{\dataset\times[m]}$ as the width $n\to\infty$. Suppose that $\vp$ satisfies the smoothness condition Assumption \ref{ass:nonlinearity} and we have the non-degeneracy condition Assumption \ref{ass:non-degen}. Then as $n\to\infty$:
\begin{enumerate}
\item[(i)] The covariance of $Z^\ell$ is given as in \eqref{eq:cov-z}, $\cov(Z^\ell) = I_m \otimes \bar \Si^\ell$, and $Z^\ell$ is asymptotically Gaussian with rate $n^{-1}$:
\[
d_{TV}\Big(Z^\ell,\ \mathcal{N}\big(0,\,I_m\otimes\bar\Si^\ell\big)\Big)=O\left(n^{-1}\right).
\]
\item[(ii)] Let $\kappa^\ell_4$ be the $4$-th cumulant of $Z^\ell$ as given in \eqref{eq:kappa_from_C}. Then $Z^\ell$ is asymptotically para-Gaussian $\pgl{I_m\otimes\bar\Si^\ell}{n\ka_4^\ell}n$\footnote{Note that since $\pgl{\Si}{\ka}{n}$ already builds in a factor of $n$ into the cumulant, we need to feed in the rescaled cumulant  $n\kappa^\ell_4$. It turns out $n\kappa^\ell_4$ is order one, see Remark \ref{rem:C-is-O1}.} with rate $n^{-2}$:
\[
d_{TV}\Big(Z^\ell,\ \pgl{I_m\otimes\bar\Si^\ell}{n \ka_4^\ell}n  
\Big)=O\left(n^{-2}\right).
\]
%where the fourth cumulant tensor $\ka^\ell$ is given in terms of the \emph{covariances} of the random matrix $\Si^\ell$ by
%\begin{align*}
%\ka^\ell\big((\alpha,i),(\beta,j),(\gamma,k),(\delta,l)\big)
%&=\de_{ij}\de_{kl}\cov\big(\Si^\ell(\alpha,\beta),\Si^\ell(\gamma,\delta)\big)\\
%&+\de_{ik}\de_{jl}\cov\big(\Si^\ell(\alpha,\gamma),\Si^\ell(\beta,\delta)\big)\\
%&+\de_{il}\de_{jk}\cov\big(\Si^\ell(\alpha,\delta),\Si^\ell(\beta,\gamma)\big).
%\end{align*}

\end{enumerate}
\end{thm}

%\begin{rem}
%\label{rem:main-thm-parameters}
%The parameters appearing in Theorem \ref{thm:TV-gauss-pg} are exactly the true second and fourth cumulants of $Z^\ell$: the covariance
%of $Z^\ell$ is $I_m\otimes\bar\Si^\ell$, and $\ka^\ell$ is its fourth cumulant tensor (Theorem \ref{thm:exact-general-cumulants}). In particular,
%$\ka^\ell$ is of order $\frac1n$, so the para-Gaussian law differs from the Gaussian law by a correction of order $\frac1n$, and part
%(ii) shows that including this correction improves the accuracy of the approximation from $n^{-1}$ to $n^{-2}$. The Kronecker
%deltas show that the non-Gaussian corrections couple different neurons: although the neurons of a layer are uncorrelated, they are
%not independent at order $\frac1n$. Explicit recursions in $\ell$ for $\bar\Si^\ell$ and $\ka^\ell$, accurate to the orders needed in
%Theorem \ref{thm:TV-gauss-pg}, are given in Section \ref{...}.
%\end{rem}

\begin{rem}
\label{rem:part-i-known}
Note that the rate for Gaussian approximation in part (i) of Theorem \ref{thm:TV-gauss-pg} has been previously proven, in dimension one in \cite{FavaroHaninMarinucciNourdinPeccati2025}, and in any dimension in \cite{CelliPeccati2025}, where
the rate $n^{-1}$ is also shown to be optimal. We include part (i) to contrast it with part (ii), showing how the error rate improves to $n^{-2}$ from the order 4 Hermite correction in the \pg\ law. Our proof of part (ii) shares the
interpolation between the mean covariance and the random covariance with \cite{BasteriTrevisan2024,CelliPeccati2025}, and
carries the expansion one order further (concretely, this appears in our proof as keeping more terms of a particular Taylor series expansion, see Remark \ref{rem:proof-sketch}).

Instead of using only the order $k=4$ cumulants in our law,  it would be possible to extend to \pg\ law to include all cumulants  up to and including order $2p$ to create ``part $p$)'' of the theorem which achieves an error rate $n^{-p}$. The proof of this extenesion would simply go by extending the Taylor series expansion to include $2p$ terms. For concreteness and to improve the exposition of this article, we choose to focus on the effect of the fourth cumulant only here.
\end{rem}

\begin{rem}[Proof sketch of Theorem \ref{thm:TV-gauss-pg}]
\label{rem:proof-sketch}
While the details of the proof involve delicate analysis and presented in Section \ref{sec:para-proof}, the main idea of the proof of Theorem \ref{thm:TV-gauss-pg} is an instructive Taylor series expansion, so we sketch it out here.

Fix a bounded test function $f$. Conditionally on layer $\ell-1$, the vector $Z^\ell$ is exactly Gaussian with the  covariance
$I_m\otimes\Si^\ell$, so all of the non-Gaussianity of $Z^\ell$ comes from the fluctuation $\Delta^\ell:=\Si^\ell-\bar\Si^\ell$ of this (conditionally random) covariance matrix $\Si^\ell$ around its
mean. To understand the effect of the random fluctuation, consider as follows: let $g_f(t)$ be the Gaussian expectation of $f$ at the covariance $I_m\otimes(\bar\Si^\ell+t\Delta^\ell)$, $g_f(t):=\mathbb{E}_{X\sim\mathcal{N}(0,I_m\otimes(\bar\Si^\ell+t\Delta^\ell))}[f(X)]$,
which interpolates between the non-random mean $\bar\Si^\ell$ at $t=0$ (and consequently a perfectly Gaussian approximation) and the full quantity of interest $\bar\Si^\ell+(1)\Delta^\ell =\Si^\ell$ at $t=1$, in the sense that
$\mathbb{E}[g_f(1)]=\mathbb{E}[f(Z^\ell)]$. Taylor expanding $g_f(1)$ around $t=0$ and taking expectation gives
\[
\mathbb{E}[f(Z^\ell)]=\mathbb{E}[g_f(1)]=\mathbb{E}\Big[g_f(0)+g_f^\prime(0)+\tfrac12g_f^{\prime\prime}(0)+\tfrac16g_f^{\prime\prime\prime}(0)\Big]+\text{(remainder)}.
\]
The expectation of the first two terms is exactly the Gaussian law, since $g_f(0)=\mathbb{E}_{X\sim\mathcal{N}(0,I_m\otimes\bar\Si^\ell)}[f(X)]$ and
$\mathbb{E}[g_f^\prime(0)]=0$ (Lemma \ref{lem:g-prime-mean-zero}). It turns out that the expectation of the first three terms $\mathbb{E}\Big[g_f(0)+g_f^\prime(0)+\tfrac12g_f^{\prime\prime}(0)]$ is exactly the expectation under the para-Gaussian law
$\int f\,\d \pgl{I_m\otimes\bar\Si^\ell}{n\ka^\ell}{n}$ (Proposition \ref{prop:second-order-is-PG}), with the second derivative of $\frac12\mathbb{E}[g_f''(0)]$ corresponding exactly to the 4th cumulant and Hermite terms in the \pg\ law. The fourth term is of order $n^{-2}$
(Lemma \ref{lem:third-order}). The Taylor remainders are controlled by Lemma \ref{lem:taylor-remainder}, using the moment bounds on
$\Delta^\ell$ from \cite{Hanin2024} (see Lemma \ref{lem:hanin-estimates}). 

Altogether, truncating after the first two terms gives part (i)
of Theorem \ref{thm:TV-gauss-pg} with error $O(\norm{f}_\infty n^{-1})$, and truncating after the first four terms gives part (ii) with
error $O(\norm{f}_\infty n^{-2})$ (Corollary \ref{cor:error-summary}). Since these errors depend on $f$ only through $\norm{f}_\infty$,
taking the supremum over $\norm{f}_\infty\leq1$ gives the bounds in total variation.

\end{rem}

\subsection{Recursions for calculating the covariance and fourth cumulant}
\label{sec:recursions}
Theorem \ref{thm:TV-gauss-pg} shows that the \pg\ approximation is much better than the Gaussian. But operationally speaking, how can one calculate the 4th cumulants needed as an input into this law? The key observation is that from Corollary \ref{cor:cumulant_two_and_four} the 4th cumulant in layer $\ell+1$ depends only on the covariances from the layer-$\ell$-measurable random matrix $\Si^{\ell+1}$, through the following tensor $C^\ell$ defined as follows.

    \begin{defn}
    \label{def:C}

     We define the tensor $C^\ell: \left( A\times A \right) \times \left( A \times A \right) \to \mathbb{R}$ to be the covariances of $\Si^\ell$
    \begin{align}
    \label{eq:defCell}
      C^\ell(\alpha,\beta;\gamma,\delta): = n \cov \left[ \Si^\ell(\alpha,\beta),\Si^\ell(\gamma,\delta)\right] 
      \end{align}
  Note that from the symmetry $\cov(X,Y)=\cov(Y,X)$, we see that $C^\ell$ has three fundamental symmetries: $C^\ell(\alpha,\beta;\gamma,\delta)=C^\ell(\beta,\alpha;\gamma,\delta)$,  $C^\ell(\alpha,\beta;\gamma,\delta)=C^\ell(\alpha,\beta;\delta,\gamma)$, and finally $C^\ell(\alpha,\beta;\gamma,\delta)=C^\ell(\gamma,\delta;\alpha,\beta)$.

    \end{defn}

    \begin{rem}
    \label{rem:C-is-O1}
    As $n\to \infty$, the 4th cumulant will go to zero like $1/n$, and it is this rescalled $C^\ell$ tensor that remains order 1 as $n \to \infty$. This is why we rescale it in this way.
    \end{rem}

 Calculating the covariance that define $C^\ell$ amounts to computing expected values of neurons in each layer. This creates a virtuous cycle: by approximating the law of layer $\ell$ as a para-Gaussian, we may use the para-Gaussian law in layer $\ell$ to compute the parameters of the para-Gaussian approximation in layer $\ell+1$. 

Moreover, calculating expectations with respect to the para-Gaussian law is also easy, because the density of the para-Gaussian density \eqref{eq:paraGaussian-density} factors as a Gaussian times $\times \big(1+\text{(Hermite terms)} \big)$, which means that any para-Gaussian expectation is actually a Gaussian expectation with a slightly modified integrand, namely for any test function $f$, one has from the definition \eqref{eq:paraGaussian-density} that para-Gaussian expectations in general satisfy

$$ \mathbb{E}_{Z \sim \pgdk{\Sigma}\kappa n k }[ f(Z) ] = \GaussExp{Z \sim \mathcal{N}(0,\Si)}\left[ f(Z)\left( 1 + \frac{1}{\sqrt{n}^{k-2}} \sum_{\substack{k_1,\ldots,k_d\\ k_1+k_2+\ldots+k_d = k}} \frac{\kappa_{k_1, k_2,\ldots,k_d}}{k_1!k_2!\cdots k_d!} H_{k_1,\ldots,k_p}(Z;\Sigma) \right) \right],$$
where $\GaussExp{Z^\ell \sim \mathcal{N}(0,\Si)}$ denotes the expectation where $Z^\ell$ has a \emph{purely Gaussian law}. We use the (somewhat unsual) notation $\GaussExp{}$ for this purely Gaussian measure\footnote{Other authors, for example \cite{Hanin2024}, use notations like $\langle\cdot\rangle$ for this kind of pure Gaussian expectation.}.

As a final source of serendipity, there is an additional exchangability between various neurons in a given layer, which further simplifies the bookeeping needed. (To wit, the quantities we track $\bar \Si^\ell$ and $C^\ell$ are indexed by the set of inputs $A$ only, and the full distribution among inputs and neurons $A\times [n]$ are derived from these.) For example, by the exchangability of neurons in any layer, from the definition of $\Si$, the $C^\ell$ can be written as

    \begin{equation}
    \label{eq:C-as-exchangable}
        C^\ell(\alpha,\beta;\gamma,\delta) = \cov\left[ \vp(z_{\alpha,1}^{\ell-1})\vp(z_{\beta,1}^{\ell-1}),\vp(z_{\gamma,1}^{\ell-1})\vp(z_{\delta,1}^{\ell-1})\right] + (n-1)\cov\left[ \vp(z_{\alpha,1}^{\ell-1})\vp(z_{\beta,1}^{\ell-1}),\vp(z_{\gamma,2}^{\ell-1})\vp(z_{\delta,2}^{\ell-1})\right]
    \end{equation}
With these simplification in hand, and taking into account the \emph{size} of the errors leads to the following theorem that records these recursion relations for $\bar \Si^\ell$ and $C^\ell$ that allow us to compute the para-Gaussian approximation to error size $O(1/n^2)$. The details of this are carried out in Section \ref{sec:proof-main} and the final recursion for $\bar\Si^\ell$ and $C^\ell$ are stated in the below theorem.

 \begin{thm}
		\label{thm:main}
		%Let layer $0$ be Gaussian, $(z^0_{\al,k})\sim N(0,I_n\otimes K)$. For every fixed $\ell\ge 0$, the law of the random vector $(z^\ell_{\al,k})_{\al\in A, k\in [n]}$ is asymptotically \pg\ on $\bR^{n \abs{A}}$ of order $4$, with the following recursive descritpion of its covariance and 4th cumulant.

        Under Assumptions \ref{ass:nonlinearity} and \ref{ass:non-degen}, suppose we are given the average covariance matrix $\bar\Si^{\ell}$ and the tensor $C^\ell$ in layer $\ell$. Then the average covariance $\bar \Si^{\ell+1}$ and the quantity $C^{\ell+1}$ obey the following recursions:
        
        {\underline{Mean Covariance.}} 
        
        \begin{align}
        \label{eq:mainSigmaExpansion}
        \bar\Si^{\ell+1}(\al,\be)
        = \GaussExp{}\ab{\vp^{\ell}_\al\vp^{\ell}_\be}
          + \frac{1}{8n}\sum_{\ga_1,\ga_2,\ga_3,\ga_4\in A} C^{\ell}(\ga_1,\ga_2;\ga_3,\ga_4)\,
            \GaussExp{}\ab{\vp^{\ell}_\al\vp^{\ell}_\be\,H^{\ell}_{\ga_1\ga_2\ga_3\ga_4}}
          + \mathcal{O}\left(\frac{1}{n^2}\right),
        \end{align}
        
       {\underline{Fourth Cumulant.}} To order $\mathcal{O}(\frac{1}{n})$, the $C^{\ell+1}$ satisfies the recursion relation

        %\begin{align}
        %\label{eq:mainC}
        %C_\ell(\al,\al';\be,\be') &= \cov\big(\vp_\al\vp_{\al'},\,\vp_\be\vp_{\be'}\big)
        % + \frac14\sum_{\ga_1,\ga_2,\ga_3,\ga_4\in A} C_{\ell-1}(\ga_1,\ga_2;\ga_3,\ga_4)\,
        %   \bE\ab{\vp_\al\vp_{\al'}\,H_{\ga_1\ga_2}}\,\bE\ab{\vp_\be\vp_{\be'}\,H_{\ga_3\ga_4}}.
        %\end{align}

        \begin{align}
\label{eq:mainC}
C^{\ell+1}(\al,\al';\be,\be')
&= \GaussExp{}\ab{\vp^{\ell}_\al\vp^{\ell}_{\al'}\vp^{\ell}_\be\vp^{\ell}_{\be'}}
 - \GaussExp{}\ab{\vp^{\ell}_\al\vp^{\ell}_{\al'}}\,
   \GaussExp{}\ab{\vp^{\ell}_\be\vp^{\ell}_{\be'}} \notag\\
&\quad + \frac14\sum_{\ga_1,\ga_2,\ga_3,\ga_4\in A} C^{\ell}(\ga_1,\ga_2;\ga_3,\ga_4)\,
   \GaussExp{}\ab{\vp^{\ell}_\al\vp^{\ell}_{\al'}\,H^{\ell}_{\ga_1\ga_2}}\,
   \GaussExp{}\ab{\vp^{\ell}_\be\vp^{\ell}_{\be'}\,H^{\ell}_{\ga_3\ga_4}} + \mathcal{O}\left(\frac{1}{n}\right),
\end{align}
where:
\begin{itemize}
  \item $Z^{\,\ell}_1=(z^{\ell}_{\ga,1})_{\ga\in A} \in \mathbb{R}^{|A|}$ is the vector of outputs of neuron $1$ of
        layer $\ell$ over all inputs (by exchangeability any single neuron index $k$ has the same  distribution).
  \item $\GaussExp{}$ denotes expectation treating $ Z^{\,\ell}_1$ as \emph{pure Gaussian} with
        mean $0$ and covariance $\bar \Si^{\ell}$.
  \item $\vp^{\ell}_\al := \vp\rb{z^{\ell}_{\al,1}}$ for $\al\in A$.
\item $H^{\ell}_{\ga_1\ga_2\ga_3\ga_4} := H_{\vec e_{\ga_1}+\vec e_{\ga_2}+\vec e_{\ga_3}+\vec e_{\ga_4}}
        \rb{Z^{\,\ell}_1;\bar \Si^{\ell}}$ is the fourth-order Hermite polynomial from \eqref{eq:H_is_iu_to_the_n}. 
        Equivalently, these expectations can also be evaluated by Gaussian integration by parts 
        $\GaussExp{}\ab{F\,H^{\ell}_{\ga_1\ga_2\ga_3\ga_4}} = \GaussExp{}\ab{\partial_{\ga_1}\partial_{\ga_2}\partial_{\ga_3}\partial_{\ga_4}F}$ (see Lemma \ref{lem:gaussian-ibp} for details).
  \item $H^{\ell}_{\ga_1\ga_2} := H_{\vec e_{\ga_1}+\vec e_{\ga_2}}\rb{Z^{\,\ell}_1;\bar \Si^{\ell}}$
        is the second-order Hermite polynomial , where $\vec e_\ga\in\bZ_{\geq0}^{A}$ has a $1$ in
        position $\ga$ and $0$ elsewhere. Explicitly, writing $\Si=\bar\Si^{\ell}$ to make the inverse clear, this is
        \begin{align*}
        H^{\ell}_{\ga_1\ga_2}
          = \rb{\Si^{-1} Z^{\,\ell}_1}_{\ga_1}\rb{\Si^{-1} Z^{\,\ell}_1}_{\ga_2}
            - \rb{\Si^{-1}}_{\ga_1\ga_2},
        \end{align*}
        which is valid also when $\ga_1=\ga_2$.
\end{itemize}

\end{thm}

The proof of Theorem \ref{thm:main} is presented in Section \ref{sec:proof-main}.

\subsection{Related work}
\label{sec:literature}

The study of wide random neural networks at initialization starts from the observation that, as the width $n$ tends to
infinity with the depth fixed, the neurons converge to a Gaussian process, sometimes refered to as the neural network Gaussian process (NNGP) \cite{Lee2018NNGP,Matthews2018,Hanin2023GP}. Our work is concerned
with what happens just beyond this limit: the corrections of order $\frac1n$, how to describe them, and how accurately they
describe the network. In this section we explain how our results relate to the existing literature on this question, which
comes from several directions.

\subsubsection{Perturbative descriptions of finite-width networks}

The most closely related line of work describes finite-width networks perturbatively in $\frac1n$. At a physics level of rigor,
Yaida \cite{Yaida2020} was the first to obtain a recursion in the depth for the fourth cumulant of a single neuron, and to study its
behaviour for deep networks. These ideas were developed much further in the monograph of Roberts, Yaida and Hanin
\cite{RobertsYaidaHanin2022}, which derives, again at a physics level of rigor, recursions for the four-point vertex (which plays the
role of our $C^\ell$) and for the $\frac1n$ correction to the covariance with several inputs, as well as for
many other objects, such as the neural tangent kernel and its evolution during training. 

Hanin \cite{Hanin2024} then gave a
mathematically rigorous account of this picture. He showed that the $2k$-th cumulants of the neurons are of order $n^{1-k}$
\cite[Theorem 3]{Hanin2024}, derived a $\frac1n$ expansion for the expectations of smooth or continuous test functions
\cite[Theorem 4]{Hanin2024}, showed that the cumulants form a hierarchy determined recursively from layer to layer
\cite[Corollary 5]{Hanin2024}, and computed explicitly the resulting recursions for a single input \cite[Corollary 6]{Hanin2024}.
Indeed, many of our proofs rely on several of the detailed analysis and estimates developed in \cite{Hanin2024}.

Our contributions relative to this line of work are described below.

\paragraph{Describing the distribution of the network with Hermite polynomials} The perturbative approaches above describe the network through its cumulants and/or
through expectations of test functions. We instead explictly describe the \emph{law} of the network as a para-Gaussian law.  Theorem \ref{thm:TV-gauss-pg}
says that the joint distribution of finitely many neurons is close, in total variation distance, to an explicit signed measure,
the \pg\ law, whose density is a Gaussian density multiplied by a correction made of fourth order Hermite polynomials. 

Our point of view has two practical consequences. First, the approximation can be compared directly with Monte Carlo simulations at
the level of densities, as in Figures \ref{fig:intro-not-quite-gaussian}, \ref{fig:linear-one-neuron-two-inputs} and
\ref{fig:relu-one-neuron-two-inputs}. Second, it yields statements for arbitrary bounded test functions, including indicator
functions of sets, with errors that depend on the test function only through its supremum. This is because in our approach derivatives of the test function instead become expectations with the Hermite polynomials (see Lemma \ref{lem:gaussian-ibp}
and Remark \ref{rem:G-derivatives-hermite}).

We also believe this way of thinking in terms of para-Gaussian distributions is useful to the human understanding of this area. This description in terms of Hermite polynomials can be more easily conceptualized than more complicated formulas involving derivative operators, even in cases where the two formulas mathematically coincide. We are hopeful that the exposition we provide here will be useful to readers.

\paragraph{Explicit formulas for the cumulants of all orders in terms of the random matrix $\Si$.} It is stated in \cite[Proposition 15]{Hanin2024},  that the cumulants of the neurons can be expressed through ``finite sums of'' cumulants of the random covariance matrix $\Si^\ell$.
Theorem \ref{thm:exact-general-cumulants} makes this expression completely explicit: for every width $n$, the $2p$-th cumulant of
the neurons is a sum over the $(2p-1)!!$ pairings of the indices, each weighted by Kronecker deltas in the neuron indices and by
a $p$-th cumulant of entries of $\Si^\ell$. This formula makes visible the structure of the dependence between different neurons.
Neurons in the same layer are uncorrelated, but they are not independent, and the Kronecker deltas $\delta_{j_i,j_k}$ from Theorem \ref{thm:exact-general-cumulants} create higher order non-trivial cumulants. We give two proofs of this, one slick proof with the law of total cumulance and one directly from the definition of cumulants via generating functions.

\paragraph{Recursions with several inputs and general activation functions.} The recursions of Theorem \ref{thm:main} for the
mean covariance $\bar\Si^\ell$ and for $C^\ell$ hold for any set of inputs and any activation function satisfying the
assumptions of \cite{Hanin2024}. This generalizes the single-input recursions of \cite[Corollary 6]{Hanin2024} to several inputs. Written
with Hermite polynomials, the coefficients of these recursions are explicit Gaussian integrals, which in the ReLU case we evaluate numerically (see Section \ref{sec:relu-application}).

\paragraph{Exact solutions in special cases.} For activation functions that are positively homogeneous of degree one, such as the identity and
ReLU, the law of a \emph{single} neuron on a \emph{single} input is known in closed form: it is a Gaussian with an independent
random variance given by a product of independent weighted chi-squared random variables \cite{HaninNica2020b}, which can
also be described with Meijer $G$-functions \cite{ZavatoneVethPehlevan2021}. In particular, the exact single-input formulas
$C^\ell=n[(1+\frac2n)^\ell-1]$ for linear networks and $C^\ell=n[(1+\frac5n)^\ell-1]$ for ReLU networks (which we state in Lemma
\ref{lem:relu-one-input}) are already known. 

A new result is the exact solution for linear networks with
\emph{several inputs and several neurons}, Lemma \ref{lem:linear-explicit}, which is valid for every $n$ and every depth. It shows
that, beyond the two pairings $K(\al,\be)K(\al',\be')+K(\al,\be')K(\al',\be)$ predicted by the $\frac1n$ recursion, the exact
fourth cumulant contains a third pairing $K(\al,\al')K(\be,\be')$ with a coefficient $w_\ell$ that is invisible at first order
but becomes important once $\ell/n$ is of order one (Figure \ref{fig:linear-exact-vs-first-order}).

\subsubsection{Deep networks and the depth-to-width ratio}

The non-Gaussian corrections studied in this paper are small in a single layer, of order $\frac1n$, but they accumulate as the
signal passes through the network. A useful way to see this is through the fourth cumulant of a single neuron. In a ReLU network
with one input, for example, it is $\frac3n C^\ell$ with $C^\ell=5\ell+O(\frac1n)$ (Lemma \ref{lem:relu-one-input}): each layer adds a
roughly equal amount of non-Gaussianity, so that after $\ell$ layers the correction is of order $\frac{\ell}{n}$ rather than
$\frac1n$. The relevant small parameter for a deep network is therefore not the inverse width $\frac1n$ alone, but the
depth-to-width ratio $\frac{\ell}{n}$.

When $\frac{\ell}{n}$ is small, the \pg\ law is an accurate description, and this is the regime of the present paper. When
$\frac{\ell}{n}$ is of order one, every order of the expansion in $\frac1n$ contributes at the same size, and a description by
finitely many cumulants is no longer adequate. In this regime, the network has been studied by taking $n$ and $\ell$ to infinity
together with $\frac{\ell}{n}$ fixed. For a single input, the variance of a neuron then becomes log-normal rather than
deterministic \cite{HaninNica2020b}. With
two inputs, the angle between their representations evolves as a Markov chain whose behaviour is governed by $\frac{\ell}{n}$
\cite{JakubNica2024}.

The exact solution for linear networks (Lemma \ref{lem:linear-explicit}) shows clearly how the two regimes fit together in that special case. The exact
coefficients are built from powers such as $(1+\frac2n)^\ell$, which are approximately $e^{2\ell/n}$. The first order recursion
keeps only the linear term $1+\frac{2\ell}{n}$ of this exponential. This is accurate when $\frac{\ell}{n}$ is small and fails once
$\frac{\ell}{n}$ is of order one, as Figure \ref{fig:linear-exact-vs-first-order} shows. The same behaviour appears in ReLU
networks, where the $\frac1n$ recursion underestimates the cumulants once $\frac{5\ell}{n}$ is no longer small (Figures
\ref{fig:relu-cumulants-vs-depth} and \ref{fig:relu-two-neurons-two-inputs}). In this sense the \pg\ law is the first non-trivial
approximation, valid for small $\frac{\ell}{n}$, to the non-Gaussian behaviour that the works above describe at finite
$\frac{\ell}{n}$.

\subsubsection{Generalization of earlier work by the authors}

This paper extends the ideas from the workshop paper \cite{NicaOrtmann2024M3L}, which introduced the \pg\ distribution through its
characteristic function, for a single neuron in a network with one hidden layer and at most two inputs, and computed the
resulting density with Hermite polynomials. The present paper treats networks of arbitrary depth, any finite number of inputs and
any finite number of neurons, including the dependence between different neurons. We prove that the approximation holds in total
variation with rate $n^{-2}$ in Theorem \ref{thm:TV-gauss-pg} and gives the recursions that propagate the parameters of the \pg\ law from layer to layer in Theorem \ref{thm:main}. 

The
relationship between the normalization $\frac1{8}$ used in \cite{NicaOrtmann2024M3L} and the Edgeworth normalization $\frac1{4!}$
used here is explained in Corollary \ref{cor:PG-two-forms}. The Hermite polynomial machinery we use throughout, in particular the
Gaussian integral formula for multivariable Hermite polynomials, was developed in the expository paper on Hermite polynomials \cite{NicaOrtmann2026}.

\subsubsection{Edgeworth Expansions for Multivariate Neural Networks}
After the first version of this paper was posted, we were made aware of recent work by Celli \cite{Celli2026}, which establishes total
variation bounds of order $n^{-m}$, for every $m\in\mathbb N$, between the neurons and a multivariate Edgeworth expansion of order
$4m-1$, together with matching lower bounds. For $m=2$ this gives the same rate $n^{-2}$ as part (ii) of Theorem \ref{thm:TV-gauss-pg},
and the proof follows a similar Taylor expansion along an interpolation of the covariance.

While both expansions approximate the same law, and agree with each other up to an error of order
$n^{-m}$ when truncated at the same order, they are not the same term by term! The expansion of \cite{Celli2026} is centred at the
infinite width kernel $K^\ell$ (see \eqref{eq:K-def}) and its coefficients are moments of the whitened matrix
$\sqrt{K^\ell}^{-1}(\Si^\ell-K^\ell)\sqrt{K^\ell}^{-1}$. Since $\bar\Si^\ell=K^\ell+\frac1n\si^\ell+O(n^{-2})$ with $\si^\ell\neq0$ in general (see \eqref{eq:sigma-correction}), this expansion
contains a non-zero order 2 Hermite correction with coefficient  $\frac1n$, which accounts for this shift of the covariance from $K^\ell$ to $\bar\Si^\ell$. We
instead centre the expansion at the mean covariance $\bar\Si^\ell$ itself. The first order term then vanishes identically (Lemma
\ref{lem:g-prime-mean-zero}), and the only correction at order $\frac1n$ is the fourth order Hermite term, whose coefficients are the
covariances $C^\ell$ of the random matrix $\Si^\ell$. In contrast, each Hermite term of the expansion around $K^\ell$ in
\cite{Celli2026} mixes the centred moments of $\Si^\ell$ with powers of the mean shift $\bar\Si^\ell-K^\ell$.

In both expansions, the coefficients are moments of the random covariance $\Si^\ell$, which are not explicit for deep networks with
several inputs. Our Theorem \ref{thm:main} gives recursions to compute these moments in terms of two simple objects, $\bar\Si^\ell$ and $C^\ell$, to the order
needed to generate approximations to deep non-linear networks as in Figures \ref{fig:relu-cumulants-vs-depth}--\ref{fig:relu-one-neuron-two-inputs}.

    %\section{Main theoretical results}

    %\subsection{Conditional Gaussian law}

    %\begin{defn}

     %For every layer $\ell$, define $\Si^\ell: A \times A \to \mathbb{R}$ be the random matrix, measurable with respect to layer $\ell-1$,
     %$$\Si^\ell(\alpha,\beta) := \frac{1}{n} \sum_{j=1}^{n} \varphi\rb{z^{\ell-1}_{\alpha,j}} \varphi\rb{z^{\ell-1}_{\beta,j}}$$
     % \end{defn}

	% {\bf{TODO:}} Should we move Prop \ref{prop:PhiAsY} and the definition of $Y$ here?

\subsection{AI Disclosure}

Claude Opus 5.5 was used as a writing assistant throughout. The slick proof of Theorem \ref{thm:exact-general-cumulants} using the law of total cumulance was also discovered by Claude. (Our original human proof is included in Appendix \ref{subsec:second-proof-exact-general-cumulants}.) Claude also wrote the code to run the Monte Carlo experiments and generate  the figures in the paper.

    \section{More Results}

    In addition to the two main theorems, Theorem \ref{thm:TV-gauss-pg}
 and Theorem \ref{thm:main}, in this section we develop a general formula for the cumulants of the neurons in terms of the cumulants of the random matrix $\Si$ in Theorem \ref{thm:exact-general-cumulants}. We also apply our theory to linear neural networks in Section \ref{sec:linear-application} and ReLU networks in Section \ref{sec:relu-application}.

    \subsection{Exact formulas for higher cumulants}
    \label{sec:higher_cumulants}

    \begin{thm}
    \label{thm:exact-general-cumulants}
   All the odd cumulants of neurons are 0, and the $2p$-cumulants of the neuron activation in layer $\ell$ can be written in terms of $\Si^\ell$ as follows:
% In general, the $2p$-th cumulant for $2p$ neurons (not necessarily distinct) is given in terms of the $p$-th cumulants of $\Si^\ell$, namely
     \begin{align*}
    \kappa\left( \left\{ z_{\alpha_i,j_i}^\ell \right\}_{i=1}^{2p}  \right) & = \sum_{\{\{a_1,b_1\},\ldots,\{a_p,b_p\}\} \in \mathcal{P}(2p)} \left( \prod_{k=1}^p \delta_{j_{a_k},j_{b_k}} \right) \kappa \left( \left\{ \Si^\ell(\alpha_{a_k},\alpha_{b_k})\right\}_{k=1}^p \right),
    \end{align*}
where the sum $\{\{a_1,b_1\},\ldots,\{a_p,b_p\}\} \in \mathcal{P}(2p)$ is over all of the $\abs{\mathcal{P}(2p)}=(2p-1)!!$ possible pairings of the indices $\{{1,\ldots,2p}\}$, more precisely $\{\{a_1,b_1\},\ldots,\{a_p,b_p\}\} \in \mathcal{P}(2p)$ if and only if $\bigcup_{i=1}^{p} \{a_i, b_i\} = \{1,\ldots,2p\}$ and $\{a_i, b_i\} \cap \{a_j, b_j\} = \emptyset \text{ for } i \ne j$
    \end{thm}

We will give two proofs of Theorem \ref{thm:exact-general-cumulants}. The first, which is brief and relies on the law of total cumulance, is stated here. The second, which is more instructive but longer, can be found in Appendix \ref{subsec:second-proof-exact-general-cumulants}.

\begin{proof}[First proof of Theorem \ref{thm:exact-general-cumulants}]
    We will appeal to the \emph{law of total cumulance} \cite{Brillinger1969}, which states that for any finite collection of
random variables $X_1,\ldots,X_m$, and a sigma-algebra $\cF$ to condition on, one has
\begin{equation}
\label{eq:total-cumulance}
\ka\rb{X_1,\ldots,X_m}=\sum_{\pi \vdash \{1,\ldots,m\}}\ka\rb{\set{\ka\rb{\set{X_i}_{i\in B}\,\big|\,\cF}}_{B\in\pi}},
\end{equation}
where the sum is over all set partitions $\pi$ of $\set{1,\ldots,m}$ into disjoint blocks (of any size), the inner
$\ka(\,\cdot\mid\cF)$ is the \emph{conditional} joint cumulant of the variables in the block $B$ (which is itself a $\cF$-measurable random
variable), and the outer $\ka$ is the ordinary joint cumulant of these $\abs{\pi}$ random variables. (For example, when $m=2$ this
is the law of total covariance, $\cov(X,Y)=\bE\ab{\cov(X,Y\mid\cF)}+\cov\rb{\bE[X\mid\cF],\bE[Y\mid\cF]}$.)

In our case, using the sigma-algebra $\cF=\sigma\left( \left\{ z^{\ell-1}_{\al,k} \right\}_{\al \in A,k \in [n]} \right)$ of the previous layer, we see that because the random variable $z^\ell_{\alpha,k}$ conditionally Gaussian (see Remark \ref{rem:cond-gauss}), the \emph{only non-zero} conditional cumulants are the 2nd cumulant (aka covariance) which has $\kappa\left( z^\ell_{\alpha_1,k_1}, z^\ell_{\alpha_2,k_2} \given \cF \right) = \cov\left(z^\ell_{\alpha_1,k_1}, z^\ell_{\alpha_2,k_2} \given \cF \right) = \delta_{k_1,k_2} \Si^\ell(\alpha_1,\alpha_2)$. Because all other cumulants are zero for Gaussian random varaibles, the only partitions $\pi$ that have non-zero contribution in \eqref{eq:total-cumulance} are precisely the pair partitions. This shows that the odd cumulants are 0 (since there are no pair partitions) and that the even cumulants are

\begin{align*}
    \kappa\left( \left\{ z_{\alpha_i,j_i}^\ell \right\}_{i=1}^{2p}  \right) & = \sum_{\{\{a_1,b_1\},\ldots,\{a_p,b_p\}\} \in \mathcal{P}(2p)}  \kappa \left( \left\{ \delta_{j_{a_k},j_{b_k}} \Si^\ell(\alpha_{a_k},\alpha_{b_k})\right\}_{k=1}^p \right),
    \end{align*} 
whence the result follows by multi-linearity of cumulants to pull out the $\delta_{j_{a_k},j_{b_k}}$ factors.
\end{proof}

    \subsection{Linear Activation Function}
\label{sec:linear-application}

        When $\vp$ is the identity function, i.e. $\vp(x)=x$, we obtain a very nice simplification of the formulas of Theorem \ref{thm:main}. Moreover, we are able to solve the recursion equation for the cumulants exactly (not just to order $1/n$). Recall from \eqref{eq:K-def} the definition $K^0(\alpha,\beta) = \frac{1}{n_{in}} \langle x_\al,x_\be \rangle$ denotes the covariance of the input layer. For the linear network $\varphi(x)=x$ it is also clear that the infinite width limit $K^\ell(\alpha,\beta) =  K^0(\alpha,\beta)$ is constant, so we just refer to this constant as $K(\alpha,\beta)$ for the linear network.

        \begin{cor}[Linear networks: exact recursive formulas]
\label{cor:linear}
Let $\vp(x)=x$ and let layer $0$ be Gaussian with kernel $K$. Then for every $n$ and every $\ell\geq0$:

\noindent\underline{Mean Covariance.} $\bar\Si^\ell=K$ exactly for all layers.

\noindent\underline{Fourth cumulant.} The tensor $C^\ell(\al,\al';\be,\be'):=n\cov\ab{\Si^\ell(\al,\al'),\Si^\ell(\be,\be')}$
in \eqref{eq:defCell} satisfies $C^0=0$ and the \emph{exact} recursion
\begin{align}
\label{eq:linearC}
C^\ell(\al,\al';\be,\be')
&= K(\al,\be)K(\al',\be') + K(\al,\be')K(\al',\be) + C^{\ell-1}(\al,\al';\be,\be') \notag\\
&\quad + \frac1n\left(\,C^{\ell-1}(\al,\be;\al',\be') + C^{\ell-1}(\al,\be';\al',\be)\,\right).
\end{align}
\end{cor}

\begin{proof}
\underline{Mean Covariance.} Since $\vp(x)=x$, by \eqref{eq:Sigma-def-intro} we have
$\bar \Si^\ell(\al,\be)=\bE\ab{z^{\ell-1}_{\al,1}z^{\ell-1}_{\be,1}}=\bar\Si^{\ell-1}(\al,\be)$, so by induction
$\Si^\ell=\Si^0=K$ for all $\ell$.

\underline{Fourth cumulant.} Layer $0$ is Gaussian, so its fourth (and higher) cumulants vanish, consistent with $C^0=0$. By \eqref{eq:C-as-exchangable} and using $\vp(x)=x$, we have
\begin{align}
\label{eq:linear-split}
C^\ell(\al,\al';\be,\be')
= \cov\ab{z^{\ell-1}_{\al,1}z^{\ell-1}_{\al',1},\,z^{\ell-1}_{\be,1}z^{\ell-1}_{\be',1}}
 + (n-1)\cov\ab{z^{\ell-1}_{\al,1}z^{\ell-1}_{\al',1},\,z^{\ell-1}_{\be,2}z^{\ell-1}_{\be',2}} .
\end{align}
For mean-zero random variables $X,Y,U,V$ we have the identity
$\bE\ab{XYUV}=\ka(X,Y,U,V)+\bE[XY]\bE[UV]+\bE[XU]\bE[YV]+\bE[XV]\bE[YU]$, and so
\begin{align}
\label{eq:cov-of-products}
\cov\ab{XY,UV} = \bE[XU]\,\bE[YV] + \bE[XV]\,\bE[YU] + \ka(X,Y,U,V).
\end{align}
We apply \eqref{eq:cov-of-products} to both terms of \eqref{eq:linear-split}. The mean covariances of layer $\ell-1$ are
$\bar \Si^{\ell-1}=K$ (shown above) on the same neuron and $0$ across different neurons, and its fourth cumulants are
given exactly by \eqref{eq:kappa_from_C} (Theorem \ref{thm:exact-general-cumulants}).

\emph{Same neuron:} With all four variables on neuron $1$, all three $\delta$-products in
\eqref{eq:kappa_from_C} equal $1$, and so
\begin{align*}
\cov\ab{z^{\ell-1}_{\al,1}z^{\ell-1}_{\al',1},\,z^{\ell-1}_{\be,1}z^{\ell-1}_{\be',1}}
&= K(\al,\be)K(\al',\be') + K(\al,\be')K(\al',\be)\\
&\quad + \frac1n\Big(C^{\ell-1}(\al,\al';\be,\be') + C^{\ell-1}(\al,\be;\al',\be') + C^{\ell-1}(\al,\be';\al',\be)\Big).
\end{align*}

\emph{Two different neurons:} Now $\bE\ab{z_{\al,1}z_{\be,2}}=\bE\ab{z_{\al,1}z_{\be',2}}=0$, so the
covariance terms in \eqref{eq:cov-of-products} vanish, and in \eqref{eq:kappa_from_C} only the term
$\delta_{11}\delta_{22}$ survives:
\begin{align*}
\cov\ab{z_{\al,1}z_{\al',1},\,z_{\be,2}z_{\be',2}} = \frac1n\,C^{\ell-1}(\al,\al';\be,\be') .
\end{align*}

\noindent Substituting both into \eqref{eq:linear-split}, the term $\frac1n C^{\ell-1}(\al,\al';\be,\be')$ from the
same-neuron part combines with $\frac{n-1}{n}C^{\ell-1}(\al,\al';\be,\be')$ from the cross-neuron part to give
$C^{\ell-1}(\al,\al';\be,\be')$, and \eqref{eq:linearC} follows. 
\end{proof}

\begin{lem}[Linear networks: explicit solution to 4th order cumulant]
\label{lem:linear-explicit}
Recall that the 4th order cumulant tensor is given in terms of the $C^\ell$ which satisfies the recursion given in Corollary \ref{cor:linear} and that the mean covariance is $\bar\Sigma^\ell = K$ exactly. The exact solution for the $C^\ell$ is

\begin{align}
\label{eq:linearCsolution}
C^\ell(\al,\al';\be,\be') = v_\ell \Big( K(\alpha,\beta)K(\alpha',\beta')+K(\al,\be')K(\al',\be) \Big) + w_\ell\,K(\al,\al')K(\be,\be'),
\end{align}
where the coefficents $v_\ell\in \mathbb{R},w_\ell\in \mathbb{R}$ are given explicitly by

\begin{equation}
\label{eq:wv_formula}
v_\ell := \frac n3\left[\rb{1+\frac2n}^{\ell}-\rb{1-\frac1n}^{\ell}\right],\quad\quad
w_\ell := \frac n3\left[\rb{1+\frac2n}^{\ell}+2\rb{1-\frac1n}^{\ell}-3\right].
\end{equation}

\end{lem}

\begin{proof}
The proof goes by showing that the form of the solution \eqref{eq:linearCsolution} is preserved when the scalars $v_\ell$, $w_\ell$ satisfy the recursion \eqref{eq:certain-recursion} below. This recursion is then solved to yield \eqref{eq:wv_formula}.

For notational convenience, let us define $P_1=K(\al,\be)K(\al',\be')$, $P_2=K(\al,\be')K(\al',\be)$ and $P_3=K(\al,\al')K(\be,\be')$ to be the the three possible pairings. 

Substituting the ansatz \eqref{eq:linearCsolution} for
$C^{\ell-1}$ into the two permuted terms of \eqref{eq:linearC} gives
$C^{\ell-1}(\al,\be;\al',\be')=v_{\ell-1}(P_2+P_3)+w_{\ell-1}P_1$ and $C^{\ell-1}(\al,\be';\al',\be)=v_{\ell-1}(P_1+P_3)+w_{\ell-1}P_2$,
so the form of the solution is preserved by \eqref{eq:linearC} when $v_\ell$ and $w_\ell$ satisfy
\begin{align}
\label{eq:certain-recursion}
v_\ell=1+v_{\ell-1}+\tfrac1n\big(v_{\ell-1}+w_{\ell-1}\big),\qquad w_\ell=w_{\ell-1}+\tfrac2n\,v_{\ell-1},\qquad v_0=w_0=0.
\end{align}
This can be solved by doing the linear combinations $s_\ell:=2v_\ell+w_\ell$ and $d_\ell:=v_\ell-w_\ell$. Notice then that  $s_\ell=2+(1+\frac2n)s_{\ell-1}$ and
$d_\ell=1+(1-\frac1n)d_{\ell-1}$, so $s_\ell=n[(1+\frac2n)^\ell-1]$ and $d_\ell=n[1-(1-\frac1n)^\ell]$. Solving
$v_\ell=\frac13(s_\ell+d_\ell)$ and $w_\ell=\frac13(s_\ell-2d_\ell)$ gives \eqref{eq:wv_formula}.
\end{proof}

\begin{rem}
Note that for fixed $\ell$, we have $v_\ell = \ell + O(1/n)$ and $w_\ell = 0 + O(1/n)$, which gives the linearized $$C^\ell(\al,\al';\be,\be') = \ell \Big( K(\al,\be)K(\al',\be') +K(\al,\be')K(\al',\be) \Big) + O(1/n).$$ This also could have been obtained from the recursion in Theorem \ref{thm:main}. We compare the exact values from the solution in Lemma \ref{lem:linear-explicit} to this $O(1/n)$ approximation to Monte Carlo simulations for various aspects of the network below.
\end{rem}

\begin{figure}[H]
    \centering
    \includegraphics[width=\linewidth]{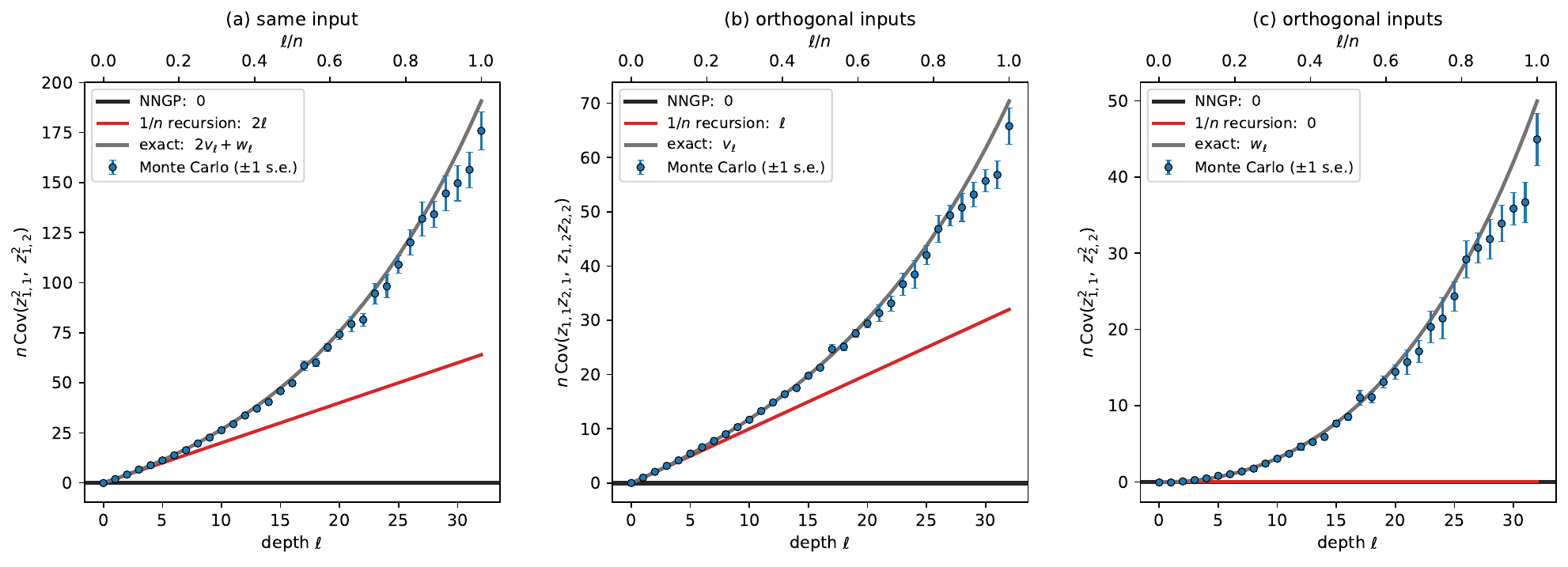}
    \caption{\footnotesize \textbf{Linear network. Fourth cumulants.} Quantities related to the fourth cumulant of a linear network
    ($\vp(x)=x$, width $n=32$) for depths up to $\ell=32$. Each plot compares three solutes: 1. the exact values from Lemma \ref{lem:linear-explicit}, 2. the values obtained
    from the $1/n$ recursion of Theorem \ref{thm:main}, and 3. from Monte Carlo simulations. We use two orthogonal unit inputs ($K=I_2$) and two different neurons, so that each statistic isolates one exact coefficient; here $z_{\al,k}$ denotes the output of neuron $k$ on
    input $\al$, with $\al \in \{1,2\}$ and $k \in \{1,2\}$. The plots show
    (a) Same input: $n\cov\big(z_{1,1}^2,z_{1,2}^2\big)=C^\ell(1,1;1,1)$, exactly $2v_\ell+w_\ell=n\big[(1+\frac2n)^\ell-1\big]$ and
    $2\ell$ according to the $1/n$ recursion.
    (b) Orthogonal inputs: $n\cov\big(z_{1,1}z_{2,1},\,z_{1,2}z_{2,2}\big)=C^\ell(1,2;1,2)$, exactly $v_\ell$ and $\ell$ according to the
    $1/n$ recursion.
    (c) Orthogonal inputs: $n\cov\big(z_{1,1}^2,z_{2,2}^2\big)=C^\ell(1,1;2,2)$, exactly $w_\ell$ and $0$ according to the $1/n$
    recursion.
    Grey: exact solution for $C^\ell$ from Lemma \ref{lem:linear-explicit}; red: $O(1/n)$ approximation from the recursion of
    Theorem \ref{thm:main}; blue: Monte Carlo over $2^{16}$ networks, with $\pm1$ standard error computed from batch means. The pure neural network
    Gaussian (NNGP) approxiamtion predicts $0$ for all these quantities. Note that the $1/n$ recursion is accurate while $\ell/n$ is small but gets worse when $\ell/n$ is of order
    one. Additionally, the coefficient $w_\ell$ of the third pairing $K(\al,\al')K(\be,\be')$ is of order $1/n$, so this pairing
    contributes to the fourth cumulant $\ka=\frac1nC^\ell$ only at order $n^{-2}$ and is therefore absent from the $1/n$ recursion but
    present in the exact solution and in the simulations. }
    \label{fig:linear-exact-vs-first-order}
\end{figure}
\begin{figure}[H]
    \centering
    \includegraphics[width=\linewidth]{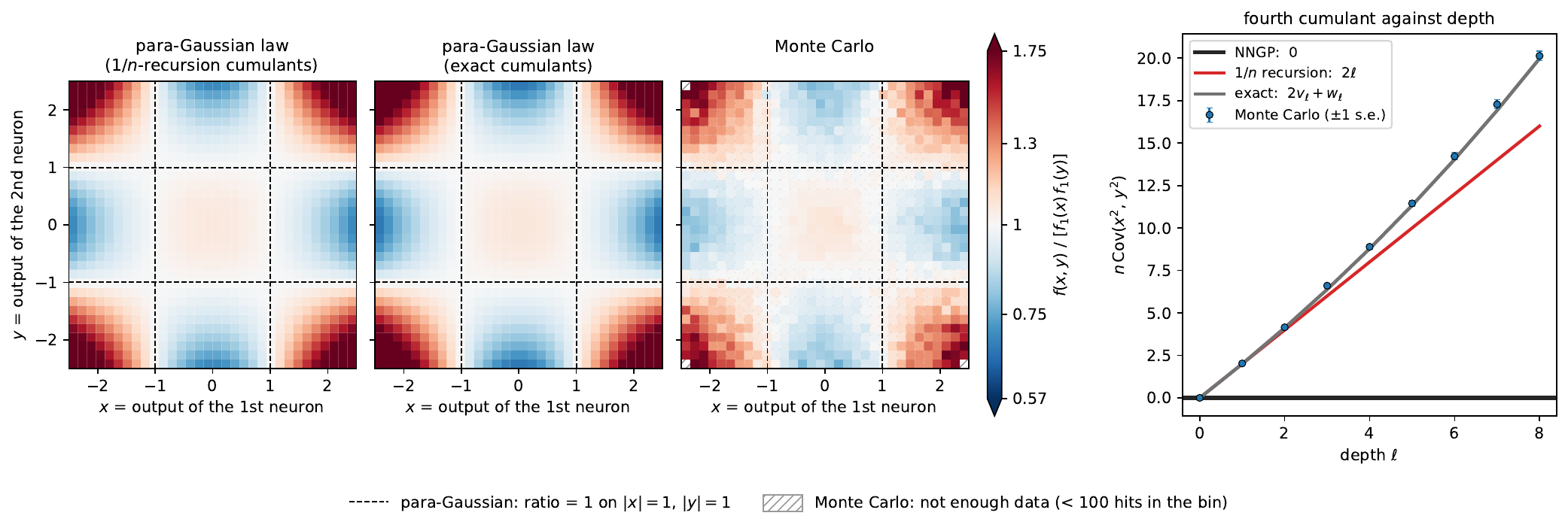}
    \caption{\footnotesize\textbf{Linear network. Two neurons outputs, one input. Uncorrelated but not independent.} The distribution of two different neurons of a linear network
    ($\vp(x)=x$, width $n=32$) receiving the same unit norm input, i.e. $A={\alpha}$ and we look at the joint distribution of the random variables $X=z_{\alpha,1},Y=z_{\alpha,2}$. These random variables are uncorrelated, but not independent. \textbf{Left} three panels: the ratio $f(x,y)/\big[f_1(x)f_1(y)\big]$ of
    the joint density of the two neurons' outputs to the product of their marginal densities at depth $\ell=4$, which the NNGP
    predicts to be identically $1$. The para-Gaussian law predicts this as $1+\frac{\ka}{4}(x^2-1)(y^2-1)$, where $\ka=\ka(x,x,y,y)$ is the
    cross-neuron fourth cumulant. We show both the result $\ka=\frac{2\ell}{n}$ from the $1/n$-approximated-recursion and with the exact value
    $\ka=\frac{2v_\ell+w_\ell}{n}$, next to Monte Carlo over $2^{22}$ neuron pairs. \textbf{Right}: the cumulant $n\cov(x^2,y^2)$ against depth, with the $1/n$ recursion
    $2\ell$, the exact value $2v_\ell+w_\ell$ and Monte Carlo ($\pm1$ standard error).}
    \label{fig:linear-two-neurons-one-input}
\end{figure}

\begin{figure}[H]
    \centering
    \includegraphics[width=0.85\linewidth]{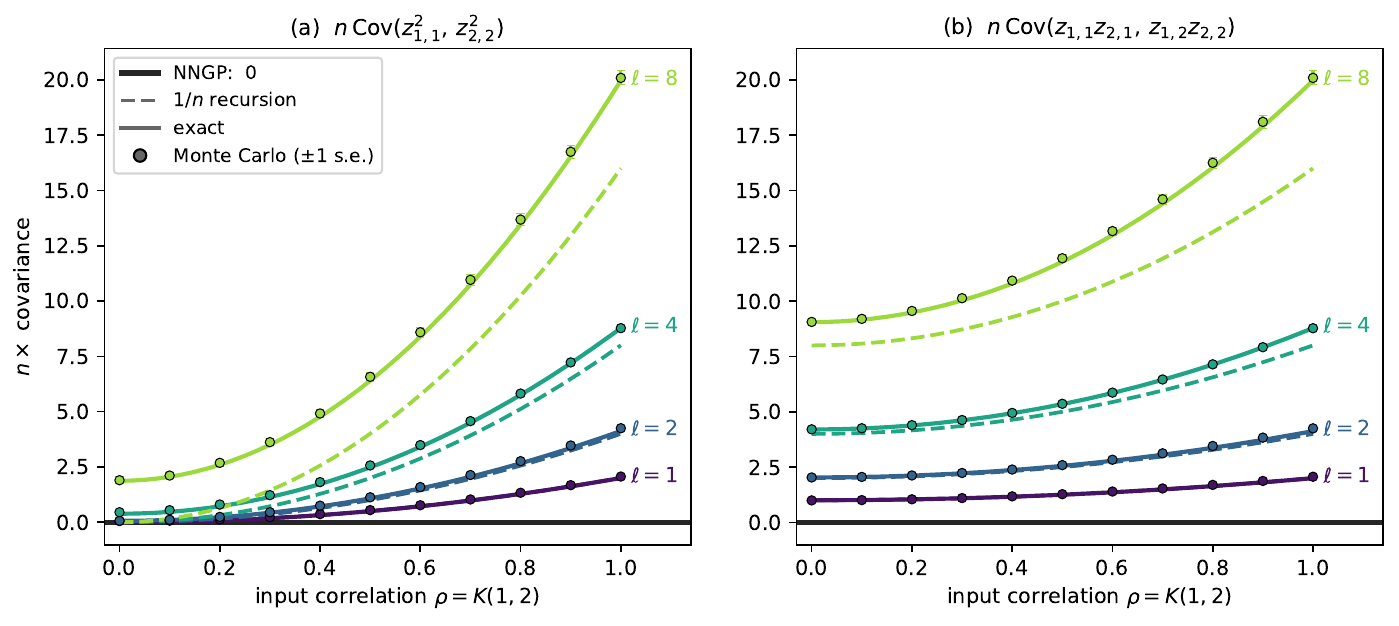}
    \caption{\footnotesize\textbf{Linear network. Two neurons, two inputs. Dependence on the input correlation.} Cross-neuron fourth cumulants of a linear
    network ($\vp(x)=x$, width $n=32$) for two unit norm inputs with correlation $\rho=K(1,2)$, against $\rho$ for depths
    $\ell=1,2,4,8$; here $z_{\al,k}$ denotes the output of neuron $k$ on input $\al$ where $\al \in \{1,2\}$ and $k\in \{1,2\}$.
    \textbf{(a)} $n\cov\big(z_{1,1}^2,z_{2,2}^2\big)=C^\ell(1,1;2,2)$, the dependence between the first neuron on input $1$ and the second neuron
    on input $2$: $2\ell\rho^2$ according to the $1/n$ recursion, exactly $2\rho^2v_\ell+w_\ell$.
    \textbf{(b)} $n\cov\big(z_{1,1}z_{2,1},\,z_{1,2}z_{2,2}\big)=C^\ell(1,2;1,2)$, the dependence between the products of each neuron's outputs
    on the two inputs: $\ell(1+\rho^2)$ according to the $1/n$ recursion, exactly $(1+\rho^2)v_\ell+\rho^2w_\ell$.
    Dashed lines: $1/n$ recursion of Theorem \ref{thm:main}; solid lines: exact solution from Lemma \ref{lem:linear-explicit};
    markers: Monte Carlo over $2^{22}$ neuron pairs per value of $\rho$ ($\pm1$ standard error). The Gaussian prediction (NNGP) predicts $0$ in all cases. 
    At $\rho=0$ the $1/n$ recursion predicts no dependence in (a), while the exact cumulant $w_\ell$ is positive and clearly
    visible at $\ell=8$.}
    \label{fig:linear-two-neurons-two-inputs}
\end{figure}
\begin{figure}[H]
    \centering
    \includegraphics[width=\linewidth]{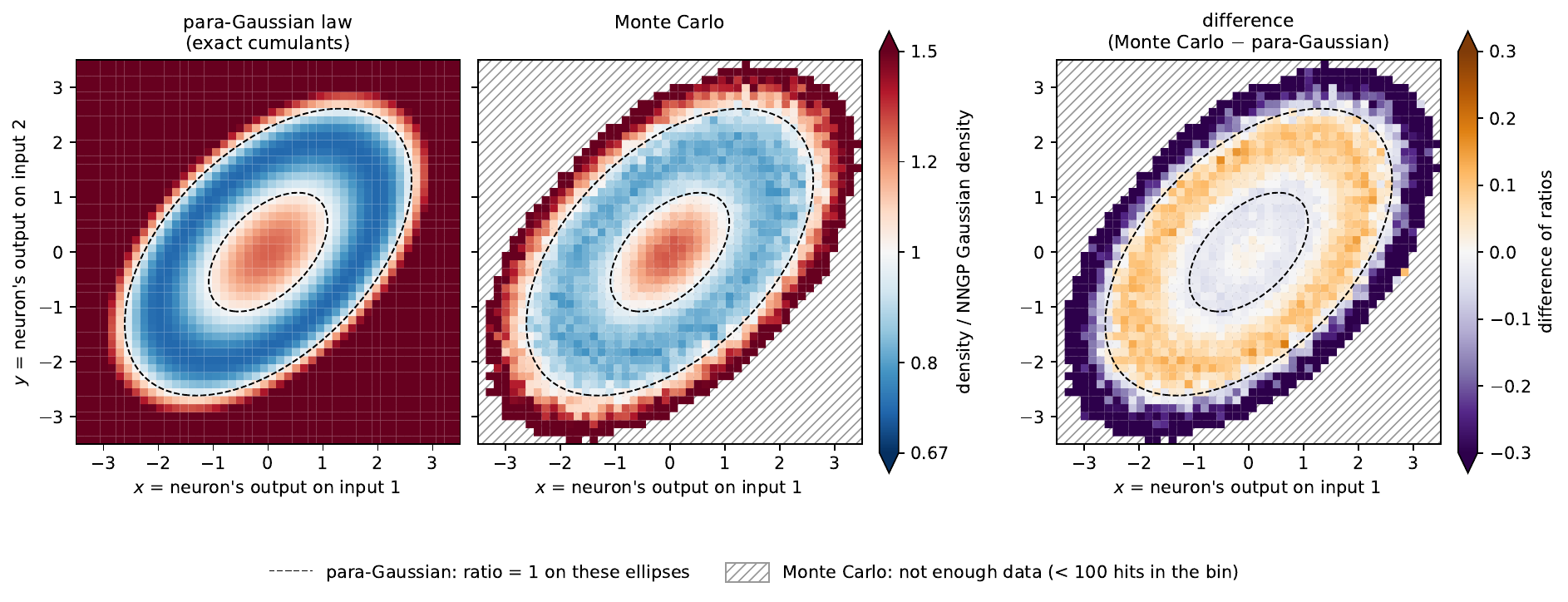}
    \caption{\footnotesize \textbf{Linear network. One output neuron, two inputs. The joint density.} Representation of the joint density of the random
    vector $\big(z^\ell_{1,1},z^\ell_{2,1}\big)$ consisting of the outputs of the first neuron of layer $\ell=4$ for two unit norm inputs
    $A=\{1,2\}$ with correlation $K^{0}(1,2)=0.5$, in a linear network ($\vp(x)=x$, width $n=32$). The left and middle panels show the
    {\bf \emph{ratio of the density}} to the Gaussian (NNGP) predicted density $\mathcal{N}(0,K^{0})$ (numbers $>1$ indicate a higher
    density, and numbers $<1$ a lower density, than the pure Gaussian prediction). \textbf{Left:} the para-Gaussian density of
    Definition \ref{def:PG}, with covariance $K^0$ and fourth cumulant tensor $\ka^\ell$ for a single neuron
    ($i=j=k=l=1$), where $C^\ell$ is the exact solution \eqref{eq:linearCsolution} with $v_\ell,w_\ell$ given by \eqref{eq:wv_formula}.
    \textbf{Middle:} Monte Carlo over $2^{22}$ samples, with probabilities compared bin by bin; hatched bins contain fewer than
    $100$ samples. \textbf{Right:} the \emph{difference} between the two ratios, Monte Carlo minus para-Gaussian, which is the difference of the
    two densities measured relative to the Gaussian density. The two agree closely in the bulk of
    the distribution. The difference is close to $0$ inside the inner ellipse and of order $0.1$ in the ring between the ellipses. The
    largest differences are in the tails, where the para-Gaussian ratio keeps growing while the true density does not. The dashed curves are where the para-Gaussian ratio equals $1$, i.e.\ where the para-Gaussian and Gaussian densities agree.}
    \label{fig:linear-one-neuron-two-inputs}
\end{figure}

\subsection{ReLU Activation Function}
\label{sec:relu-application}
We now turn to the normalized ReLU, $\vp(x)=\sqrt2\max\{x,0\}$. Unlike the linear case, the Gaussian expectations appearing in
Theorem \ref{thm:main} cannot be combined into a closed recursion for a finite set of coefficients, and we do not have an exact
formula for the fourth cumulant coefficient $C^\ell$ when there are several inputs. In particular, the covariance itself now has a
non-trivial finite width correction, which we denote as \begin{equation}
\label{eq:sigma-correction}
\sigma^\ell:=\lim_{n\to\infty}n(\bar\Sigma^\ell-K^\ell)
\end{equation} defined so that $\bar\Si^\ell=K^\ell+\frac1n\si^\ell+O(n^{-2})$ by \eqref{eq:mainSigmaExpansion}, with $\si^\ell\neq0$
for distinct inputs. Instead, we evaluate the $1/n$ recursions \eqref{eq:mainSigmaExpansion} and \eqref{eq:mainC} numerically. Since $\vp$ is
positively homogeneous, every Gaussian expectation in these recursions factorizes into a radial moment, which is explicit, and an
integral over angles, which we compute by Gauss--Legendre quadrature on the arcs between the kinks of $\vp$. (Closed forms for these integrals can also be written in terms of the functions $J_{a,b}$ of
\cite{JakubNica2024}.)

The one exception, where there \emph{is} an exact formula is the case of a a single input, where the recursion closes exactly. This result was equivalently obtained as the 4th moment in \cite{HaninNica2020b}. We repeat this part of the argument in the lemma below.

\begin{lem}[ReLU networks with one input]
\label{lem:relu-one-input}
For one unit norm input, $A = \{\alpha\}$, we have $\bar\Si^\ell(\alpha,\alpha)=1$ and $C^\ell(\alpha,\alpha;\alpha,\alpha)=n\left[\left(1+\frac5n\right)^\ell-1\right]$.
\end{lem}
\begin{proof}
For notational convenience, we drop writing the dependence on $\alpha$.  Conditionally on layer $\ell-1$, the neurons of layer $\ell$ are i.i.d.\ $\mathcal{N}(0,\Si^\ell)$, and
since $\vp$ is positively homogeneous $\vp(cx)=|c|\vp(x)$, we observe that we can pull out the norm of the vector out of the ReLU to obtain $$\Si^{\ell+1}=\Si^\ell\cdot\frac1n\sum_{j=1}^n\vp(G_j)^2$$ with $G_j$ i.i.d.\ standard Gaussians independent
of $\Si^\ell$. As $\bE[\vp(G)^2]=1$ and $\var[\vp(G)^2]=\bE[\vp(G)^4]-1=5$, this gives $\bE[\Si^{\ell+1}]=\bE[\Si^\ell]$ and
$\bE\big[(\Si^{\ell+1})^2\big]=\left(1+\frac5n\right)\bE\big[(\Si^\ell)^2\big]$, and hence $n\var[\Si^\ell]=n\left[\left(1+\frac5n\right)^\ell-1\right]$.
\end{proof}

\begin{figure}[H]
    \centering
    \includegraphics[width=\linewidth]{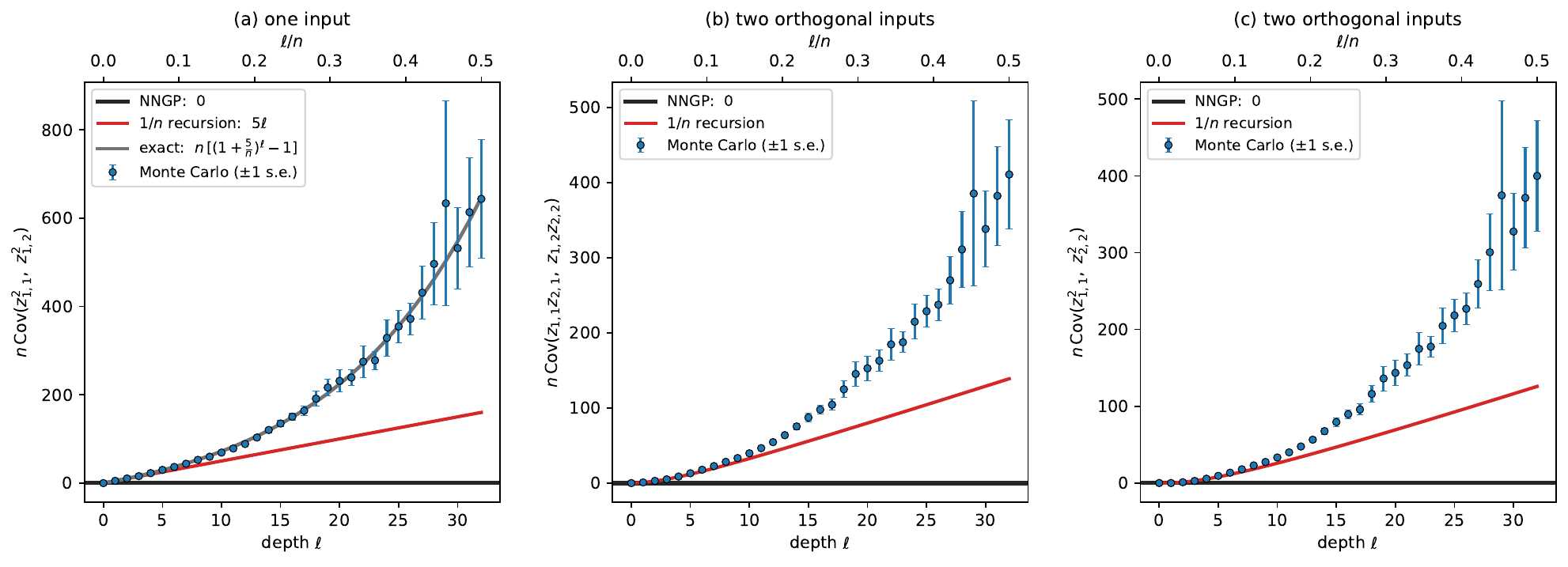}
    \caption{\footnotesize\textbf{ReLU network. Fourth cumulants.} Quantities related to the fourth cumulants of a ReLU network ($\vp(x)=\sqrt2\max\{x,0\}$, width
    $n=64$) for depths up to $\ell=32$, for two different neurons; $z_{\al,k}$ denotes the output of neuron $k$ on input $\al$.
    \textbf{(a)} One unit norm input: $n\cov\big(z_{1,1}^2,z_{1,2}^2\big)=C^\ell(1,1;1,1)$, which equals $n\big[(1+\frac5n)^\ell-1\big]$ exactly
    (Lemma \ref{lem:relu-one-input}) and $5\ell$ according to the $1/n$ recursion.
    \textbf{(b)}, \textbf{(c)} Two unit norm inputs that are orthogonal at layer $0$ ($K=I$):
    \textbf{(b)} $n\cov\big(z_{1,1}z_{2,1},\,z_{1,2}z_{2,2}\big)=C^\ell(1,2;1,2)$ and \textbf{(c)} $n\cov\big(z_{1,1}^2,z_{2,2}^2\big)=C^\ell(1,1;2,2)$. For
    several inputs no exact formula is available, and we show the $1/n$ recursion \eqref{eq:mainC}, evaluated numerically.
    Red: $1/n$ recursion; grey: exact (one input only); blue: Monte Carlo over $2^{15}$ networks, with $\pm1$ standard error
    computed from batch means. The Gaussian prediciton (NNGP) is $0$ in every panel. The $1/n$ recursion is accurate while
    $\ell/n$ is small, and underestimates the cumulants once $\ell/n$ is of order one.}
    \label{fig:relu-cumulants-vs-depth}
\end{figure}
\begin{figure}[H]
    \centering
    \includegraphics[width=\linewidth]{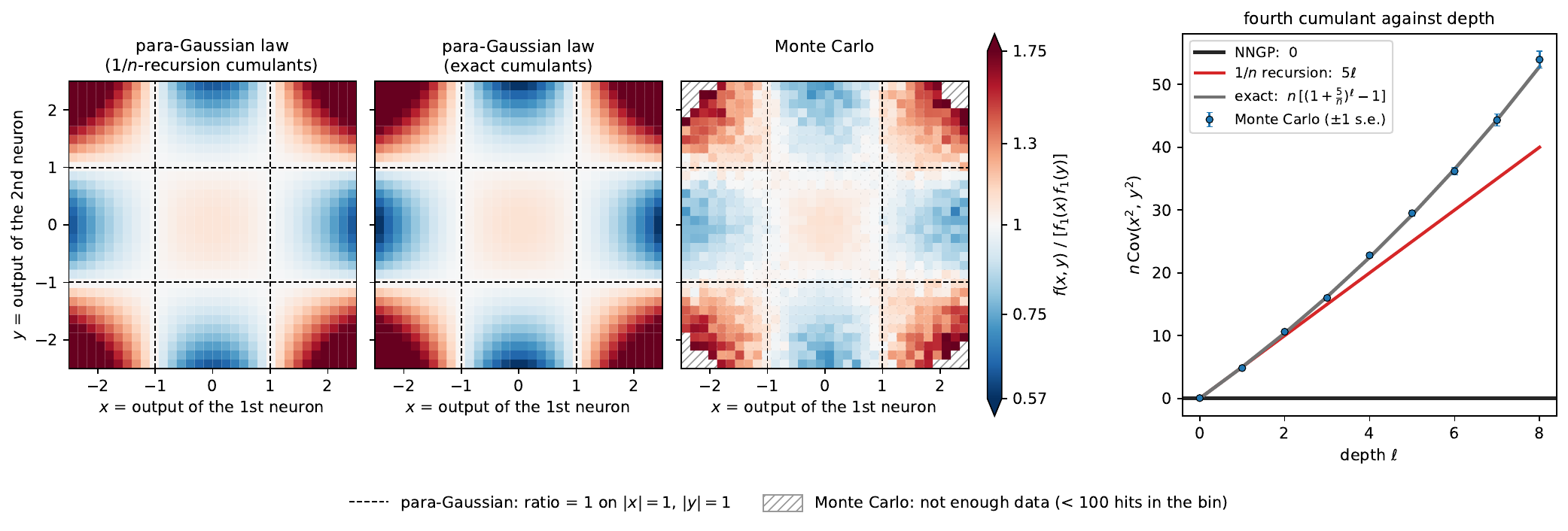}
    \caption{\footnotesize\textbf{ReLU network. Two neuron outputs, one input. Uncorrelated but not independent.} The distribution of two different
    neurons of a ReLU network ($\vp(x)=\sqrt2\max\{x,0\}$, width $n=64$) receiving the same unit norm input, i.e.\ $A=\{\alpha\}$ and we
    look at the joint distribution of the random variables $X=z^\ell_{\alpha,1}$, $Y=z^\ell_{\alpha,2}$. These random variables are
    uncorrelated, but not independent. \textbf{Left} three panels: the ratio $f(x,y)/\big[f_1(x)f_1(y)\big]$ of the joint density of the two
    neurons' outputs to the product of their marginal densities at depth $\ell=4$, which the NNGP predicts to be identically $1$. The
    para-Gaussian law predicts this as $1+\frac{\ka}{4}(x^2-1)(y^2-1)$, where $\ka=\ka(x,x,y,y)$ is the cross-neuron fourth cumulant. We
    show both the result $\ka=\frac{5\ell}{n}$ from the $1/n$-approximated recursion and the exact value
    $\ka=\big(1+\frac5n\big)^\ell-1$ from Lemma \ref{lem:relu-one-input}, next to Monte Carlo over $2^{21}$ neuron pairs. \textbf{Right}: the cumulant $n\cov(x^2,y^2)$ against depth, with the $1/n$ recursion $5\ell$,
    the exact value $n\big[\big(1+\frac5n\big)^\ell-1\big]$ and Monte Carlo ($\pm1$ standard error).}
    \label{fig:relu-two-neurons-one-input}
\end{figure}
\begin{figure}[H]
    \centering
    \includegraphics[width=0.85\linewidth]{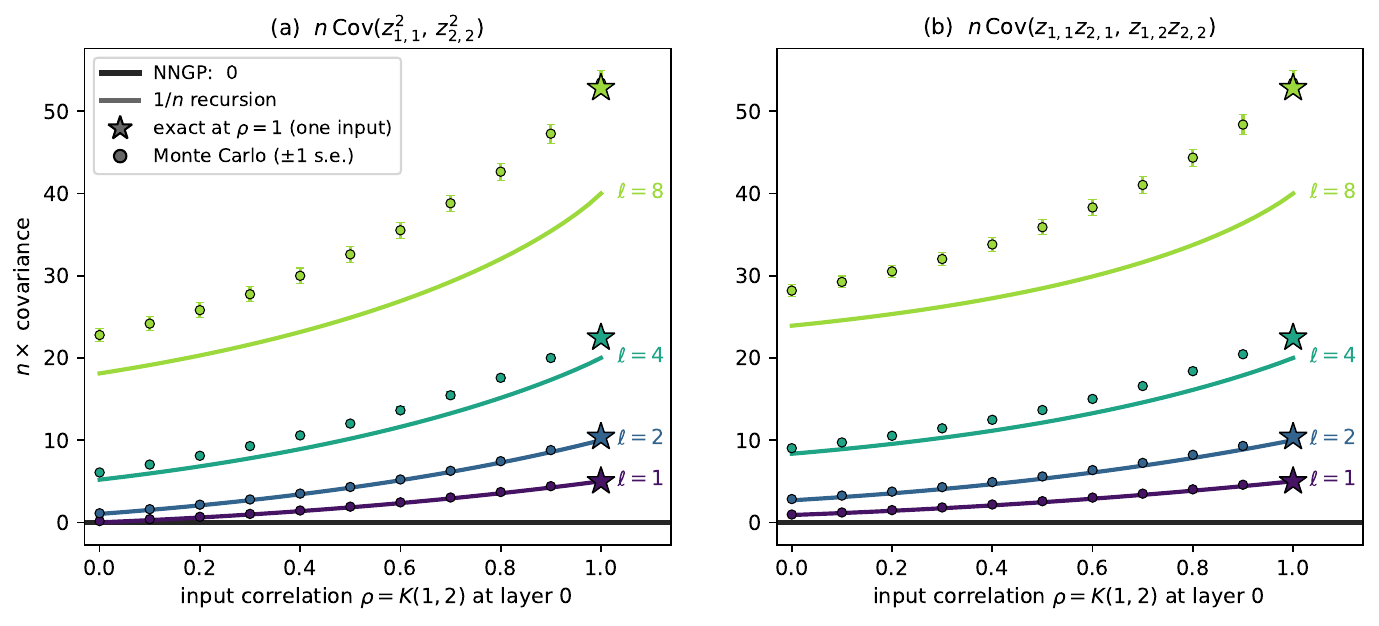}
    \caption{\footnotesize\textbf{ReLU network. Two neurons, two inputs. Dependence on the input correlation.} Cross-neuron fourth cumulants of a ReLU
    network ($\vp(x)=\sqrt2\max\{x,0\}$, width $n=64$) for two unit norm inputs with correlation $\rho=K^0(1,2)$ at layer $0$, against $\rho$
    for depths $\ell=1,2,4,8$. Here $z_{\al,k}$ denotes the output of neuron $k$ on input $\al$ where $\al\in\{1,2\}$ and $k\in\{1,2\}$.
    \textbf{(a)} $n\cov\big(z_{1,1}^2,z_{2,2}^2\big)=C^\ell(1,1;2,2)$, the dependence between the first neuron on input $1$ and the second neuron
    on input $2$.
    \textbf{(b)} $n\cov\big(z_{1,1}z_{2,1},\,z_{1,2}z_{2,2}\big)=C^\ell(1,2;1,2)$, the dependence between the products of each neuron's outputs
    on the two inputs.
    Solid lines: $1/n$ recursion of Theorem \ref{thm:main}, evaluated numerically, since for ReLU there is no exact formula with two
    distinct inputs. Stars represent the exact value $n\big[(1+\frac5n)^\ell-1\big]$ at $\rho=1$, where the two inputs coincide (Lemma
    \ref{lem:relu-one-input}). Monte Carlo over $2^{20}$ neuron pairs per value of $\rho$ ($\pm1$ standard error) is shown as cirlces. The
    Gaussian prediction (NNGP) predicts $0$ in all cases. Unlike the linear case, the cumulant in (a) is positive at $\rho=0$ already
    according to the $1/n$ recursion, since ReLU layers make inputs that are orthogonal at layer $0$ positively correlated. The $1/n$
    recursion underestimates the cumulants for large values of $\ell/n$.}
    \label{fig:relu-two-neurons-two-inputs}
\end{figure}

\begin{figure}[H]
    \centering
    \includegraphics[width=0.85\linewidth]{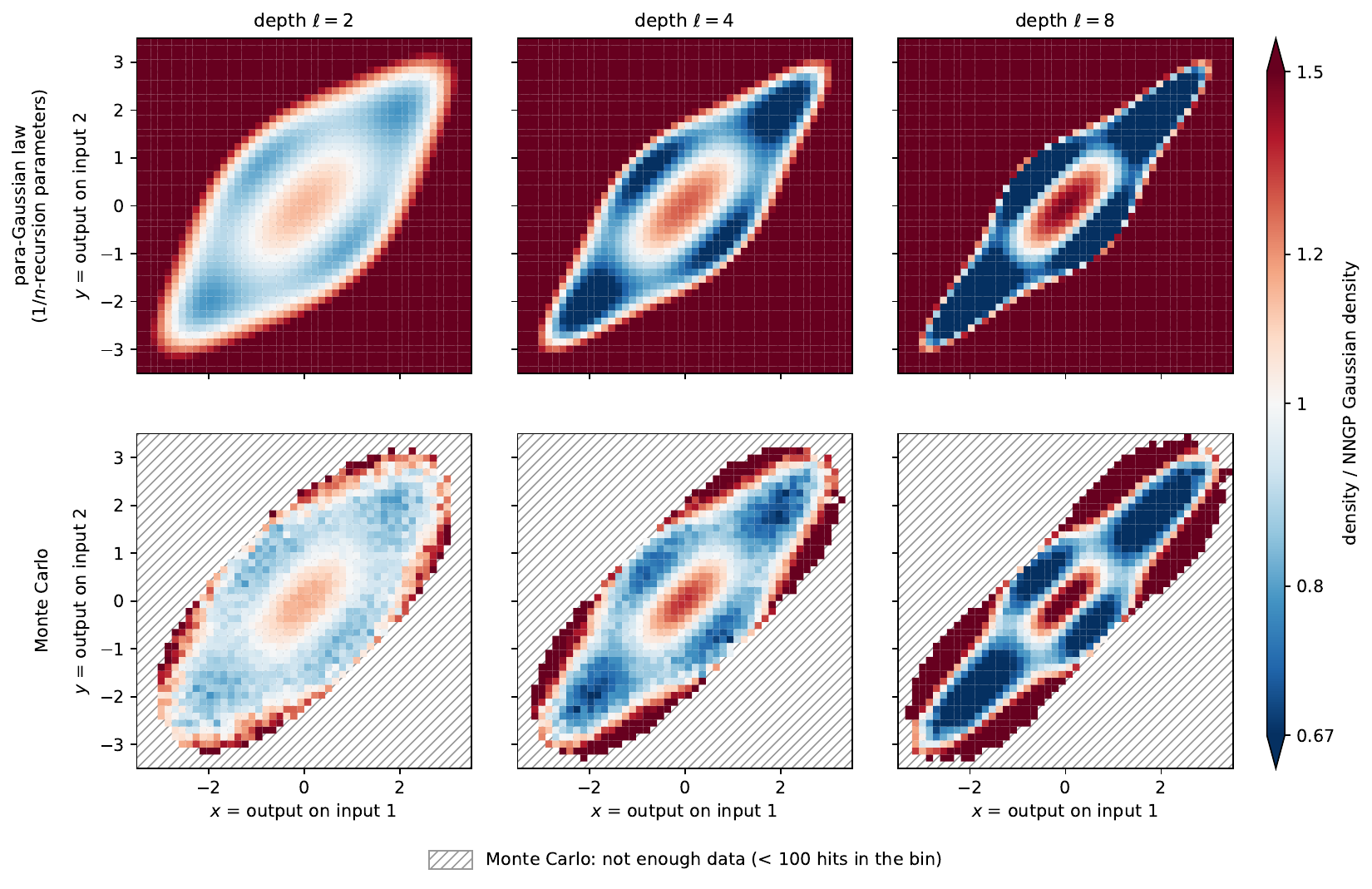}
    \caption{\footnotesize\textbf{ReLU network. One neuron, two inputs. The joint density as depth grows.} A representation of the joint density of the random vector
    $\big(z^\ell_{1,1},z^\ell_{2,1}\big)$, the outputs of the first neuron of layer $\ell$ on two unit norm inputs with correlation
    $K(1,2)=0.5$ at layer $0$, in a ReLU network ($\vp(x)=\sqrt2\max\{x,0\}$, width $n=64$), for depths $\ell=2,4,8$ (columns). Each
    panel shows the \emph{ratio} of a density to the NNGP Gaussian density of $\mathcal{N}(0,K^\ell)$, on a common colour scale. \textbf{Top row}: the
    para-Gaussian density of Definition \ref{def:PG}, with covariance $K^\ell+\frac1n\si^\ell$ from \eqref{eq:mainSigmaExpansion} and fourth
    cumulant tensor computed with $C^\ell$ from the $1/n$ recursion
    \eqref{eq:mainC}; all Gaussian expectations are evaluated numerically. Bottom row: Monte Carlo over $2^{21}$ samples per
    depth, with probabilities compared bin by bin; hatched bins contain fewer than $100$ samples. No exact formula is available for
    comparison in this case.}
    \label{fig:relu-one-neuron-two-inputs}
\end{figure}

\section{Proof of the Asymptotic Convergence, Theorem \ref{thm:TV-gauss-pg}}
\label{sec:para-proof}

Theorem \ref{thm:TV-gauss-pg} is about understanding the distribution of $Z^\ell = \left\{ z_{\alpha,j}^\ell \right\}_{\alpha\in A,\, j\in[m]}$
by studying the expected value $\mathbb{E}[f(Z^\ell)]$, where $f:\mathbb{R}^{A\times[m]}\to\mathbb{R}$ is some test function.
By Remark \ref{rem:cond-gauss}, the distribution of $Z^\ell$ is conditionally Gaussian, $Z^\ell\sim\mathcal{N}\left(0,I_m\otimes\Si^\ell\right)$,
with the random matrix $\Si^\ell$ defined in \eqref{eq:Sigma-def-intro}. The idea of the proof is to decompose $\Si^\ell$ into its mean plus
random fluctuations, and then Taylor series expand around the mean. In order to do this, we introduce the following notations.

\subsection{\texorpdfstring{$g_f$}{The interpolation function} and its derivatives}

\begin{defn}
For a given test function $f:\mathbb{R}^{A\times[m]}\to\mathbb{R}$, we introduce the auxiliary function $G_f$, defined
on positive semi-definite matrices $S$ indexed by $A\times[m]$ by
\begin{equation}
\label{def:G_f}
    G_f(S) := \mathbb{E}_{X\sim\mathcal{N}(0,S)}\left[ f(X)\right].
\end{equation}
We also introduce notation for the fluctuation of $\Si^\ell$ around its mean, namely
\begin{align}
    \bar{\Si}^\ell(\alpha,\beta) &:= \mathbb{E}\left[ \Si^\ell(\alpha,\beta) \right], \\
    \Delta^\ell(\alpha,\beta) &:= \Si^\ell(\alpha,\beta)- \bar{\Si}^\ell(\alpha,\beta),
\end{align}
so that $I_m\otimes\Si^\ell=I_m\otimes\bar{\Si}^\ell+I_m\otimes\Delta^\ell$. Finally, we define the random one-variable
function $g_f:[0,1]\to\mathbb{R}$ by
\begin{equation}
\label{def:g_f}
    g_f(t) := G_f\Big(I_m\otimes\left(\bar{\Si}^\ell + t \Delta^\ell\right) \Big),
\end{equation}
where we suppress the dependence of $g_f$ on $\ell$ in the notation.
\end{defn}

\begin{rem}
By conditional Gaussianity (Remark \ref{rem:cond-gauss}), $g_f(1)=G_f(I_m\otimes\Si^\ell)=\mathbb{E}\left[f(Z^\ell)\,\middle|\,\mathcal{F}^{\ell-1}\right]$, where $\mathcal F^{\ell-1}$ is the $\sigma$-algebra generated by layer $\ell-1$, and hence
we recover the quantity of interest as
\[
\mathbb{E}\left[g_f(1)\right]=\mathbb{E}\left[G_f\left(I_m\otimes\Si^\ell\right)\right]=\mathbb{E}\left[f(Z^\ell)\right].
\]
On the other hand, $g_f(0)=G_f(I_m\otimes\bar{\Si}^\ell)$ uses only the mean $\bar{\Si}^\ell$ (which is a non-random matrix!), and is
the purely Gaussian approximation. The function $g_f$ thus interpolates between these two cases: $t=0$ gives the Gaussian
approximation, and $t=1$ gives $G_f\left(I_m\otimes\Si^\ell\right)$. The proof of Theorem \ref{thm:TV-gauss-pg} compares the two by
Taylor expanding $g_f$ around $t=0$ to approximate its value at $t=1$. To do this, we first prove some lemmas about $g_f$ and its derivatives.
\end{rem}

\begin{lem}
\label{lem:G-derivatives}
Let $S_0$ be a positive definite matrix and $D$ a symmetric matrix, both indexed by $A\times[m]$. Let $r\in\mathbb{N}$, and
suppose $f$ is $2r$ times differentiable with all derivatives of order at most $2r$ polynomially bounded. Then for every $t$
such that $S_0+tD$ is positive definite,
\[
\frac{d^r}{dt^r} G_f(S_0+tD) = \frac{1}{2^r} \sum_{a_1,b_1,\ldots,a_r,b_r \in A\times[m]}
\Big( \prod_{i=1}^r D_{a_i b_i} \Big)\, \mathbb{E}_{X\sim\mathcal{N}(0,S_0+tD)}\Big[ \big(\partial_{a_1}\partial_{b_1} \cdots
\partial_{a_r}\partial_{b_r} f\big)(X)\Big].
\]
\end{lem}

\begin{proof}
Since $G_f$ is a smooth function of $S$ on the positive definite matrices \cite[Lemma 13]{Hanin2024}, its directional
derivative $\frac{d}{dt}\big|_{t=0}G_f(S+tD)$ is linear in $D$. Writing $D=D_+-D_-$ with $D_\pm$ positive semi-definite, it
therefore suffices to prove the case $r=1$ when $D$ is positive semi-definite.

In that case, notice that the distribution of $X\sim\mathcal{N}(0,S_0+tD)$ can be realized as a diffusion process
$X_t$ starting from $X_0 \sim \mathcal{N}(0,S_0)$ with \emph{constant} infinitesimal drift and volatility,
$\d X_t = 0\, \d t + D^{1/2}\, \d W_t$, where $W_t$ is a standard Brownian motion independent of $X_0$. Indeed, $X_t$ is
Gaussian with covariance $S_0+tD$. By It\^o's lemma (or simply the heat equation in this case), the drift of the process
$Y_t = f(X_t)$ is $\frac{1}{2} \sum_{a,b \in A\times[m]} D_{ab} \left(\partial_a \partial_b f \right)(X_t)$, and the
martingale part has mean zero since $f$ and its derivatives are polynomially bounded. Taking expected values and integrating
establishes that
$$G_f(S_0+tD)- G_f(S_0)= \mathbb{E}[Y_t]-\mathbb{E}[Y_0]=\intop_0^t \frac{1}{2} \sum_{a,b \in A\times[m]} D_{ab}\,
\mathbb{E}\Big[\left(\partial_a \partial_b f\right)(X_s)\Big] \d s,$$
whence it follows by the fundamental theorem of calculus that
$\frac{d}{dt}G_f(S_0+tD) = \frac12\sum_{a,b \in A\times[m]} D_{ab}\, \mathbb{E}\Big[\left(\partial_a \partial_b f\right)(X_t)\Big]$.

This establishes the result for the base case $r=1$, and shows the effect of a single derivative in $t$. The formula for
general $r$ follows by induction, applying the same argument to compute the derivative
$\frac{d}{dt}\mathbb{E}\Big[\left(\partial_{a_1} \partial_{b_1} \cdots \partial_{a_{r-1}} \partial_{b_{r-1}} f\right)(X_t)\Big]$.
\end{proof}

The following formula can be found as Theorem 7.6 in \cite{NicaOrtmann2026}:

\begin{lem}[Gaussian integration by parts with Hermites]
\label{lem:gaussian-ibp}
Let $V$ be a positive definite matrix indexed by $A\times[m]$, let $X\sim\mathcal{N}(0,V)$, and suppose $f$ is $k$ times
differentiable with polynomially bounded derivatives. Then for all $a_1,\ldots,a_k\in A\times[m]$,
\[
\mathbb{E}\Big[\big(\partial_{a_1}\cdots\partial_{a_k}f\big)(X)\Big]=\mathbb{E}\Big[f(X)\,H_{a_1\cdots a_k}(X;V)\Big],
\]
where we use the index-list notation for the Hermite polynomials $H_{a_1\cdots a_k}(\vec x;\Si):=H_{\vec e_{a_1}+\cdots+\vec e_{a_k}}(\vec x;\Si)$
\end{lem}

% \begin{proof}
% Write $\vp_V$ for the density of $\mathcal{N}(0,V)$, so that the left-hand side is
% $\int(\partial_{a_1}\cdots\partial_{a_k}f)(x)\,\vp_V(x)\,\d x$. Integrating by parts $k$ times moves each derivative from $f$ onto
% $\vp_V$, at the cost of a factor of $-1$ each time; the boundary terms vanish because $\vp_V$ and all its derivatives decay faster
% than any polynomial, while $f$ and its derivatives are polynomially bounded. This gives
% \[
% \mathbb{E}\Big[\big(\partial_{a_1}\cdots\partial_{a_k}f\big)(X)\Big]=(-1)^k\int f(x)\,\partial_{a_1}\cdots\partial_{a_k}\vp_V(x)\,\d x.
% \]
% Finally, by the derivative formula for the Hermite polynomials \cite[Corollary 7.5]{NicaOrtmann2026},
% $\partial_{a_1}\cdots\partial_{a_k}\vp_V=(-1)^kH_{a_1\cdots a_k}(\cdot\,;V)\,\vp_V$, and the result follows.
% \end{proof}

\begin{rem}
\label{rem:G-derivatives-hermite}
Combining Lemma \ref{lem:gaussian-ibp} with Lemma \ref{lem:G-derivatives}, the expectations of derivatives of $f$ in the formula for
$\frac{d^r}{dt^r}G_f(S_0+tD)$ can be replaced by expectations of $f$ against Hermite polynomials:
\[
\frac{d^r}{dt^r}G_f(S_0+tD)=\frac1{2^r}\sum_{a_1,b_1,\ldots,a_r,b_r\in A\times[m]}\Big(\prod_{i=1}^rD_{a_ib_i}\Big)\,
\mathbb{E}_{X\sim\mathcal{N}(0,S_0+tD)}\Big[f(X)\,H_{a_1b_1\cdots a_rb_r}\big(X;S_0+tD\big)\Big].
\]
In this form, no derivatives of $f$ appear: all the derivatives act on the Gaussian density, where they become Hermite
polynomials. This is what allows us to identify the appearace of the para-Gaussian law, and bound the error estimates below to depend on $f$ only through $\norm{f}_\infty$.
\end{rem}

\subsubsection{ \texorpdfstring{Evaluating $\mathbb{E}[g_f'(0)]$ and $\mathbb{E}[g_f''(0)]$}{Expectation of the derivatives}}

\begin{lem}
\label{lem:g-prime-mean-zero}
Suppose $\bar{\Si}^\ell$ is positive definite, $\mathbb{E}\norm{\Si^\ell}<\infty$, and $f$ is twice differentiable with
polynomially bounded derivatives. Then $$\mathbb{E}\left[g_f^\prime(0)\right]=0.$$
\end{lem}

\begin{proof}
Since $g_f$ is defined from the matrix $S_0 + tD$ with $S_0=I_m\otimes\bar{\Si}^\ell$ and $D=I_m\otimes\Delta^\ell$, we have by Lemma \ref{lem:G-derivatives} with $r=1$ that
\[
g_f^\prime(0)=\frac12\sum_{a,b\in A\times[m]}\left(I_m\otimes\Delta^\ell\right)_{ab}\;c_{ab},\qquad
c_{ab}:=\mathbb{E}_{X\sim\mathcal{N}\left(0,I_m\otimes\bar{\Si}^\ell\right)}\big[(\partial_a\partial_bf)(X)\big].
\]
The coefficients $c_{ab}$ are finite and deterministic, since they depend only on the mean $\bar{\Si}^\ell$ and not on the
fluctuation $\Delta^\ell$. Hence $g_f^\prime(0)$ is a linear function of the entries of $\Delta^\ell$, and taking expectations
gives $\mathbb{E}\left[g_f^\prime(0)\right]=\frac12\sum_{a,b}\mathbb{E}\left[\left(I_m\otimes\Delta^\ell\right)_{ab}\right]\,c_{ab}=0$,
since $\mathbb{E}[\Delta^\ell]=0$ by its definition.
\end{proof}
\begin{lem}
\label{lem:g-second-deriv}
Suppose $\bar{\Si}^\ell$ is positive definite, $\mathbb{E}\norm{\Si^\ell}^2<\infty$, and $f$ is four times differentiable with
polynomially bounded derivatives. Write $X\sim\mathcal{N}\left(0,I_m\otimes\bar{\Si}^\ell\right)$ for the Gaussian approximation. Then
\[
\frac12\,\mathbb{E}\left[g_f^{\prime\prime}(0)\right]
=\frac18\sum_{i,k\in[m]}\ \sum_{\alpha,\beta,\gamma,\delta\in A}
\cov\left(\Si^\ell(\alpha,\beta),\Si^\ell(\gamma,\delta)\right)\,
\mathbb{E}\Big[\big(\partial_{(\alpha,i)}\partial_{(\beta,i)}\partial_{(\gamma,k)}\partial_{(\delta,k)}f\big)(X)\Big].
\]
Equivalently, $\frac12\,\mathbb{E}\left[g_f^{\prime\prime}(0)\right]=\frac18\sum_{a,b,c,e\in A\times[m]}T^\ell_{abce}\,
\mathbb{E}\big[(\partial_a\partial_b\partial_c\partial_ef)(X)\big]$, where
\begin{equation}
\label{eq:T-def}
T^\ell_{abce}:=\mathbb{E}\left[\left(I_m\otimes\Delta^\ell\right)_{ab}\left(I_m\otimes\Delta^\ell\right)_{ce}\right]
=\de_{ij}\,\de_{kl}\,\cov\left(\Si^\ell(\alpha,\beta),\Si^\ell(\gamma,\delta)\right)
\end{equation}
for $a=(\alpha,i)$, $b=(\beta,j)$, $c=(\gamma,k)$, $e=(\delta,l)$.
\end{lem}

\begin{proof}
Since $g_f$ is defined from the matrix $S_0+tD$ with $S_0=I_m\otimes\bar{\Si}^\ell$ and $D=I_m\otimes\Delta^\ell$, we have by
Lemma \ref{lem:G-derivatives} with $r=2$ that
\[
g_f^{\prime\prime}(0)=\frac14\sum_{a,b,c,e\in A\times[m]}\left(I_m\otimes\Delta^\ell\right)_{ab}\left(I_m\otimes\Delta^\ell\right)_{ce}\;c_{abce},
\qquad c_{abce}:=\mathbb{E}\big[(\partial_a\partial_b\partial_c\partial_ef)(X)\big].
\]
By the assumptions on $\Sigma$, the coefficients $c_{abce}$ are finite and deterministic, so
$g_f^{\prime\prime}(0)$ is a quadratic function of the entries of $\Delta^\ell$, and its expectation depends only on the second
moments of $\Delta^\ell$. Taking expectations gives the second formula, with $T^\ell$ as in \eqref{eq:T-def}. To evaluate
$T^\ell$, note that $\left(I_m\otimes\Delta^\ell\right)_{(\alpha,i),(\beta,j)}=\de_{ij}\Delta^\ell(\alpha,\beta)$, and that
$\mathbb{E}\left[\Delta^\ell(\alpha,\beta)\Delta^\ell(\gamma,\delta)\right]=\cov\left(\Si^\ell(\alpha,\beta),\Si^\ell(\gamma,\delta)\right)$
since $\mathbb{E}[\Delta^\ell]=0$. Using the Kronecker deltas to collapse the sums over $j$ and $l$ gives the desired result.
\end{proof}

\subsection{Pair symmetries and alternate form of the para-Gaussian law}
\begin{defn}
\label{def:pair-symmetric}
Let $I$ be a finite index set and let $T=(T_{abcd})_{a,b,c,d\in I}$ be a tensor with four indices. For a permutation
$\pi\in S_4$ of the four index positions, we write $T_{\pi(abcd)}$ for the entry of $T$ with its indices rearranged by
$\pi$; for example, if $\pi$ swaps the second and third positions, then $T_{\pi(abcd)}=T_{acbd}$.
\begin{itemize}
\item We say $T$ is {\bf \emph{fully symmetric}} if $T_{\pi(abcd)}=T_{abcd}$ for every $\pi\in S_4$ and all $a,b,c,d\in I$,
that is, if $T$ is unchanged by any rearrangement of its indices.
\item We say $T$ is {\bf \emph{pair-symmetric}} if
\[
T_{abcd}=T_{bacd}=T_{abdc}=T_{cdab}\qquad\text{for all }a,b,c,d\in I,
\]
that is, $T$ is symmetric within its first pair of indices, within its second pair of indices, and under exchanging the two
pairs.
\end{itemize}

\end{defn}

\begin{rem}
Examples of fully symmetric tensors are the
    fourth cumulants $\ka(Y_a,Y_b,Y_c,Y_d)$ of a random vector and the 4th order Hermite polynomials where the indices tell us which variables we are using\footnote{Do not confuse this notation with the other convention for Hermite polynomials of \emph{how many times} each variable appears; that convention is not symmetric!} $H_{abcd}(x;\Si):=H_{\vec e_a+\vec e_b+\vec e_c+\vec e_d}(x;\Si).$
    
    The prototypical example of a pair-symmetric tensor in this work is $T_{abcd}=\cov\left(Y_aY_b,\,Y_cY_d\right)$ for a random vector
$(Y_a)_{a\in I}$, such as the quantity $C(\alpha,\beta;\gamma,\delta)$ from \eqref{eq:defCell}. The tensor $T^\ell$ in \eqref{eq:T-def} is pair-symmetric, since $\Delta^\ell$ is a symmetric matrix and
covariance is symmetric in its two arguments. The three terms of the pairing sum correspond to the three ways of grouping
four indices into two pairs: $\{ab\,|\,cd\}$, $\{ac\,|\,bd\}$ and $\{ad\,|\,bc\}$. 
\end{rem}

\begin{defn}
Given a \emph{pair-symmetric tensor} $T$, we define the {\bf \emph{pairing sum}} of $T$ to be the tensor $\kappa$ given by
\[
\ka_{abcd}:=T_{abcd}+T_{acbd}+T_{adbc}.
\]
\end{defn}

\begin{lem}
\label{lem:pairing-sum}
Let $T$ be pair-symmetric with pairing sum $\ka$. Then:
\begin{enumerate}
\item[(a)] $\ka$ is fully symmetric, and $\frac13\ka$ is the full symmetrization of $T$, namely:
$$\displaystyle\frac1{4!}\sum_{\pi\in S_4}T_{\pi(abcd)}=\frac13\,\ka_{abcd}.$$
\item[(b)] For every fully symmetric tensor $H$ on $I$,
$$\displaystyle\frac1{4!}\sum_{a,b,c,d\in I}\ka_{abcd}H_{abcd}=\frac18\sum_{a,b,c,d\in I}T_{abcd}H_{abcd}.$$
\end{enumerate}
\end{lem}

\begin{proof}
The three pair symmetries of $T$ generate a subgroup of $8$ permutations on four indices. This divides $S_4$ into three cosets (each corresponding to the three pairings $\{ab\,|\,cd\}$, $\{ac\,|\,bd\}$, $\{ad\,|\,bc\}$), and each coset containing $8$ of the $4!=24$ permutations.
By the pair-symmetry, $T$ takes the same value on all permutations in a coset. Hence
$$\frac1{24}\sum_{\pi\in S_4}T_{\pi(abcd)}=\frac8{24}\left(T_{abcd}+T_{acbd}+T_{adbc}\right)=\frac13\ka_{abcd},$$ which is fully
symmetric, proving (a). For (b), since $H$ is fully symmetric, relabelling the summation indices gives
$\sum T_{\pi(abcd)}H_{abcd}=\sum T_{abcd}H_{abcd}$ for every $\pi\in S_4$. Averaging over $\pi$ and using (a) gives
$\sum T_{abcd}H_{abcd}=\frac13\sum\ka_{abcd}H_{abcd}$, and dividing by $8$ gives (b).
\end{proof}

\begin{cor}[Two forms of the para-Gaussian measure]
\label{cor:PG-two-forms}
Recall from Definition \ref{def:PG} that the para-Gaussian measure $\pgdk{\Si}{\ka}{n}{4}$ on $\bR^I$ is determined by a positive definite
covariance matrix $\Si$ and a fully symmetric four tensor $\ka$, which plays the role of the fourth cumulants, through
\[
\pgd{\Si}{\ka}{n}(x)=\vp_\Si(x)\Big[1+\frac1{4!n}\sum_{a,b,c,d\in I}\ka_{abcd}H_{abcd}(x;\Si)\Big],
\quad
\int e^{\imath\la\cdot x}\pgd{\Si}{\ka}{n}(x)\,\d x=e^{-\frac12\la^T\Si\la}\Big(1+\frac1{4!n}\sum_{a,b,c,d\in I}\ka_{abcd}\la_a\la_b\la_c\la_d\Big).
\]
Suppose that $\ka$ is the pairing sum of some pair-symmetric tensor $T$. Then $\pgd{\Si}{\ka}{n}$ can be written equivalently in
terms of $T$, with the factor $\frac1{4!}$ replaced by $\frac18$:
\[
\pgd{\Si}{\ka}{n}(x)=\vp_\Si(x)\Big[1+\frac{1}{8n}\sum_{a,b,c,d\in I}T_{abcd}H_{abcd}(x;\Si)\Big],
\quad
\int e^{\imath\la\cdot x} \pgd{\Si}{\ka}{n}(x) \d x =e^{-\frac12\la^T\Si\la}\Big(1+\frac{1}{8n}\sum_{a,b,c,d\in I}T_{abcd}\la_a\la_b\la_c\la_d\Big).
\]
\end{cor}

\begin{proof}
Apply Lemma \ref{lem:pairing-sum}(b) with the fully symmetric tensors $H_{abcd}=H_{abcd}(x;\Si)$ and
$H_{abcd}=\la_a\la_b\la_c\la_d$. (The characteristic function in the first form follows from the Fourier formula
$\int e^{\imath\la\cdot x}H_{\vec n}(x;\Si)\vp_\Si(x)\,\d x=(\imath\la)^{\vec n}e^{-\frac12\la^T\Si\la}$, which is equivalent by Fourier
inversion to \cite[Proposition 7.4, eq. (107)]{NicaOrtmann2026}, together with $(\imath\la)^{4}=\la^4$ for $\abs{\vec n}=4$.)
\end{proof}

\begin{rem}[Connection to \cite{NicaOrtmann2024M3L}]
For a single neuron ($m=1$, $I=A$) and a single layer ($\ell=1$) with $T_{\alpha\beta\gamma\delta}= C(\alpha,\beta;\gamma,\delta)$, the second form of
the characteristic function in Corollary \ref{cor:PG-two-forms} is exactly the definition of a para-Gaussian distribution in
\cite[Definition 1]{NicaOrtmann2024M3L}, with its normalization $\frac1{8n}$.
In this case the tensor $C$ is the one computed in \cite[Theorem 2]{NicaOrtmann2024M3L},
$C(\alpha,\beta;\gamma,\delta)=\cov\left(\vp(z_\alpha)\vp(z_\beta),\vp(z_\gamma)\vp(z_\delta)\right)$, which is pair-symmetric.
Corollary \ref{cor:PG-two-forms} thus shows that the $\frac1{8n}$ convention of \cite{NicaOrtmann2024M3L} and the $\frac1{4!}$ Edgeworth
convention with the cumulant describe the same measure, with fourth cumulant given by
$\ka(\alpha,\beta,\gamma,\delta)=\frac1n\left(C(\alpha,\beta;\gamma,\delta)+C(\alpha,\gamma;\beta,\delta)+C(\alpha,\delta;\beta,\gamma)\right)$. 

\end{rem}

\subsection{Taylor series expansion of \texorpdfstring{$g_f$}{the interpolation}}

\begin{prop}
\label{prop:second-order-is-PG}
Suppose $\bar{\Si}^\ell$ is positive definite, $\mathbb{E}\norm{\Si^\ell}^2<\infty$, and $f$ is four times differentiable with
polynomially bounded derivatives. Then the expected Taylor polynomial of degree $2$ of $g_f$ at $t=0$ is the
para-Gaussian law:
\[
\mathbb{E}\left[g_f(0)+g_f^\prime(0)+\frac12 g_f^{\prime\prime}(0)\right]=\int f\,\d \pgl{I_m\otimes\bar\Si^\ell}{n\ka^\ell}{n},
\]
where $\ka^\ell$ is the pairing sum of the tensor $T^\ell$ of \eqref{eq:T-def}, namely
\begin{align}
\ka^\ell\big((\alpha,i),(\beta,j),(\gamma,k),(\delta,l)\big)
& =\de_{ij}\de_{kl}\cov\big(\Si^\ell(\alpha,\beta),\Si^\ell(\gamma,\delta)\big)
+\de_{ik}\de_{jl}\cov\big(\Si^\ell(\alpha,\gamma),\Si^\ell(\beta,\delta)\big)\\
\notag&\quad
+\de_{il}\de_{jk}\cov\big(\Si^\ell(\alpha,\delta),\Si^\ell(\beta,\gamma)\big).
\end{align}
\end{prop}

\begin{proof}
Write $X\sim\mathcal{N}\left(0,I_m\otimes\bar{\Si}^\ell\right)$. We evaluate the three terms in turn.
\begin{itemize}
\item The zeroth order term is deterministic: $g_f(0)=G_f\left(I_m\otimes\bar{\Si}^\ell\right)=\mathbb{E}\left[f(X)\right]$.
\item The first order term has mean zero: $\mathbb{E}\left[g_f^\prime(0)\right]=0$ by Lemma \ref{lem:g-prime-mean-zero}.
\item The second order term is given by Lemma \ref{lem:g-second-deriv}:
\[
\frac12\,\mathbb{E}\left[g_f^{\prime\prime}(0)\right]=\frac18\sum_{a,b,c,d\in A\times[m]}T^\ell_{abcd}\,
\mathbb{E}\big[(\partial_a\partial_b\partial_c\partial_df)(X)\big].
\]
\end{itemize}
By Gaussian integration by parts, Lemma \ref{lem:gaussian-ibp}, the derivatives can be moved into a Hermite factor
\[
\mathbb{E}\big[(\partial_a\partial_b\partial_c\partial_df)(X)\big]=\mathbb{E}\Big[f(X)\,H_{abcd}\big(X;I_m\otimes\bar{\Si}^\ell\big)\Big].
\]
Adding the three terms therefore gives
\[
\mathbb{E}\left[g_f(0)+g_f^\prime(0)+\frac12 g_f^{\prime\prime}(0)\right]
=\int f(x)\,\vp_{I_m\otimes\bar{\Si}^\ell}(x)\Big[1+\frac18\sum_{a,b,c,d\in A\times[m]}T^\ell_{abcd}\,
H_{abcd}\big(x;I_m\otimes\bar{\Si}^\ell\big)\Big]\d x.
\]
Finally, $T^\ell$ is pair-symmetric (see the remark following Definition \ref{def:pair-symmetric}), and substituting
\eqref{eq:T-def} shows that its pairing sum is $\ka^\ell$. By Corollary \ref{cor:PG-two-forms}, the right-hand side is
$\int f\,\pgl{I_m\otimes\bar\Si^\ell}{n\ka^\ell}{n}(x) \d x$.
\end{proof}

\begin{rem}
The parameters of the para-Gaussian law in Proposition \ref{prop:second-order-is-PG} are exactly the true second and fourth
cumulants of $Z^\ell$: the covariance of $Z^\ell$ is $I_m\otimes\bar{\Si}^\ell$, and by Theorem \ref{thm:exact-general-cumulants} its fourth
cumulant tensor is $\ka^\ell$. In other words, Taylor expanding the random covariance to second order produces exactly the
para-Gaussian law whose covariance and fourth cumulants match those of $Z^\ell$.
\end{rem}

\subsection{Error analysis}
\label{sec:error-analysis}

In this section we prove the estimates needed to control the error terms in the Taylor expansion of $g_f$. There are two
kinds of error terms: the Taylor remainder, and the expected third order term $\frac16\mathbb{E}[g_f^{\prime\prime\prime}(0)]$.
Since we aim for bounds in total variation distance, it is important that all estimates depend on the test function $f$
only through $\norm{f}_\infty:=\sup_x\abs{f(x)}$, and not on its derivatives. The estimates combine two ingredients: a
deterministic bound on the derivatives of $G_f$ along a line (Corollary \ref{cor:G-deriv-bound}), and probabilistic estimates on
the size of the fluctuation $\Delta^\ell$, which we import from \cite{Hanin2024} (Lemma \ref{lem:hanin-estimates}).

Throughout, $\norm{M}$ denotes the operator norm of a matrix $M$, and for a symmetric matrix $M$ we write $\lambda_{\min}(M)$ for
its minimum eigenvalue. Note that $\norm{I_m\otimes M}=\norm{M}$ and $\lambda_{\min}(I_m\otimes M)=\lambda_{\min}(M)$.

\subsubsection{Estimates imported from \texorpdfstring{\cite{Hanin2024}}{Hanin}}

\begin{lem}
\label{lem:hanin-estimates}
Suppose Assumptions \ref{ass:nonlinearity} and \ref{ass:non-degen} hold. Let $K^\ell$ denote the infinite width covariance kernel restricted to the inputs $A$, defined by the
recursion $K^\ell(\alpha,\beta)=\mathbb{E}_{z\sim\mathcal{N}(0,K^{\ell-1})}\left[\vp(z_\alpha)\vp(z_\beta)\right]$, and
suppose $K^\ell$ is positive definite. Then:
\begin{enumerate}
\item[(a)] \emph{(Uniform positivity of the mean covariance $\bar\Si^\ell$.)} There exist $\lambda_0>0$ and $n_0\in\bN$ such that the
minimum eigenvalue of the mean covariance $\bar{\Si}^\ell$ is bounded below uniformly in the width:
\[
\inf_{n\geq n_0}\lambda_{\min}\left(\bar{\Si}^\ell\right)\geq\lambda_0.
\]
\item[(b)] \emph{(Moments of the fluctuation.)} For every even integer $q\geq2$ there exists $C_q$ such that
\[
\mathbb{E}\norm{\Delta^\ell}^q\leq C_q\, n^{-q/2}\qquad\text{for all }n.
\]
In particular $\mathbb{E}\norm{\Delta^\ell}^2=O(n^{-1})$ and $\mathbb{E}\norm{\Delta^\ell}^4=O(n^{-2})$.
\item[(c)] \emph{(Third moments of the fluctuation.)} There exists $C$ such that for all
$\alpha_1,\beta_1,\alpha_2,\beta_2,\alpha_3,\beta_3\in A$,
\[
\abs{\mathbb{E}\left[\Delta^\ell(\alpha_1,\beta_1)\,\Delta^\ell(\alpha_2,\beta_2)\,\Delta^\ell(\alpha_3,\beta_3)\right]}\leq C\,n^{-2}\qquad\text{for all }n.
\]
\end{enumerate}
\end{lem}

\begin{proof}
We first translate notation. The layer $z^\ell$ of this paper is the layer $z^{(\ell+1)}$ of \cite{Hanin2024}, and our random
matrix $\Si^\ell$ is the conditional covariance $\Si^{(\ell)}$ of \cite{Hanin2024} with $C_b=0$, $C_W=1$ and $\sigma=\vp$. Each entry
$\Si^\ell(\alpha,\beta)=\frac1n\sum_{j=1}^n\vp(z^{\ell-1}_{\alpha,j})\vp(z^{\ell-1}_{\beta,j})$ is a collective observable in the sense of
\cite{Hanin2024}, built from the polynomially bounded function $\vp(z_\alpha)\vp(z_\beta)$, and each entry of $\Delta^\ell$ is the
corresponding centred collective observable.

For (a): $\bar{\Si}^\ell(\alpha,\beta)=\mathbb{E}\left[\vp(z^{\ell-1}_{\alpha,1})\vp(z^{\ell-1}_{\beta,1})\right]$. By \cite[Theorem 2]{Hanin2024},
$(z^{\ell-1}_{\alpha,1})_{\alpha\in A}$ converges in distribution to $\mathcal{N}(0,K^{\ell-1})$. Moreover, applying \cite[eq. (35)]{Hanin2024}
to the polynomially bounded function $\abs{\vp(z_\alpha)\vp(z_\beta)}^2$ shows that
$\sup_n\mathbb{E}\big[\abs{\vp(z^{\ell-1}_{\alpha,1})\vp(z^{\ell-1}_{\beta,1})}^2\big]<\infty$, so the random variables
$\vp(z^{\ell-1}_{\alpha,1})\vp(z^{\ell-1}_{\beta,1})$ are uniformly integrable. Since $\vp(z_\alpha)\vp(z_\beta)$ is continuous, it follows that
$\bar{\Si}^\ell\to\mathbb{E}_{\mathcal{N}(0,K^{\ell-1})}[\vp\vp]=K^\ell$ entrywise. The minimum eigenvalue is a continuous function of the
matrix entries, so $\lambda_{\min}(\bar{\Si}^\ell)\to\lambda_{\min}(K^\ell)>0$, and hence $\lambda_{\min}(\bar{\Si}^\ell)\geq\lambda_0:=\frac12\lambda_{\min}(K^\ell)$
for all $n$ larger than some $n_0$.

For (b): by \cite[Theorem 16]{Hanin2024}, the centred collective observables $\Delta^\ell(\alpha,\beta)$ satisfy the hypotheses of
\cite[Proposition 17]{Hanin2024}, and so \cite[Lemma 18]{Hanin2024} gives $\mathbb{E}[\Delta^\ell(\alpha,\beta)^q]=O(n^{-q/2})$ for each entry
and each even $q$. Since $\norm{\Delta^\ell}\leq\big(\sum_{\alpha,\beta}\Delta^\ell(\alpha,\beta)^2\big)^{1/2}$ and all norms on the finite
dimensional space of $A\times A$ matrices are equivalent, the bound for $\norm{\Delta^\ell}$ follows.

For (c): since each entry of $\Delta^\ell$ has mean zero, the third moment equals the third joint cumulant,
$\mathbb{E}[\Delta^\ell_1\Delta^\ell_2\Delta^\ell_3]=\ka(\Delta^\ell_1,\Delta^\ell_2,\Delta^\ell_3)$, which is $O(n^{-2})$ by
\cite[Theorem 16, eq. (39)]{Hanin2024} with $k=3$, taking the functions $g_j$ there to be coordinate projections.
\end{proof}

\subsubsection{Deterministic estimates on derivatives of \texorpdfstring{$G_f$}{the interpolation}}

\begin{lem}
\label{lem:hermite-moment-bound}
Let $d,k\geq1$, let $V$ be a positive definite $d\times d$ matrix, let $X\sim\mathcal{N}(0,V)$, and let $a_1,\ldots,a_k\in\{1,\ldots,d\}$.
Consider the multivariable Hermite polynomial $H_{a_1\cdots a_k}(x;V)$ with this covariance structure. We have the bound in
expectation:
\[
\mathbb{E}\Big[\abs{H_{a_1\cdots a_k}(X;V)}\Big]\leq c_{k,d}\,\lambda_{\min}(V)^{-k/2},\qquad c_{k,d}:=d^k\sqrt{k!}
\]
\end{lem}

\begin{proof}
We reduce to the standard Gaussian using the Gaussian integral formula for the Hermite polynomials,
\cite[Definition 1.4, eq. (7)]{NicaOrtmann2026}, which in our coordinate-list notation reads
\[
H_{a_1\cdots a_k}(x;V)=\mathbb{E}_{Z\sim\mathcal{N}(0,V^{-1})}\Big[\prod_{i=1}^k\Big(\big(V^{-1}x\big)_{a_i}+\imath Z_{a_i}\Big)\Big].
\]
Set $x=V^{1/2}y$, so that $V^{-1}x=V^{-1/2}y$, and realize $Z$ as $Z=V^{-1/2}W$ with $W\sim\mathcal{N}(0,I_d)$; this has the correct
covariance $V^{-1/2}V^{-1/2}=V^{-1}$. Each factor in the product is then a linear combination of the corresponding factors for the
identity covariance:
\[
\big(V^{-1/2}y\big)_{a_i}+\imath\big(V^{-1/2}W\big)_{a_i}=\sum_{c=1}^d\big(V^{-1/2}\big)_{a_ic}\,\big(y_c+\imath W_c\big).
\]
Expanding the product and taking the expectation inside the finite sum gives
\[
H_{a_1\cdots a_k}\big(V^{1/2}y;V\big)=\sum_{c_1,\ldots,c_k}\Big(\prod_{i=1}^k(V^{-1/2})_{a_ic_i}\Big)\,
\mathbb{E}_{W\sim\mathcal{N}(0,I_d)}\Big[\prod_{i=1}^k\big(y_{c_i}+\imath W_{c_i}\big)\Big]
=\sum_{c_1,\ldots,c_k}\Big(\prod_{i=1}^k(V^{-1/2})_{a_ic_i}\Big)\,H_{c_1\cdots c_k}(y;I_d),
\]
where in the last step we used the same formula with $V=I_d$. Since $X=V^{1/2}Y$ with $Y\sim\mathcal{N}(0,I_d)$, this gives
\[
\mathbb{E}\abs{H_{a_1\cdots a_k}(X;V)}\leq\sum_{c_1,\ldots,c_k}\Big(\prod_{i=1}^k\abs{(V^{-1/2})_{a_ic_i}}\Big)\,\mathbb{E}\abs{H_{c_1\cdots c_k}(Y;I_d)}.
\]
Each entry of $V^{-1/2}$ is at most $\norm{V^{-1/2}}=\lambda_{\min}(V)^{-1/2}$ in absolute value, and there are $d^k$ terms in the
sum. It remains to bound $\mathbb{E}\abs{H_{c_1\cdots c_k}(Y;I_d)}$. Write $\vec n:=\vec e_{c_1}+\cdots+\vec e_{c_k}$, so that
$H_{c_1\cdots c_k}=H_{\vec n}$ in multi-index notation. For the identity covariance the Hermite polynomial and its dual coincide, so the
orthogonality relation \cite[Theorem 2.4]{NicaOrtmann2026} gives $\mathbb{E}\big[H_{\vec n}(Y;I_d)^2\big]=\vec n!:=\prod_a n_a!\leq k!$. By the
Cauchy--Schwarz inequality,
\[
\mathbb{E}\abs{H_{c_1\cdots c_k}(Y;I_d)}\leq\Big(\mathbb{E}\big[H_{\vec n}(Y;I_d)^2\big]\Big)^{1/2}=\sqrt{k!}
\]
Combining these estimates gives the result.
\end{proof}

\begin{cor}
\label{cor:G-deriv-bound}
Let $S_0$ be positive definite and $D$ symmetric, both indexed by $A\times[m]$, and suppose $f$ is bounded and $2r$ times
differentiable with polynomially bounded derivatives. Then for every $t$ such that $S_0+tD$ is positive definite,
\[
\abs{\frac{d^r}{dt^r}G_f(S_0+tD)}\leq c^\prime_{r}\,\norm{f}_\infty\left(\frac{\norm{D}}{\lambda_{\min}(S_0+tD)}\right)^r,
\]
where $c^\prime_r$ depends only on $r$, $\abs{A}$ and $m$.
\end{cor}

\begin{proof}
Write $V:=S_0+tD$ and $X\sim\mathcal{N}(0,V)$. By Lemma \ref{lem:G-derivatives} and Gaussian integration by parts
\cite[Theorem 7.6]{NicaOrtmann2026}, $\mathbb{E}[(\partial_{a_1}\cdots\partial_{a_{2r}}f)(X)]=\mathbb{E}[f(X)H_{a_1\cdots a_{2r}}(X;V)]$, we have
\[
\frac{d^r}{dt^r}G_f(S_0+tD)=\frac1{2^r}\sum_{a_1,b_1,\ldots,a_r,b_r\in A\times[m]}\Big(\prod_{i=1}^rD_{a_ib_i}\Big)\,
\mathbb{E}\Big[f(X)\,H_{a_1b_1\cdots a_rb_r}(X;V)\Big].
\]
Each $\abs{D_{a_ib_i}}\leq\norm{D}$, and by Lemma \ref{lem:hermite-moment-bound} with $k=2r$ and $d=\abs{A}m$, each expectation is at
most $\norm{f}_\infty\,c_{2r,\abs{A}m}\,\lambda_{\min}(V)^{-r}$. There are $(\abs{A}m)^{2r}$ terms, so the bound follows with
$c^\prime_r:=2^{-r}(\abs{A}m)^{2r}c_{2r,\abs{A}m}$.
\end{proof}

\begin{rem}
Corollary \ref{cor:G-deriv-bound} involves only $\norm{f}_\infty$, even though Lemma \ref{lem:G-derivatives} expresses the derivatives
of $G_f$ through derivatives of $f$. This is because Gaussian integration by parts moves all the derivatives onto the Gaussian
density, where they become Hermite polynomials. The price is the factor $\lambda_{\min}(S_0+tD)^{-r}$, which blows up as the
covariance becomes singular. This is why the Taylor remainder below requires a separate argument when the random matrix
$\Si^\ell$ is far from its mean.
\end{rem}

\subsubsection{The Taylor remainder}

\begin{lem}
\label{lem:taylor-remainder}
Suppose $\bar{\Si}^\ell$ is positive definite, let $N\geq1$, and suppose $f$ is bounded and $2(N+1)$ times differentiable with
polynomially bounded derivatives. Then there is a constant $C_N$, depending only on $N$, $\abs{A}$ and $m$, such that for every
realization of the random matrix $\Delta^\ell$,
\[
\abs{\,g_f(1)-\sum_{r=0}^N\frac{g_f^{(r)}(0)}{r!}\,}\leq C_N\,\norm{f}_\infty\left(\frac{\norm{\Delta^\ell}}{\lambda_{\min}(\bar{\Si}^\ell)}\right)^{N+1}.
\]
\end{lem}

\noindent\emph{Idea of the proof.} The left-hand side is the error made by replacing $g_f(1)$ with its Taylor polynomial
of degree $N$ at $t=0$. Taylor's theorem expresses this error through the derivative $g_f^{(N+1)}$ somewhere on the segment
$[0,1]$, and Corollary \ref{cor:G-deriv-bound} bounds that derivative by $\big(\norm{\Delta^\ell}/\lambda_{\min}(S_t)\big)^{N+1}$, where
$S_t:=I_m\otimes(\bar{\Si}^\ell+t\Delta^\ell)$ is the covariance at time $t$. This is a good bound as long as $\lambda_{\min}(S_t)$ stays
comparable to $\lambda_{\min}(\bar{\Si}^\ell)$ along the whole segment, which is the case when the fluctuation $\Delta^\ell$ is small.
When the fluctuation is large, however, $S_t$ may become singular (or nearly so) for some $t\in(0,1]$; for instance, the
empirical covariance $\Si^\ell$ itself may be singular. In that case Taylor's theorem gives nothing useful, and we instead
bound the two sides separately, using only that $g_f(1)$ is an expectation of $f$. In both cases the bound comes out as a
power of $\norm{\Delta^\ell}/\lambda_{\min}(\bar{\Si}^\ell)$, which is what allows us to take expectations later using the moment bounds of
Lemma \ref{lem:hanin-estimates}.

\begin{proof}
For notational convenience, set $\lambda_*:=\lambda_{\min}(\bar{\Si}^\ell)>0$. Throughout the proof, write
$S_t:=I_m\otimes(\bar{\Si}^\ell+t\Delta^\ell)$, so that $g_f(t)=G_f(S_t)$ and $\lambda_{\min}(S_0)=\lambda_*$.

\emph{Preliminaries.} The value $g_f(1)=G_f(I_m\otimes\Si^\ell)$ is well defined even when $\Si^\ell$ is singular, as the expectation of $f$
under a (possibly degenerate) Gaussian law, and in all cases
\begin{equation}
\label{eq:crude-g1}
\abs{g_f(1)}\leq\norm{f}_\infty.
\end{equation}
The derivatives $g_f^{(r)}(0)$ are also well defined, since $S_0$ is positive definite and $G_f$ is smooth on positive definite
matrices \cite[Lemma 13]{Hanin2024}. By Corollary \ref{cor:G-deriv-bound} at $t=0$, using $\norm{I_m\otimes\Delta^\ell}=\norm{\Delta^\ell}$,
\begin{equation}
\label{eq:derivs-at-0}
\abs{g_f^{(r)}(0)}\leq c^\prime_r\,\norm{f}_\infty\left(\frac{\norm{\Delta^\ell}}{\lambda_*}\right)^r\qquad\text{for each }r\geq0.
\end{equation}
We now consider two cases, according to whether the fluctuation is small or large compared to $\lambda_*$.

\emph{Case 1: small fluctuation, $\frac{\norm{\Delta^\ell} }{\lambda_*} \leq\frac12$.} By Weyl's inequality, the minimum eigenvalue of a sum of
symmetric matrices satisfies $\lambda_{\min}(M_1+M_2)\geq\lambda_{\min}(M_1)-\norm{M_2}$. Applying this with $M_1=S_0$ and
$M_2=t\,I_m\otimes\Delta^\ell$ gives, for every $t\in[0,1]$,
\[
\lambda_{\min}(S_t)\geq\lambda_*-t\norm{\Delta^\ell}\geq\lambda_*-\tfrac12\lambda_*=\tfrac12\lambda_*.
\]
So $S_t$ is positive definite for all $t$ in an open interval containing $[0,1]$, and hence $g_f$ is $N+1$ times differentiable on
$[0,1]$. Taylor's theorem with the Lagrange form of the remainder gives some $\xi\in(0,1)$ with
\[
g_f(1)-\sum_{r=0}^N\frac{g_f^{(r)}(0)}{r!}=\frac{g_f^{(N+1)}(\xi)}{(N+1)!}.
\]
By Corollary \ref{cor:G-deriv-bound} with $S_0=I_m\otimes\bar{\Si}^\ell$, $D=I_m\otimes\Delta^\ell$ and $t=\xi$, using
$\lambda_{\min}(S_\xi)\geq\frac12\lambda_*$,
\[
\abs{g_f^{(N+1)}(\xi)}\leq c^\prime_{N+1}\,\norm{f}_\infty\left(\frac{\norm{\Delta^\ell}}{\lambda_*/2}\right)^{N+1}
=2^{N+1}c^\prime_{N+1}\,\norm{f}_\infty\left(\frac{\norm{\Delta^\ell}}{\lambda_*}\right)^{N+1}.
\]
Hence in this case the left-hand side of the lemma is at most
$\frac{2^{N+1}c^\prime_{N+1}}{(N+1)!}\norm{f}_\infty\big(\norm{\Delta^\ell}/\lambda_*\big)^{N+1}$.

\emph{Case 2: large fluctuation, $\frac{\norm{\Delta^\ell}}{\lambda_*}>\frac12$.} In this case we do not use Taylor's theorem. By the triangle
inequality, \eqref{eq:crude-g1} and \eqref{eq:derivs-at-0},
\[
\abs{\,g_f(1)-\sum_{r=0}^N\frac{g_f^{(r)}(0)}{r!}\,}\leq\abs{g_f(1)}+\sum_{r=0}^N\frac{\abs{g_f^{(r)}(0)}}{r!}
\leq\norm{f}_\infty\Big(1+\sum_{r=0}^N\frac{c^\prime_r}{r!}\left(\frac{\norm{\Delta^\ell}}{\lambda_*}\right)^r\Big).
\]
The right-hand side involves powers of $\norm{\Delta^\ell}/\lambda_*$ of order at most $N$, while we want a bound by the power $N+1$.
Since $\norm{\Delta^\ell}/\lambda_*>\frac12$ in this case, we have $1<2^{N+1}\left(\norm{\Delta^\ell}/\lambda_*\right)^{N+1}$ and, more
generally, for $0\leq r\leq N$,
\[
\left(\frac{\norm{\Delta^\ell}}{\lambda_*}\right)^r=\left(\frac{\norm{\Delta^\ell}}{\lambda_*}\right)^{N+1}\left(\frac{\norm{\Delta^\ell}}{\lambda_*}\right)^{-(N+1-r)}
<2^{N+1-r}\left(\frac{\norm{\Delta^\ell}}{\lambda_*}\right)^{N+1}.
\]
Using these in each term, the left-hand side of the lemma is at most
$\big(2^{N+1}+\sum_{r=0}^N\frac{c^\prime_r}{r!}2^{N+1-r}\big)\norm{f}_\infty\big(\norm{\Delta^\ell}/\lambda_*\big)^{N+1}$ in this case.

\emph{Conclusion.} In both cases, the left-hand side of the lemma is at most
$C_N\norm{f}_\infty\big(\norm{\Delta^\ell}/\lambda_*\big)^{N+1}$, where $C_N$ is the larger of the two constants obtained above,
which depends only on $N$, $\abs{A}$ and $m$. Recalling that $\lambda_*=\lambda_{\min}(\bar{\Si}^\ell)$, this is the claimed bound.
\end{proof}

\begin{rem}
\label{rem:localization-comparison}
A similar splitting into small and large fluctuations is used by \cite{Hanin2024} in the proofs of \cite[Proposition 17]{Hanin2024}
and \cite[Theorem 4]{Hanin2024}. There, one works on the event $S_n=\set{\max_{\alpha,\beta}\abs{\Delta^\ell(\alpha,\beta)}<n^{-1/4}}$,
whose threshold shrinks with $n$, and discards its complement using $\mathbb{P}(S_n^c)=O(n^{-q})$ for every $q$
\cite[eq. (43)]{Hanin2024}, which in turn relies on moment bounds of all orders. In Lemma \ref{lem:taylor-remainder} the threshold
$\frac12\lambda_{\min}(\bar{\Si}^\ell)$ does not shrink with $n$, and the large fluctuation case is not discarded as improbable: instead,
the error there is bounded deterministically by the same power $\big(\norm{\Delta^\ell}/\lambda_{\min}(\bar{\Si}^\ell)\big)^{N+1}$ as in the small
fluctuation case. As a result, controlling the expected remainder requires only the moment $\mathbb{E}\norm{\Delta^\ell}^{N+1}$.
\end{rem}

\subsubsection{The third order term}

\begin{lem}
\label{lem:third-order}
Suppose $\bar{\Si}^\ell$ is positive definite, and suppose $f$ is bounded and six times differentiable with polynomially bounded
derivatives. Then there is a constant $C$, depending only on $\abs{A}$ and $m$, such that
\[
\abs{\mathbb{E}\left[g_f^{\prime\prime\prime}(0)\right]}\leq C\,\norm{f}_\infty\,\lambda_{\min}\left(\bar{\Si}^\ell\right)^{-3}\,
\max_{\alpha_1,\beta_1,\alpha_2,\beta_2,\alpha_3,\beta_3\in A}\abs{\mathbb{E}\left[\Delta^\ell(\alpha_1,\beta_1)\,\Delta^\ell(\alpha_2,\beta_2)\,\Delta^\ell(\alpha_3,\beta_3)\right]}.
\]
\end{lem}

\begin{proof}
Write $X\sim\mathcal{N}(0,I_m\otimes\bar{\Si}^\ell)$. By Lemma \ref{lem:G-derivatives} with $r=3$, $S_0=I_m\otimes\bar{\Si}^\ell$ and
$D=I_m\otimes\Delta^\ell$, followed by Gaussian integration by parts \cite[Theorem 7.6]{NicaOrtmann2026},
\[
g_f^{\prime\prime\prime}(0)=\frac18\sum_{a_1,b_1,a_2,b_2,a_3,b_3\in A\times[m]}\Big(\prod_{i=1}^3\left(I_m\otimes\Delta^\ell\right)_{a_ib_i}\Big)\,h_{a_1b_1a_2b_2a_3b_3},
\qquad h_{a_1\cdots b_3}:=\mathbb{E}\Big[f(X)\,H_{a_1b_1a_2b_2a_3b_3}\big(X;I_m\otimes\bar{\Si}^\ell\big)\Big].
\]
The coefficients $h$ are deterministic, and by Lemma \ref{lem:hermite-moment-bound} with $k=6$ and $d=\abs{A}m$ they satisfy
$\abs{h}\leq c_{6,\abs{A}m}\norm{f}_\infty\lambda_{\min}(\bar{\Si}^\ell)^{-3}$. Taking expectations, each term involves
$\mathbb{E}\big[\prod_{i=1}^3(I_m\otimes\Delta^\ell)_{a_ib_i}\big]$, which is either zero (because of the Kronecker deltas in
$I_m\otimes\Delta^\ell$) or equal to a third moment $\mathbb{E}[\Delta^\ell(\alpha_1,\beta_1)\Delta^\ell(\alpha_2,\beta_2)\Delta^\ell(\alpha_3,\beta_3)]$.
Bounding each of the $(\abs{A}m)^6$ terms gives the result with $C:=\frac18(\abs{A}m)^6c_{6,\abs{A}m}$.
\end{proof}

\subsubsection{Summary of all Taylor series error bounds}

\begin{cor}
\label{cor:error-summary}
Under the assumptions of Lemma \ref{lem:hanin-estimates}, let $n_0$ be as in Lemma \ref{lem:hanin-estimates}(a). There is a constant
$C$ such that for all $n\geq n_0$ and all bounded test functions $f$ that are eight times differentiable with polynomially
bounded derivatives,
\begin{align*}
\mathbb{E}\abs{\,g_f(1)-\Big( g_f(0)+g_f^\prime(0) \Big) \,}&\leq C\,\norm{f}_\infty\,n^{-1},\\
\mathbb{E}\abs{\,g_f(1)-\Big( g_f(0)+g_f^\prime(0)+\tfrac12g_f^{\prime\prime}(0)+\tfrac16g_f^{\prime\prime\prime}(0) \Big)\,}&\leq C\,\norm{f}_\infty\,n^{-2},\\
\abs{\mathbb{E}\left[g_f^{\prime\prime\prime}(0)\right]}&\leq C\,\norm{f}_\infty\,n^{-2}.
\end{align*}
\end{cor}

\begin{proof}
For $n\geq n_0$ we have $\lambda_{\min}(\bar{\Si}^\ell)\geq\lambda_0$ by Lemma \ref{lem:hanin-estimates}(a). In particular $\bar{\Si}^\ell$ is
positive definite, so Lemmas \ref{lem:taylor-remainder} and \ref{lem:third-order} apply, and $\lambda_{\min}(\bar{\Si}^\ell)^{-1}\leq\lambda_0^{-1}$.
The first two bounds follow by taking expectations in Lemma \ref{lem:taylor-remainder} with $N=1$ and $N=3$, bounding
$\lambda_{\min}(\bar{\Si}^\ell)^{-1}\leq\lambda_0^{-1}$, and using $\mathbb{E}\norm{\Delta^\ell}^2=O(n^{-1})$ and $\mathbb{E}\norm{\Delta^\ell}^4=O(n^{-2})$
from Lemma \ref{lem:hanin-estimates}(b). The third bound follows in the same way from Lemma \ref{lem:third-order}, bounding
$\lambda_{\min}(\bar{\Si}^\ell)^{-3}\leq\lambda_0^{-3}$, together with Lemma \ref{lem:hanin-estimates}(c).
\end{proof}

\subsection{Proof of Theorem \ref{thm:TV-gauss-pg}}

Finally, we combine all the ingredients to finish the proof of Theorem \ref{thm:TV-gauss-pg}. The lemmas of this section are stated for
test functions $f$ that are sufficiently differentiable, while the total variation distance is a supremum over all bounded measurable
$f$. The following lemma reduces the problem to smooth $f$.

\begin{lem}
\label{lem:TV-smooth}
Let $\mu$ be a finite signed Borel measure on $\bR^d$. Then
\[
\sup_{\substack{f\text{ measurable}\\ \norm{f}_\infty\leq1}}\abs{\intop f\,\d\mu}
=\sup_{\substack{f\in C_c^\infty(\bR^d)\\ \norm{f}_\infty\leq1}}\abs{\intop f\,\d\mu}.
\]
\end{lem}
This lemma is standard and follows from the density of smooth compactly supported functions in $L^1(\mu)$.

\begin{proof}[Proof of Theorem \ref{thm:TV-gauss-pg}]
Let $n_0$ be as in Lemma \ref{lem:hanin-estimates}(a). Since Definition \ref{def:asymptotically-distributed} only concerns the behaviour as
$n\to\infty$, it suffices to prove both bounds for $n\geq n_0$. For such $n$, the mean covariance $\bar\Si^\ell$ is positive definite,
so the Gaussian law $\mathcal{N}(0,I_m\otimes\bar\Si^\ell)$ and the para-Gaussian law $\pgl{I_m\otimes\bar\Si^\ell}{n\ka^\ell}{n}$ are well defined,
and both are finite signed measures on $\bR^{A\times[m]}$. By Definition \ref{def:TV} and Lemma \ref{lem:TV-smooth}, applied to the
difference between the law of $Z^\ell$ and each of these measures, it suffices to bound
\[
\abs{\mathbb{E}[f(Z^\ell)]-\intop f\,\d\nu}
\]
uniformly over $f\in C_c^\infty(\bR^{A\times[m]})$ with $\norm{f}_\infty\leq1$, for $\nu$ each of the two approximating measures. Fix such
an $f$. Since $f$ is bounded with bounded derivatives of all orders, it satisfies the hypotheses of all the lemmas of this section.

\emph{Expansion of $\mathbb{E}[f(Z^\ell)]$.} By conditional Gaussianity (see the remark following the definition of $g_f$),
$\mathbb{E}[f(Z^\ell)]=\mathbb{E}[g_f(1)]$. The derivatives $g_f^{(r)}(0)$, $r\leq3$, are polynomials of degree $r$ in the entries of
$\Delta^\ell$ with deterministic coefficients (Lemma \ref{lem:G-derivatives}), so they are integrable by Lemma \ref{lem:hanin-estimates}(b).
For any $N\leq3$ we may therefore write
\begin{equation}
\label{eq:main-proof-expansion}
\mathbb{E}[f(Z^\ell)]=\sum_{r=0}^N\frac1{r!}\,\mathbb{E}\Big[g_f^{(r)}(0)\Big]+\mathbb{E}\Big[g_f(1)-\sum_{r=0}^N\frac{g_f^{(r)}(0)}{r!}\Big].
\end{equation}

\emph{Part (i).} Take $N=1$ in \eqref{eq:main-proof-expansion}. The zeroth order term is deterministic and equals the Gaussian
approximation, $g_f(0)=\mathbb{E}_{X\sim\mathcal{N}(0,I_m\otimes\bar\Si^\ell)}[f(X)]$, and the first order term vanishes,
$\mathbb{E}[g_f^\prime(0)]=0$, by Lemma \ref{lem:g-prime-mean-zero}. Hence
\[
\abs{\mathbb{E}[f(Z^\ell)]-\mathbb{E}_{X\sim\mathcal{N}(0,I_m\otimes\bar\Si^\ell)}[f(X)]}
=\abs{\mathbb{E}\Big[g_f(1)-g_f(0)-g_f^\prime(0)\Big]}\leq C\,n^{-1}
\]
by Corollary \ref{cor:error-summary}, where $C$ does not depend on $f$. Taking the supremum over $f$ gives
$d_{TV}\big(Z^\ell,\mathcal{N}(0,I_m\otimes\bar\Si^\ell)\big)\leq\frac C2\,n^{-1}$ for all $n\geq n_0$.

\emph{Part (ii).} Take $N=3$ in \eqref{eq:main-proof-expansion}. By Proposition \ref{prop:second-order-is-PG}, the expectation of the terms
of order at most two is exactly the para-Gaussian law,
\[
\mathbb{E}\Big[g_f(0)+g_f^\prime(0)+\tfrac12g_f^{\prime\prime}(0)\Big]=\intop f\,\d \pgl{I_m\otimes\bar\Si^\ell}{n\ka^\ell}{n}.
\]
Hence
\[
\mathbb{E}[f(Z^\ell)]-\intop f\,\d \pgl{I_m\otimes\bar\Si^\ell}{n\ka^\ell}{n}
=\frac16\,\mathbb{E}\Big[g_f^{\prime\prime\prime}(0)\Big]
+\mathbb{E}\Big[g_f(1)-g_f(0)-g_f^\prime(0)-\tfrac12g_f^{\prime\prime}(0)-\tfrac16g_f^{\prime\prime\prime}(0)\Big].
\]
By Corollary \ref{cor:error-summary}, both terms on the right-hand side are bounded in absolute value by $C\,n^{-2}$, with $C$ not
depending on $f$. Taking the supremum over $f$ gives
$d_{TV}\big(Z^\ell,\pgl{I_m\otimes\bar\Si^\ell}{n\ka^\ell}{n}\big)\leq\frac76\,C\,n^{-2}$ for all $n\geq n_0$ as desired.
\end{proof}

\section{Derivation of recursion relations, Theorem \ref{thm:main}}
\label{sec:proof-main}

In this section we prove Theorem \ref{thm:main}. Throughout, we work under the assumptions of Lemma \ref{lem:hanin-estimates}. We
assume in addition that $\vp$ is continuous and take $n\geq n_0$, so that $\bar\Si^\ell$ is positive definite. We write
$Z^\ell_j=(z^\ell_{\ga,j})_{\ga\in A}$ for the $j$-th neuron of layer $\ell$. As in the statement of Theorem \ref{thm:main}, $\GaussExp{}$
denotes an expectation in which the neurons of layer $\ell$ are replaced by independent Gaussian vectors with covariance $\bar\Si^\ell$:
for a function $F$ of $m$ neurons,
\[
\GaussExp{}\big[F\big]:=\mathbb{E}_{X\sim\mathcal{N}(0,\,I_m\otimes\bar\Si^\ell)}\big[F(X)\big],
\]
so that for $m=1$ this is the expectation over a single Gaussian neuron with covariance $\bar\Si^\ell$, and for $m=2$ it is the
expectation over two independent such neurons $X_1,X_2$. Expectations written $\mathbb{E}$ are always over the true network. All Hermite
polynomials $H_{\ga_1\cdots\ga_k}$ of a single neuron are taken with respect to the covariance $\bar\Si^\ell$.

\subsection{Estimating expectation of polynomials under para-Gaussian law}

The first estimate we need says that expectations of polynomially bounded functions of finitely many neurons can be computed with the
\pg\ law, up to an error of order $n^{-2}$. It is a version of the $\frac1n$ expansion of \cite[Theorem 4]{Hanin2024}, which we
import in the language of this paper. (We need this for functions such as $\vp(z_\al)\vp(z_{\al'})\vp(z_\be)\vp(z_{\be'})$, which are not
bounded, so that Theorem \ref{thm:TV-gauss-pg} bounding in terms of $\norm{f}_\infty$ does not apply directly. An alternative approach would be to truncate these functions at some large value at separately argue about the truncation error.)
\begin{lem}[\Pg\ expectations of polynomially bounded functions]
\label{lem:PG-expectation-poly}
Fix $m\geq1$ and let $F:\bR^{A\times[m]}\to\bR$ be continuous and polynomially bounded. Then
\begin{equation}
\label{eq:PG-expectation-poly}
\mathbb{E}\big[F(Z^\ell)\big]=\GaussExp{}\big[F\big]+\frac1{8n}\sum_{i,k\in[m]}\ \sum_{\ga_1,\ga_2,\ga_3,\ga_4\in A}C^\ell(\ga_1,\ga_2;\ga_3,\ga_4)\,
\GaussExp{}\Big[F\,H_{(\ga_1,i)(\ga_2,i)(\ga_3,k)(\ga_4,k)}\big(\,\cdot\,;I_m\otimes\bar\Si^\ell\big)\Big]+O\left(\frac1{n^2}\right),
\end{equation}
that is, $\mathbb{E}[F(Z^\ell)]=\int F\,\d \pgl{I_m\otimes\bar\Si^\ell}{n\ka^\ell}{n}+O(n^{-2})$. In particular,
$\mathbb{E}[F(Z^\ell)]=\GaussExp{}[F]+O\big(\frac1n\big)$.
\end{lem}

\begin{proof}
By \cite[Theorem 4]{Hanin2024} with $r=0$ and $q_*=1$, applied with the dictionary of the proof of Lemma \ref{lem:hanin-estimates},
$\mathbb{E}[F(Z^\ell)]$ is equal, up to an error of order $n^{-2}$, to the sum of three terms: 1. the Gaussian expectation $\GaussExp{}[F]$, 2. a
term linear in the fluctuation $\Delta^\ell$, which vanishes since $\mathbb{E}[\Delta^\ell]=0$, and 3. the term
$\frac18\sum_{a,b,c,d}T^\ell_{abcd}\,\GaussExp{}\big[\partial_a\partial_b\partial_c\partial_dF\big]$, where $T^\ell$ is the covariance tensor of
\eqref{eq:T-def} and the derivatives are taken in the weak sense. This last term is the same expression that appears in Lemma
\ref{lem:g-second-deriv}. Moving the derivatives onto the Gaussian density as in Lemma \ref{lem:gaussian-ibp} turns it into
$\frac18\sum_{a,b,c,d}T^\ell_{abcd}\,\GaussExp{}\big[F\,H_{abcd}(\,\cdot\,;I_m\otimes\bar\Si^\ell)\big]$, and collapsing the Kronecker deltas in
$T^\ell_{(\ga_1,i)(\ga_2,j)(\ga_3,k)(\ga_4,l)}=\frac1n\de_{ij}\de_{kl}C^\ell(\ga_1,\ga_2;\ga_3,\ga_4)$ gives \eqref{eq:PG-expectation-poly}. The
identification with the \pg\ law follows from Corollary \ref{cor:PG-two-forms}, as in the proof of Proposition
\ref{prop:second-order-is-PG}. Finally, $C^\ell$ is bounded uniformly in $n$ by Lemma \ref{lem:hanin-estimates}(b), so the correction
term is of order $\frac1n$.
\end{proof}

\subsection{Spliting Hermite polynomials}

 The covariance is $ I_m\otimes\bar\Si^\ell$ which is block
diagonal. Under the Gaussian law $\GaussExp{}$, this means the $m$ neurons are independent Gaussian vectors of size $|A|$. The Hermite polynomials inherit this structure, as the following lemma makes precise

\begin{lem}[Hermite polynomials split over independent neurons]
\label{lem:hermite-splits}
Let $S$ be a positive definite $A\times A$ matrix and $m\geq1$. For any indices $(\ga_1,i_1),\ldots,(\ga_k,i_k)\in A\times[m]$ and any
$x=(x_1,\ldots,x_m)$ with $x_i\in\bR^A$,
\[
H_{(\ga_1,i_1)\cdots(\ga_k,i_k)}\big(x;I_m\otimes S\big)=\prod_{i=1}^m H_{\set{\ga_r\,:\,i_r=i}}\big(x_i;S\big),
\]
where $H_{\set{\ga_r:i_r=i}}(x_i;S)$ is the Hermite polynomial in the variable $x_i$ of those indices that belong to neuron $i$, and is
equal to $1$ if there are none.
\end{lem}

\begin{proof}
By the Gaussian integral formula \cite[Definition 1.4, eq. (7)]{NicaOrtmann2026} with $V=I_m\otimes S$,
\[
H_{(\ga_1,i_1)\cdots(\ga_k,i_k)}(x;V)=\mathbb{E}_{Z\sim\mathcal{N}(0,V^{-1})}\Big[\prod_{r=1}^k\Big(\big(V^{-1}x\big)_{(\ga_r,i_r)}+\imath Z_{(\ga_r,i_r)}\Big)\Big].
\]
Since $V^{-1}=I_m\otimes S^{-1}$, we have $(V^{-1}x)_{(\ga,i)}=(S^{-1}x_i)_\ga$, and the blocks $(Z_{(\ga,i)})_{\ga\in A}$ of $Z$ belonging to
different neurons are independent, each with law $\mathcal{N}(0,S^{-1})$. Grouping the factors of the product according to their
neuron, the expectation factorizes into one factor for each neuron, and the factor of neuron $i$ is
$H_{\set{\ga_r:i_r=i}}(x_i;S)$, again by the Gaussian integral formula.
\end{proof}

\subsection{Expectations of one neuron and covariances of two neurons}

We now compute the three kinds of quantities that $\bar\Si^{\ell+1}$ and $C^{\ell+1}$ are made of. In the following, $f$, $f_1$ and $f_2$
denote continuous polynomially bounded functions on $\bR^A$, which we apply to single neurons; in the proof of Theorem \ref{thm:main}
they will be products of the form $\vp_\al\vp_\be$. We write $\sum_{\vec\ga}$ for the sum over $\ga_1,\ga_2,\ga_3,\ga_4\in A$ and
$C^\ell(\vec\ga)=C^\ell(\ga_1,\ga_2;\ga_3,\ga_4)$.

\begin{lem}[Expectations of one neuron]
\label{lem:one-neuron}
\[
\mathbb{E}\big[f(Z^\ell_1)\big]=\GaussExp{}\big[f\big]+\frac1{8n}\sum_{\vec\ga}C^\ell(\vec\ga)\,\GaussExp{}\big[f\,H_{\ga_1\ga_2\ga_3\ga_4}\big]+O\left(\frac1{n^2}\right).
\]
\end{lem}

\begin{proof}
This is Lemma \ref{lem:PG-expectation-poly} with $m=1$, in which case the neuron indices are $i=k=1$ and the Hermite polynomial is the
fourth order Hermite polynomial $H_{\ga_1\ga_2\ga_3\ga_4}$ of the single neuron.
\end{proof}

The leading term in Lemma \ref{lem:one-neuron} is what one would obtain if the neuron were exactly Gaussian, and the correction is
the fourth order Hermite term of the \pg\ law. For $f=\vp_\al\vp_\be$ this correction is the reason that, for a non-linear $\vp$ such
as ReLU, the mean covariance differs from the infinite width kernel at order $\frac1n$. For the identity, $\vp_\al\vp_\be$ is a polynomial
of degree two, which is orthogonal to the fourth order Hermite polynomials, and the correction vanishes.

\begin{lem}[Covariance within one neuron]
\label{lem:same-neuron}
\[
\cov\big[f_1(Z^\ell_1),\,f_2(Z^\ell_1)\big]=\GaussExp{}\big[f_1f_2\big]-\GaussExp{}\big[f_1\big]\,\GaussExp{}\big[f_2\big]+O\left(\frac1n\right).
\]
\end{lem}

\begin{proof}
Apply the last statement of Lemma \ref{lem:PG-expectation-poly}, with $m=1$, to the three functions $f_1f_2$, $f_1$ and $f_2$.
\end{proof}

Lemma \ref{lem:same-neuron} describes a covariance that is present even when the input to the layer is exactly Gaussian: two products
$f_1$ and $f_2$ of the same Gaussian neuron are in general correlated. In contrast, two \emph{different} Gaussian neurons are
independent, so the covariance between them vanishes to leading order; its value comes entirely from the \pg\ correction, and it is
here that the splitting of Lemma \ref{lem:hermite-splits} enters.

\begin{lem}[Covariance between two different neurons]
\label{lem:different-neurons}
\[
\cov\big[f_1(Z^\ell_1),\,f_2(Z^\ell_2)\big]=\frac1{4n}\sum_{\vec\ga}C^\ell(\vec\ga)\,\GaussExp{}\big[f_1H_{\ga_1\ga_2}\big]\,\GaussExp{}\big[f_2H_{\ga_3\ga_4}\big]+O\left(\frac1{n^2}\right).
\]
\end{lem}

\begin{proof}
We apply Lemma \ref{lem:PG-expectation-poly} with $m=2$ to the function $F(x_1,x_2)=f_1(x_1)f_2(x_2)$ of two neurons. In the correction
term of \eqref{eq:PG-expectation-poly}, the four indices of the Hermite polynomial are $(\ga_1,i),(\ga_2,i),(\ga_3,k),(\ga_4,k)$, with neuron
labels $i,k\in\set{1,2}$. By Lemma \ref{lem:hermite-splits}, this Hermite polynomial splits according to which neuron each index belongs
to, and since $X_1$ and $X_2$ are independent, each expectation then factorizes over the two neurons. There are two kinds of terms.
\begin{itemize}
\item If $i=k$, all four indices belong to the same neuron, and the fourth order Hermite polynomial stays on that neuron. For
$i=k=1$ the term is $\GaussExp{}\big[f_1H_{\ga_1\ga_2\ga_3\ga_4}\big]\,\GaussExp{}\big[f_2\big]$, and for $i=k=2$ it is
$\GaussExp{}\big[f_1\big]\,\GaussExp{}\big[f_2H_{\ga_1\ga_2\ga_3\ga_4}\big]$.
\item If $i\neq k$, two indices belong to each neuron, and the fourth order Hermite polynomial splits into a product of two second
order Hermite polynomials, one on each neuron: for $(i,k)=(1,2)$,
\[
H_{(\ga_1,1)(\ga_2,1)(\ga_3,2)(\ga_4,2)}\big(X;I_2\otimes\bar\Si^\ell\big)=H_{\ga_1\ga_2}(X_1)\,H_{\ga_3\ga_4}(X_2),
\]
and the term is $\GaussExp{}\big[f_1H_{\ga_1\ga_2}\big]\,\GaussExp{}\big[f_2H_{\ga_3\ga_4}\big]$. The case $(i,k)=(2,1)$ gives the same contribution
after relabelling $(\ga_1,\ga_2)\leftrightarrow(\ga_3,\ga_4)$, by the symmetry $C^\ell(\ga_1,\ga_2;\ga_3,\ga_4)=C^\ell(\ga_3,\ga_4;\ga_1,\ga_2)$.
\end{itemize}
Altogether,
\begin{align*}
\mathbb{E}\big[f_1(Z^\ell_1)f_2(Z^\ell_2)\big]
&=\GaussExp{}[f_1]\,\GaussExp{}[f_2]
+\frac1{8n}\sum_{\vec\ga}C^\ell(\vec\ga)\Big(\GaussExp{}\big[f_1H_{\ga_1\ga_2\ga_3\ga_4}\big]\,\GaussExp{}[f_2]+\GaussExp{}[f_1]\,\GaussExp{}\big[f_2H_{\ga_1\ga_2\ga_3\ga_4}\big]\Big)\\
&\quad+\frac1{4n}\sum_{\vec\ga}C^\ell(\vec\ga)\,\GaussExp{}\big[f_1H_{\ga_1\ga_2}\big]\,\GaussExp{}\big[f_2H_{\ga_3\ga_4}\big]+O\left(\frac1{n^2}\right).
\end{align*}
On the other hand, multiplying the expansions of Lemma \ref{lem:one-neuron} for $f_1$ and for $f_2$ gives
\[
\mathbb{E}\big[f_1(Z^\ell_1)\big]\,\mathbb{E}\big[f_2(Z^\ell_2)\big]
=\GaussExp{}[f_1]\,\GaussExp{}[f_2]
+\frac1{8n}\sum_{\vec\ga}C^\ell(\vec\ga)\Big(\GaussExp{}\big[f_1H_{\ga_1\ga_2\ga_3\ga_4}\big]\,\GaussExp{}[f_2]+\GaussExp{}[f_1]\,\GaussExp{}\big[f_2H_{\ga_1\ga_2\ga_3\ga_4}\big]\Big)
+O\left(\frac1{n^2}\right).
\]
Subtracting, the terms in which the Hermite polynomial stays on one neuron cancel exactly, and only the terms coupling the two
neurons remain, which is the claim.
\end{proof}

The covariance in Lemma \ref{lem:different-neurons} is of order $\frac1n$, as expected for two neurons that are uncorrelated but not
independent. It is built from a pair of second order Hermite polynomials, one on each neuron, which arise from splitting the fourth
order Hermite polynomial of the \pg\ law between the two neurons.
\subsection{Proof of Theorem \ref{thm:main}}

\begin{proof}[Proof of Theorem \ref{thm:main}]
Recall the definition of the random covariance matrix of layer $\ell+1$,
\[
\Si^{\ell+1}(\al,\be)=\frac1n\sum_{j=1}^n\vp\big(z^\ell_{\al,j}\big)\vp\big(z^\ell_{\be,j}\big),
\]
and recall that the neurons $Z^\ell_1,\ldots,Z^\ell_n$ of layer $\ell$ are exchangeable.

\emph{Mean covariance.} Taking expectations, each of the $n$ terms in the definition of $\Si^{\ell+1}$ has the same expectation, so
\[
\bar\Si^{\ell+1}(\al,\be)=\mathbb{E}\big[\Si^{\ell+1}(\al,\be)\big]=\mathbb{E}\big[\vp(z^\ell_{\al,1})\vp(z^\ell_{\be,1})\big].
\]
This is the expectation of the function $f=\vp_\al\vp_\be$ of the single neuron $Z^\ell_1$, so by Lemma \ref{lem:one-neuron},
\[
\bar\Si^{\ell+1}(\al,\be)=\GaussExp{}\big[\vp_\al\vp_\be\big]+\frac1{8n}\sum_{\vec\ga}C^\ell(\vec\ga)\,\GaussExp{}\big[\vp_\al\vp_\be\,H_{\ga_1\ga_2\ga_3\ga_4}\big]
+O\left(\frac1{n^2}\right),
\]
which is \eqref{eq:mainSigmaExpansion}.

\emph{Fourth cumulant: splitting into one neuron and two neurons.} Write $f_1=\vp_\al\vp_{\al'}$ and $f_2=\vp_\be\vp_{\be'}$, so that
$\Si^{\ell+1}(\al,\al')=\frac1n\sum_jf_1(Z^\ell_j)$ and $\Si^{\ell+1}(\be,\be')=\frac1n\sum_kf_2(Z^\ell_k)$. By the definition of $C^{\ell+1}$ and
bilinearity of the covariance,
\[
C^{\ell+1}(\al,\al';\be,\be')=n\cov\big[\Si^{\ell+1}(\al,\al'),\Si^{\ell+1}(\be,\be')\big]
=\frac1n\sum_{j,k=1}^n\cov\big[f_1(Z^\ell_j),\,f_2(Z^\ell_k)\big].
\]
By exchangeability, the $n$ terms with $j=k$ are all equal to the covariance within the first neuron, and the $n(n-1)$ terms with
$j\neq k$ are all equal to the covariance between the first two neurons. Hence
\begin{equation}
\label{eq:C-split-proof}
C^{\ell+1}(\al,\al';\be,\be')=\underbrace{\cov\big[f_1(Z^\ell_1),\,f_2(Z^\ell_1)\big]}_{\text{one neuron}}
+(n-1)\underbrace{\cov\big[f_1(Z^\ell_1),\,f_2(Z^\ell_2)\big]}_{\text{two different neurons}}.
\end{equation}

\emph{The term within one neuron.} By Lemma \ref{lem:same-neuron},
\begin{equation}
\label{eq:proof-same}
\cov\big[f_1(Z^\ell_1),\,f_2(Z^\ell_1)\big]
=\GaussExp{}\big[\vp_\al\vp_{\al'}\vp_\be\vp_{\be'}\big]-\GaussExp{}\big[\vp_\al\vp_{\al'}\big]\,\GaussExp{}\big[\vp_\be\vp_{\be'}\big]+O\left(\frac1n\right).
\end{equation}

\emph{The term between two different neurons.} By Lemma \ref{lem:different-neurons},
\[
\cov\big[f_1(Z^\ell_1),\,f_2(Z^\ell_2)\big]
=\frac1{4n}\sum_{\vec\ga}C^\ell(\vec\ga)\,\GaussExp{}\big[\vp_\al\vp_{\al'}\,H_{\ga_1\ga_2}\big]\,\GaussExp{}\big[\vp_\be\vp_{\be'}\,H_{\ga_3\ga_4}\big]+O\left(\frac1{n^2}\right).
\]
In \eqref{eq:C-split-proof} this is multiplied by $n-1$. Since $\frac{n-1}{4n}=\frac14-\frac1{4n}$, and the sum over $\vec\ga$ is bounded
uniformly in $n$ (because $C^\ell$ is, by Lemma \ref{lem:hanin-estimates}(b)), the term $-\frac1{4n}$ only contributes at order $\frac1n$, and
so does $(n-1)\cdot O(n^{-2})$. Therefore
\begin{equation}
\label{eq:proof-different}
(n-1)\cov\big[f_1(Z^\ell_1),\,f_2(Z^\ell_2)\big]
=\frac14\sum_{\vec\ga}C^\ell(\vec\ga)\,\GaussExp{}\big[\vp_\al\vp_{\al'}\,H_{\ga_1\ga_2}\big]\,\GaussExp{}\big[\vp_\be\vp_{\be'}\,H_{\ga_3\ga_4}\big]+O\left(\frac1n\right).
\end{equation}

\emph{Combining.} Substituting \eqref{eq:proof-same} and \eqref{eq:proof-different} into \eqref{eq:C-split-proof},
\begin{align*}
C^{\ell+1}(\al,\al';\be,\be')
&=\GaussExp{}\big[\vp_\al\vp_{\al'}\vp_\be\vp_{\be'}\big]-\GaussExp{}\big[\vp_\al\vp_{\al'}\big]\,\GaussExp{}\big[\vp_\be\vp_{\be'}\big]\\
&\quad+\frac14\sum_{\ga_1,\ga_2,\ga_3,\ga_4\in A}C^\ell(\ga_1,\ga_2;\ga_3,\ga_4)\,\GaussExp{}\big[\vp_\al\vp_{\al'}\,H_{\ga_1\ga_2}\big]\,
\GaussExp{}\big[\vp_\be\vp_{\be'}\,H_{\ga_3\ga_4}\big]+O\left(\frac1n\right),
\end{align*}
which is \eqref{eq:mainC}, as desired. The first line comes from the covariance within one neuron and is present even when layer $\ell$ is
exactly Gaussian. The second line comes from the $n-1$ pairs of different neurons, each of which contributes a covariance of order
$\frac1n$ through a pair of second order Hermite polynomials. 
\end{proof}

%%%%%%%%%%%%%%%%%%%%%%%%%%%%%%%%%%%%%%%%%%%%%%%%%%%%%%%%%%%%%%%%%%%%%%%%%%%%%%%%%%%%%%%%%%%%%%%%%%%%%%%%%%%%%%%%%%%%%%%%%%%%%%%%%%%%%%%%%%%%%%%%%%%%%%%%%%%%%%%%%%%%%%%%%%%%%%%%%%%%%%%%%%%%%%%%%%%%%%%%%%%%%%%%%%%%%%%%%%%%%%%%%%%%%%%%%%%%%%%%%%%%%%%%%%%%%%%%%%%%%%%%%%%%%%%%%%%%%%%%%%%%%%%%

\appendix

    	\section{\Pg\ theory}
	
	\label{sec:pg-theory}
	
	Recall the \pg\ measure with scaling parameter $n$ and configuration ($\kappa$, $k$, $\Sigma$) as the measure $\pglk\Si\ka n k$ on $\bR^d$ with density $\pgdk \Si\ka n k$. Here we discuss some basic properties of this measure.% In Section \ref{subsec:second-proof-exact-general-cumulants} we give the second proof of Theorem \ref{thm:exact-general-cumulants}.
    
% \begin{rem}
%     \label{rem:PG-law}
%         In general the \pg\ measure is a finite signed measure. However,  orthogonality of the Hermite polynomials with respect to the Gaussian measure means that it always integrates to 1, see Proposition \ref{prop:integrability}. So if $\pgdk \Si \ka n k(x)\geq 0$ for all $x$ we get a probability measure. (For $n$ large enough this is always the case.) If so, we also talk about the \emph{\pg\ distribution} or \emph{law} that we denote by $\pglk\Si \ka n k$, or $\pgl\Si \ka n$ when $k=4$. 
%     \end{rem}

\label{sec:basic-pg}
        	\begin{rem}
	   	\label{rem:Fourier-PG}
    By the theory of multi-dimensional Hermite polynomials \cite{NicaOrtmann2026}, the \pg\ measure can be equivalently characterised as having Fourier transform (or characteristic function, if it is a probability measure)
\begin{align}
	 \label{eq:Fourier-PG}
	\pgfourierk \Si\ka n k (\imath \vec{u}) %&= \exp\left( -n \frac{1}{2} \left(\frac{\vec{u}}{\sqrt{n}}\right)^T \Sigma \left(\frac{\vec{u}}{\sqrt{n}}\right) + n \sum_{\substack{k_1,\ldots,k_d\\ \sum k_i = k}} \frac{\kappa_{k_1,\ldots,k_d}}{k_1!\cdots k_d!} \left(\imath \frac{u_1}{\sqrt{n}}\right)^{k_1}\cdots \left(\imath \frac{u_d}{\sqrt{n}}\right)^{k_d} + n \epsilon\left(\frac{\vec{u}}{\sqrt{n}}\right) \right) \nonumber \\ 
	&= \exp\left( -\frac{1}{2}\vec{u}^T \Sigma \vec{u} \right) \left(1+ \frac{1}{\sqrt{n}^{k-2}} \sum_{\substack{k_1,\ldots,k_d\\ \sum k_i = k}} \frac{\kappa_{k_1,\ldots,k_d}}{k_1!\cdots k_d!} (\imath \vec{u})^{\vec{k}} \right) 
\end{align}
where we use the notation $(i\vec u)^{\vec k}=\prod_{a=1}^d (iu_a)^{k_a}$.
	\end{rem}
	% [\textbf{TODO: } keep this remark or or remove it? If we use Hanin-style proof we might not need this.	Under mild conditions, such as existence of a second moment, this can be equivalently stated in terms of characteristic functions, see also Section  \ref{sec:pg-theory} below]

% ------
    
%         \textbf{TODO}: I took out the part about eventual uniform positivity, I don't think it holds in general: take $d=1$  and $k=4$ with $\kappa_4<0$. Then we get
%         $$\phi(x)\rb{1+\frac1{4!n} \kappa_4 H_4(x)}$$
%         and so the sign only depends on the term inside the bracket. Now $H_4(x)=x^4-6x^2+3$, so $1+\frac1n \kappa_4 H_4(x)$ is a quartic with negative coefficient.

%         Can go a bit further: let $Q(x)$ be the term inside the brackets, write $t=x^2$. Let's take a Bernoulli $\pm 1$ random variable which has $\ka_4=-2$. Then
%         $$
%         Q(x) = 1-\frac{t^2-6t+3}{12n} = 12n\rb{-t^2+6t+12n-3}
%         $$
%         which has roots at $t_\pm = 3\pm \sqrt{12n+6}$. In other words, the density is negative for $x\geq \sqrt{3+ \sqrt{12n+6}}$.

% -------        
    \begin{prop}
    \label{prop:integrability}
        Let $\Si$ be symmetric positive definite and $n,k\in\bN$. Then the function $\pgdk \Sigma\ka n k(\vec{x})$ is integrable and
            \begin{align}
            \label{eq:int-pg}
                \int_{\bR^d} \pgdk \Sigma\ka n k(x)\,\d x = 1.
            \end{align}
            Thus, if $\pgdk \Sigma\ka n k (\vec{x})\geq0$ a.e.  then it is the density of a probability measure (that we call the \pg\ distribution or law).
        %     \item For any configuration ($\kappa$, $k$, $\Sigma$) there exists $N$ such that $\pgdk \Sigma\ka n k$ is a probability measure for every $n\geq N$
        % \end{enumerate}
    \end{prop}

    \begin{proof}\ 
    %\begin{enumerate}
    The fact that $\pgdk\Si\ka n k$ is integrable follows from the fact that can be written as a polynomial multiplied with a Gaussian density. Equation \eqref{eq:int-pg} follows from orthogonality of the Hermite polynomials with respect to the Gaussian measure, see Theorem 2.4 of \cite{NicaOrtmann2026}. Concretely, we are applying Theorem 2.4 with $\wt {H}_0=1$.
       % \item \textbf{TODO}. In \cite{NicaOrtmann2026}: (check oscillator): gaussian times Hermite is uniformly bounded.
    %\end{enumerate}
        
    \end{proof}

    \begin{prop}[Projection property]
    \label{prop:projPG}
        Let $I$ be some index set and suppose that  $(z_j\colon j \in I)$ have the \pg\ law of order $k$ with covariance kernel $\Si$, 4-cumulants $\ka$ and scale parameter $n$. Let further $J\subset I$. Then $(z_j\colon j\in J)$ are also \pg\ of order $k$ with kernel $\Si\vert_J$, 4-cumulants $\kappa_{j\colon j\in J}$ and scale parameter $n$
    \end{prop}
    \begin{proof}
        We will use the equivalent definition using the characteristic function, see Remark \ref{rem:Fourier-PG}. The result follows directly by noting that if $\phi$ is the characteristic function of $(z_j\colon j \in I)$, then the  following function is the characteristic function of $(z_j\colon j \in J)$:
        $$(\lambda_j\colon j\in J)\longmapsto \phi(\la^0(I)),$$
        where $\la^0(I)$ is the $I$-vector obtained by `filling in' 0 on every position indexed in $I\setminus J$.
    \end{proof}

        \begin{prop}
        \label{prop:lin-PG}
            Any linear combination of \pg\ random variables is itself \pg. More specifically,  let $\la\in\bR^d$ and suppose that $Z=(Z_1,\ldots,Z_d)$ is distributed according to $\pglk\Si\ka n k$. Then $\hat Z=\la_1Z_1+\ldots+\la_d Z_d$ is distributed according to $\pglk{\hat\Si}{\hat\ka}{n}{k}$, where
            $$\hat\Si=\vec\la^T\Si\vec\la\quad\text{and}\quad\hat\ka = k! \sum_{\substack{k_1,\ldots,k_d\\ \sum k_i = k}} \frac{\kappa_{k_1,\ldots,k_d}}{k_1!\cdots k_d!} \vec \la^{\vec{k}}.$$
        \end{prop}
        \begin{proof}
           We will once again appeal to Remark \ref{rem:Fourier-PG}. The characteristic function of $\hat Z$ is given by
           \begin{align*}
               \phi_{\hat Z} &= \phi_Z\rb{t\vec \la} = \exp\set{ -\frac{1}{2}(t\vec \la)^T \Sigma t\vec \la } \left(1+ \frac{1}{\sqrt{n}^{k-2}} \sum_{\substack{k_1,\ldots,k_d\\ \sum k_i = k}} \frac{\kappa_{k_1,\ldots,k_d}}{k_1!\cdots k_d!} (\imath t\vec \la)^{\vec{k}} \right) \\
              & = \exp\set {-\frac{1}{2}\rb{\vec \la^T \Sigma \vec \la} t^2}  \left(1+ \frac{1}{\sqrt{n}^{k-2}}\frac{1}{k!} \rb{ k! \sum_{\substack{k_1,\ldots,k_d\\ \sum k_i = k}} \frac{\kappa_{k_1,\ldots,k_d}}{k_1!\cdots k_d!} \vec \la^{\vec{k}}} (\imath t)^k\right) 
           \end{align*}
           Comparing with the $d=1$ form of \eqref{eq:Fourier-PG} gives the result.
        \end{proof}

% \textbf{TODO:} Write relevant facts about PG distributions:
	
	% \begin{enumerate}[(i)]
	% 	\item integral formula just before Theorem 1.16, then refer to it.
	% %	\item\label{projPG} `projection property': if $z_j\colon j \in I$ are \pg\ with kernels $K$, $C$ then $z_j\colon j\in J$ are also paragaussian with $K\vert_J$, $C\vert_J$ for any $J\subset I$. The way to prove this is to use the characteristic function definition and set $\lambda_j=0$ for all $j\notin J$
	% %	\item Any linear combination of \pg\ is para-Gaussian
	% %	\item The Edgeworth theorem says: sums of i.i.d. random variables (and some others) are asymptotically para-Gaussian.
	% \end{enumerate} 

    We end this section with a remark about computing expectation of small-order polynomials with respect to the \pg\ law. This was exploited in \cite{NicaOrtmann2024M3L} to obtain the one-point distribution of the neural network.
    
    \begin{rem}
    \label{rem:expectation-low-order-poly}
        Suppose that $Z_1,\ldots,Z_d$ are distributed according to $\pglk\Si\ka n k$. For any polynomial $p$ in $d$ variables of degree less than $k$, the expectation of $p(Z)$ is simply the corresponding Gaussian integration:
        $$ \bE \ab{p(Z)} = \GaussExp\Si \ab{p(Z)}.$$
        This is because the polynomial $p$ can be written as a linear combination of $\wt H_\ell$ for $\ell<k$. We can then once more appeal to orthogonality of the Hermite polynomials with respect to the Gaussian measure (see Theorem 2.4 of \cite{NicaOrtmann2026}) to see that none of them contribute to the expecation, and only the Gaussian term remains.
    \end{rem}

\section{Alternate hands-on proof of Theorem \ref{thm:exact-general-cumulants}}
\label{subsec:second-proof-exact-general-cumulants}

This proof uses the hands-on direct definition of cumulants from the exponential generating function rather than the slick law of total cumulance.

    \begin{proof}[Second proof of Theorem \ref{thm:exact-general-cumulants}]
Let $\vec{\la} = \left( \la_{\al,j} \right)_{\al \in A, j\in [n]}$ be a set of coefficients that index all of $A \times [n]$, and consider the characteristic function $\Phi(\vec\la) := \mathbb{E}[e^{\imath \vec{\la} \cdot z^\ell}]$. On one hand, by the very definition of the multivariable cumulants, we can read the cumulants as the coefficients in the expansion of $\ln \Phi(\vec{\la})$, namely 
\begin{equation}
\label{eq:ln-char-func}
\ln\Phi(\vec\la)=\sum_{q\geq1}\frac{\imath^q}{q!}\sum_{(\alpha_1,j_1),\ldots,(\alpha_q,j_q)\in (A\times[n])^q}
\ka\rb{ \set{z^\ell_{\al_k,j_k} }_{k=1}^q }\la_{\al_1,j_1}\cdots\la_{\al_q,j_q},
\end{equation}
On the other hand, we will now calculate $\ln \Phi(\vec{\la})$ directly in terms of $\Si$. Comparing this calculation will yield the desired result by matching the terms with \eqref{eq:ln-char-func}. 

Firstly, we claim that the joint characteristic function of layer $\ell$ can be written as
$\Phi(\vec\la)=\bE\ab{e^{-Y_n}}$, where the random variable $Y_n$ is simply the quadratic form:

\begin{equation}
\label{eq:Y-as-SiHat}
Y_n
:=\frac12\,\vec\la^{T}\rb{I_n\otimes\Si^\ell}\vec\la = \frac12\sum_{(\al_1,j_1),(\al_2,j_2)\in(A\times[n])^2}
\delta_{j_1,j_2}\,\Si^\ell(\al_1,\al_2)\,\la_{\al_1,j_1}\la_{\al_2,j_2}.
\end{equation}
This follows directly from the the fact that $z^\ell$ is conditionally Gaussian with covariance $I_n \otimes \Si^\ell$, and invoking the charactersitic function of a Gaussian.

Next, we can compute $\ln \Phi(\vec{\la}) = \ln\mathbb{E}[e^{-Y_n}]$ by using the cumulant expansion for the scalar random variable $Y$, namely $\ln\bE\ab{e^{tY}}=\sum_{p\geq1}\frac{t^p}{p!}\ka_p(Y)$, where  $\ka_p(Y)=\kappa(\underbrace{Y,\ldots,Y}_p)$ is the $p$-th cumulant of $Y$. Taking $t=-1$, this gives
\begin{equation}
\label{eq:lnPhi-Y}
\ln\Phi(\vec\la)=\ln\bE\ab{e^{-Y_n}}
=\sum_{p\geq1}\frac{(-1)^p}{p!}\,\ka_p(Y_n).
\end{equation}

By multilinearity of the cumulants, and since $Y_n$ is a sum over $(A\times[n])^2$ random variables as in \eqref{eq:Y-as-SiHat}, the cumulant  $\ka_p(Y_n)=\ka(Y_n,\ldots,Y_n)$ is a sum over
$2p$-tuples in $(A\times[n])^{2p}$, namely:
\begin{equation}
\label{eq:kappa-p-Y}
\ka_p(Y_n)=\frac1{2^p}\sum_{(\al_1,j_1),\ldots,(\al_{2p},j_{2p})\in(A\times[n])^{2p}}
T\Big((\al_1,j_1),\ldots,(\al_{2p},j_{2p})\Big)\;\la_{\al_1,j_1}\cdots\la_{\al_{2p},j_{2p}},
\end{equation}
where
\[
T\Big((\al_1,j_1),\ldots,(\al_{2p},j_{2p})\Big)
:=\Big(\prod_{k=1}^{p}\delta_{j_{2k-1},j_{2k}}\Big)\,
\ka\rb{\Si^\ell(\al_1,\al_2),\Si^\ell(\al_3,\al_4),\ldots,\Si^\ell(\al_{2p-1},\al_{2p})} .
\]
That is, the $k$-th copy of $Y_n$ contributes the consecutive pair of indices $(\al_{2k-1},j_{2k-1}),(\al_{2k},j_{2k})$.
Substituting into \eqref{eq:lnPhi-Y},
\begin{equation}
\label{eq:lnPhi-SiHat}
\ln\Phi(\vec\la)=\sum_{p\geq1}\frac{(-1)^p}{2^p\,p!}
\sum_{(\al_1,j_1),\ldots,(\al_{2p},j_{2p})\in(A\times[n])^{2p}}
T\Big((\al_1,j_1),\ldots,(\al_{2p},j_{2p})\Big)\;\la_{\al_1,j_1}\cdots\la_{\al_{2p},j_{2p}} .
\end{equation}

\medskip
Finally, we can match the coefficents of \eqref{eq:ln-char-func} and \eqref{eq:lnPhi-SiHat}, which both express
$\ln\Phi$ as sums over ordered tuples of indices, to obtain the desired resul. We match them degree by degree.

\emph{Odd degrees.} \eqref{eq:lnPhi-SiHat} contains only even powers of $\vec\la$, so all terms with $q$ odd in
\eqref{eq:ln-char-func} must vanish: all odd cumulants of the $z^\ell_{\al,j}$ are $0$.

\emph{Even degrees.} For $q=2p$ we have $\imath^{2p}=(-1)^p$, and comparing the degree-$2p$ parts gives
\begin{equation}
\label{eq:degree-2p-match}
\frac{1}{(2p)!}\sum_{(A\times[n])^{2p}}
\ka\rb{\set{z^\ell_{\al_i,j_i}}_{i=1}^{2p}}\la_{\al_1,j_1}\cdots\la_{\al_{2p},j_{2p}}
=\frac{1}{2^p\,p!}\sum_{(A\times[n])^{2p}}T\Big((\al_1,j_1),\ldots,(\al_{2p},j_{2p})\Big)\la_{\al_1,j_1}\cdots\la_{\al_{2p},j_{2p}} .
\end{equation}
To read off the cumulant, notice that a given monomial $\la_{\al_1,j_1}\cdots\la_{\al_{2p},j_{2p}}$ is produced by every
rearrangement of the tuple $(\al_1,j_1),\ldots,(\al_{2p},j_{2p})$, i.e.\ by the tuples
$(\al_{\si(1)},j_{\si(1)}),\ldots,(\al_{\si(2p)},j_{\si(2p)})$ for permutations $\si \in S_{2p}$, the group of permutations on $\{1,\ldots,2p\}$; if some indices
are repeated, each distinct tuple arises from the same number of permutations $\si$ on both sides.

Hence comparing the coefficient of a fixed monomial on both sides of \eqref{eq:degree-2p-match} gives actually a sum over all the permutations $\si \in S_{2p}$
\[
\frac{1}{(2p)!}\sum_{\si \in S_{2p}}
\ka\rb{\set{z^\ell_{\si(\al_k),\si(j_k)}}_{k=1}^{2p}}
=\frac{1}{2^p\,p!}\sum_{\si \in S_{2p}}T\Big((\al_{\si(1)},j_{\si(1)}),\ldots,(\al_{\si(2p)},j_{\si(2p)})\Big),
\]
where both sums run over all $(2p)!$ permutations $\si$. The cumulant is symmetric in its arguments, so every term on
the left equals $\ka\rb{z^\ell_{\al_1,j_1},\ldots,z^\ell_{\al_{2p},j_{2p}}}$, and therefore
\begin{equation}
\label{eq:kappa-as-sum-over-perms}
\ka\rb{\set{z^\ell_{\al_k,j_k}}_{k=1}^{2p}}
=\frac{1}{2^p\,p!}\sum_{\si \in S_{2p}}T\Big((\al_{\si(1)},j_{\si(1)}),\ldots,(\al_{\si(2p)},j_{\si(2p)})\Big) .
\end{equation}

Finally, we claim that the sum over permutations $S_{2p}$ on the RHS of \eqref{eq:kappa-as-sum-over-perms} can actually be written as a sum over pairings $\mathcal{P}(2p)$, namely
\begin{equation}
\label{eq:permutations-to-pairing}
\frac{1}{2^p\,p!}\sum_{\si \in S_{2p}}T\Big(\al_{\si(1)},j_{\si(1)}),\ldots,(\al_{\si(2p)},j_{\si(2p)})\Big)  = \sum_{\set{\set{a_1,b_1},\ldots,\set{a_p,b_p}}\in\mathcal P(2p)}\Big(\prod_{k=1}^p\delta_{j_{a_k},j_{b_k}}\Big)
\ka\rb{\set{\Si^\ell(\al_{a_k},\al_{b_k})}_{k=1}^p},
\end{equation}

 To see \eqref{eq:permutations-to-pairing}, notice that the value of $T$ only depends on which positions are grouped into the consecutive pairs $\{1,2\},\{3,4\},\ldots,\{2p-1,2p\}$: it is unchanged if we swap the two entries of any
pair (both $\delta_{j,j'}$ and $\Si^\ell(\al,\al')$ are symmetric) or if we permute the $p$ pairs among themselves (the
joint cumulant is symmetric in its arguments). These operations form a group of $2^p\,p!$ permutations. Consequently
the $(2p)!$ permutations $\si \in S_{2p}$ in \eqref{eq:kappa-as-sum-over-perms} split into equivalence classes of exactly $2^p\,p!$
permutations each, all giving the same value of $T$, and each class corresponds to exactly one pairing
$\set{\set{a_1,b_1},\ldots,\set{a_p,b_p}}\in\mathcal P(2p)$, namely the pairing
$\set{\si(1),\si(2)},\ldots,\set{\si(2p-1),\si(2p)}$. (There are $(2p)!/(2^p\,p!)=(2p-1)!!$ classes, which is indeed
$\abs{\mathcal P(2p)}$.) The common value of $T$ on the class of the pairing $\set{\set{a_1,b_1},\ldots,\set{a_p,b_p}}$ is
$\big(\prod_{k=1}^p\delta_{j_{a_k},j_{b_k}}\big)\,\ka\rb{\set{\Si^\ell(\al_{a_k},\al_{b_k})}_{k=1}^p}$. The factor
$2^p\,p!$ cancels the prefactor on the LHS of\eqref{eq:permutations-to-pairing} and yields the RHS. 

The result follows by combining \eqref{eq:kappa-as-sum-over-perms} with \eqref{eq:permutations-to-pairing}.
\end{proof}

    \printbibliography

\end{document}